\documentclass{article}
\usepackage[accepted]{icml2021}
\usepackage{url}

\usepackage{bbm}
\usepackage{bm}

\usepackage{amsfonts,epsfig,graphicx}
\usepackage{amsmath,amssymb,amsthm}

\usepackage{epsf}
\usepackage{epstopdf}
\usepackage{graphics}
\usepackage{amsfonts}
\usepackage{amsmath}
\usepackage{psfrag}
\usepackage{xcolor}
\usepackage{mathtools}
\usepackage[colorinlistoftodos]{todonotes} 
\usepackage{hyperref}
\usepackage{breakurl}

\usepackage{float}
\usepackage{thmtools}
\usepackage{thm-restate}
\usepackage[font=small,skip=0pt]{caption}
\usepackage{wrapfig}
\usepackage{subcaption}
\usepackage{hyperref}

\usepackage{pifont}
\newcommand{\cmark}{\ding{51}}%
\newcommand{\xmark}{\ding{55}}%

\DeclareCaptionLabelFormat{andtable}{#1~#2  \&  \tablename~\thetable}

\usepackage[toc,page,header]{appendix}
\usepackage{minitoc}

\newcommand{\lambdamax}[1]{\ensuremath{\lambda_{max}}}

\newlength{\widebarargwidth}
\newlength{\widebarargheight}
\newlength{\widebarargdepth}

\DeclareMathOperator{\trace}{\textbf{tr}}
\DeclareMathOperator{\rank}{rank}

\newtheorem{theo}{Theorem}[section]

\newtheorem{lem}{Lemma}[section]
\newtheorem{prop}{Proposition}[section]
\newtheorem{cor}{Corollary}[section]
\newtheorem{remark}{Remark}[section]

\theoremstyle{definition} 

\newtheorem{nota}{Notation}[section]
\newtheorem{de}{Definition}[section]
\newtheorem{exa}{Example}[section]
\newtheorem{as}{Assumption}[section]
\newtheorem{alg}{Algorithm}[section]

\newcommand{\btheo}{\begin{theo}}
\newcommand{\bde}{\begin{de}}
\newcommand{\ble}{\begin{lem}}
\newcommand{\bpr}{\begin{prop}}
\newcommand{\bno}{\begin{nota}}
\newcommand{\bex}{\begin{exa}}
\newcommand{\bcor}{\begin{cor}}
\newcommand{\spro}{\begin{proof}}
\newcommand{\bas}{\begin{as}}
\newcommand{\balg}{\begin{alg}}
\newcommand{\bremark}{\begin{remark}}

\newcommand{\etheo}{\end{theo}}
\newcommand{\ede}{\end{de}}
\newcommand{\ele}{\end{lem}}
\newcommand{\epr}{\end{prop}}
\newcommand{\eno}{\end{nota}}
\newcommand{\eex}{\end{exa}}
\newcommand{\ecor}{\end{cor}}
\newcommand{\fpro}{\end{proof}}
\newcommand{\eas}{\end{as}}
\newcommand{\ealg}{\end{alg}}
\newcommand{\eremark}{\end{remark}}
\newcommand{\reals}{\mathbb{R}}
\theoremstyle{plain}

\newtheorem{theos}{Theorem}
\newtheorem{props}{Proposition}
\newtheorem{lems}{Lemma}
\newtheorem{cors}{Corollary}
\newtheorem{rems}{Remark}

\theoremstyle{definition}
\newtheorem{exas}{Example}
\newtheorem{algs}{Algorithm}
\newtheorem{asss}{Assumption}
\newtheorem{defns}{Definition}

\newcommand{\btheos}{\begin{theos}}
\newcommand{\etheos}{\end{theos}}
\newcommand{\brems}{\begin{rems}}
\newcommand{\erems}{\end{rems}}
\newcommand{\bprops}{\begin{props}}
\newcommand{\eprops}{\end{props}}
\newcommand{\bdes}{\begin{defns}}
\newcommand{\edes}{\end{defns}}
\newcommand{\blems}{\begin{lems}}
\newcommand{\elems}{\end{lems}}
\newcommand{\bcors}{\begin{cors}}
\newcommand{\ecors}{\end{cors}}
\newcommand{\bexs}{\begin{exas}}
\newcommand{\eexs}{\end{exas}}
\newcommand{\balgs}{\begin{algs}}
\newcommand{\ealgs}{\end{algs}}
\newcommand{\bass}{\begin{asss}}
\newcommand{\eass}{\end{asss}}
\newcommand{\bit}{\begin{itemize}}
\newcommand{\eit}{\end{itemize}}

\newcommand{\data}{{\bf X}}
\newcommand{\sample}{\vec{x}}
\newcommand{\samplescalar}{x}
\newcommand{\weightscalar}{w}
\newcommand{\weight}{\vec{w}}
\newcommand{\weightmat}{\vec{W}}

\newcommand{\planted}{\vec{w}^*}
\newcommand{\plantedmat}{\vec{W}^*}

\newcommand{\plantedsub}{w}
\newcommand{\plantednostar}{\vec{w}}

\newcommand{\tilplanted}{\tilde{\vec{w}}^*}
\newcommand{\tilplantedmat}{\tilde{\vec{W}}^*}

\newcommand{\bias}{b}

\newcommand{\dual}{\boldsymbol{\lambda}}
\newcommand{\dualmat}{\boldsymbol{\Lambda}}

\newcommand{\dualscalar}{\lambda}

\newcommand{\relu}[1]{\left( #1 \right)_+}
\newcommand{\ball}{\mathcal{B}}

\newcommand{\bnvar}{\gamma}

\newcommand{\bnvarl}[1]{\gamma^{(#1)} }

\newcommand{\bnvarvecl}[1]{\bm{\gamma}^{(#1)} }

\newcommand{\bnmean}{\alpha}

\newcommand{\bnmeanl}[1]{\alpha^{(#1)} }

\newcommand{\bnmeanvecl}[1]{\bm{\alpha}^{(#1)} }

\newcommand{\bn}[1]{{\mathrm{BN}}_{\bnvar,\bnmean}\left(#1\right)}

\newcommand{\norm}[1]{\left\|#1 \right\|_{2} }

\newcommand{\normf}[1]{\left\|#1 \right\|_{F} }
 \DeclareMathOperator*{\argmin}{\arg\!\min} 
  \DeclareMathOperator*{\argmax}{\arg\!\max} 
\newcommand{\sign}{\text{sign}}
\let\vec\mathbf
\usepackage{hyperref}
\usepackage{url}
\usepackage{booktabs}
\usepackage{enumitem}

 \usepackage[stable]{footmisc}

\icmltitlerunning{Revealing the Structure of Deep Neural Networks via Convex Duality}

\begin{document}
\doparttoc 
\faketableofcontents 

\twocolumn[
\icmltitle{Revealing the Structure of Deep Neural Networks via Convex Duality}




\begin{icmlauthorlist}
\icmlauthor{Tolga Ergen}{to}
\icmlauthor{Mert Pilanci}{to}
\end{icmlauthorlist}

\icmlaffiliation{to}{Department of Electrical Engineering, Stanford University, CA, USA}

\icmlcorrespondingauthor{Tolga Ergen}{ergen@stanford.edu}

\icmlkeywords{Machine Learning, ICML}

\vskip 0.3in
]



\printAffiliationsAndNotice{} 

\begin{abstract}
We study regularized deep neural networks (DNNs) and introduce a convex analytic framework to characterize the structure of the hidden layers. We show that a set of optimal hidden layer weights for a norm regularized DNN training problem can be explicitly found as the extreme points of a convex set. For the special case of deep linear networks, we prove that each optimal weight matrix aligns with the previous layers via duality. More importantly, we apply the same characterization to deep ReLU networks with whitened data and prove the same weight alignment holds. As a corollary, we also prove that norm regularized deep ReLU networks yield spline interpolation for one-dimensional datasets which was previously known only for two-layer networks. Furthermore, we provide closed-form solutions for the optimal layer weights when data is rank-one or whitened. The same analysis also applies to architectures with batch normalization even for arbitrary data. Therefore, we obtain a complete explanation for a recent empirical observation termed Neural Collapse where class means collapse to the vertices of a simplex equiangular tight frame.
\end{abstract}

\section{Introduction}

Deep neural networks (DNNs) have become extremely popular due to their success in machine learning applications. Even though DNNs are highly over-parameterized and non-convex, simple first-order algorithms, e.g., Stochastic Gradient Descent (SGD), can be used to successfully train them. Moreover, recent work has shown that highly over-parameterized networks trained with SGD obtain simple solutions that generalize well \cite{infinite_width,parhi_minimum,ergen2020aistats,ergen2020journal}, where two-layer ReLU networks with the minimum Euclidean norm solution and zero training error are proven to fit a linear spline model in 1D regression. In addition, a recent series of work \cite{pilanci2020convex,ergen2020cnn,sahiner2021vectoroutput,vikul2021generative} showed that regularized two-layer ReLU network training problems exhibit a convex loss landscape in a higher dimensional space, which was previously attributed to the benign impacts of overparameterization \cite{brutzkus_overparameterized_linear,li_overparameterized,du_overparameterized,ergen2019shallow}. Therefore, regularizing the solution towards smaller norm weights might be the key to understand the generalization properties and loss landscape of DNNs. However, analyzing DNNs is still theoretically elusive even in the absence of nonlinear activations. To this end, we study norm regularized DNNs and develop a framework based on convex duality to characterize a set of optimal solutions to the training problem. 

 Deep linear networks have been the subject of extensive theoretical analysis due to their tractability. A line of research \cite{saxe_deeplinear,arora_deeplinearconv,recht_deeplinear,du_deeplinear,shamir_convergence_deeplinear} focused on GD training dynamics, however, they lack the analysis of solution set and generalization properties of deep networks. 
  Another line of research \cite{gunasekar_implicit_mf,arora_implicit_deepmf,srebro_lowrankrecovery} studied the generalization properties via matrix factorization and showed that linear networks trained with GD converge to minimum nuclear norm solutions. 
%
Later on, \cite{arora_linear_alignment,du_relu_alignment} showed that gradient flow enforces the layer weights to align. \cite{telgrasky_deeplinear} further proved that each layer weight matrix is asymptotically rank-one. These results provide insights to characterize the structure of the optimal layer weights, however, they require multiple strong assumptions, e.g., linearly separable training data and strictly decreasing loss function, which makes the results impractical. Furthermore, \cite{zhang2019multibranch} provided some characterizations for nonstandard networks, which are valid for hinge loss with an uncommon regularization. Unlike these studies, we introduce a complete characterization for regularized deep network training problems without requiring such assumptions.

\subsection{Our contributions}
Our contributions can be summarized as follows
 \begin{itemize}[leftmargin=*]
\item We introduce a convex analytic framework that characterizes a set of optimal solutions to regularized training problems as the extreme points of a convex set.

\item For deep linear networks, we prove that each optimal layer weight matrix aligns with the previous layers via convex duality.

\item For deep ReLU networks, we obtain the same weight alignment result for whitened or rank-one data matrices. As a corollary, we achieve closed-form solutions for the optimal hidden layer weights when the data is whitened or rank-one (see Theorem \ref{theo:deeprelu} and \ref{theo:closedform_regularized_multiclass}).

\item As another corollary, we prove that the optimal regularized ReLU networks are linear spline interpolators for one-dimensional, i.e., rank-one, data which generalizes the two-layer results for one-dimensional data in \cite{infinite_width,parhi_minimum,ergen2020aistats,ergen2020journal} to arbitrary depth. 

\item We show that whitening/rank-one assumptions can be removed by placing batch normalization in between layers (see Theorem \ref{theo:deep_vector_closedform}). Hence, our results explain a recent empirical observation, termed Neural Collapse \cite{papyan2020neuralcollapse}, where class means collapse to the vertices of a simplex equiangular tight frame (see Corollary \ref{cor:neural_collapse}).

 \end{itemize}
\renewcommand{\arraystretch}{1.0}
  \begin{figure*}[t!] 
   \centering
    \includegraphics[width=0.25\textwidth]{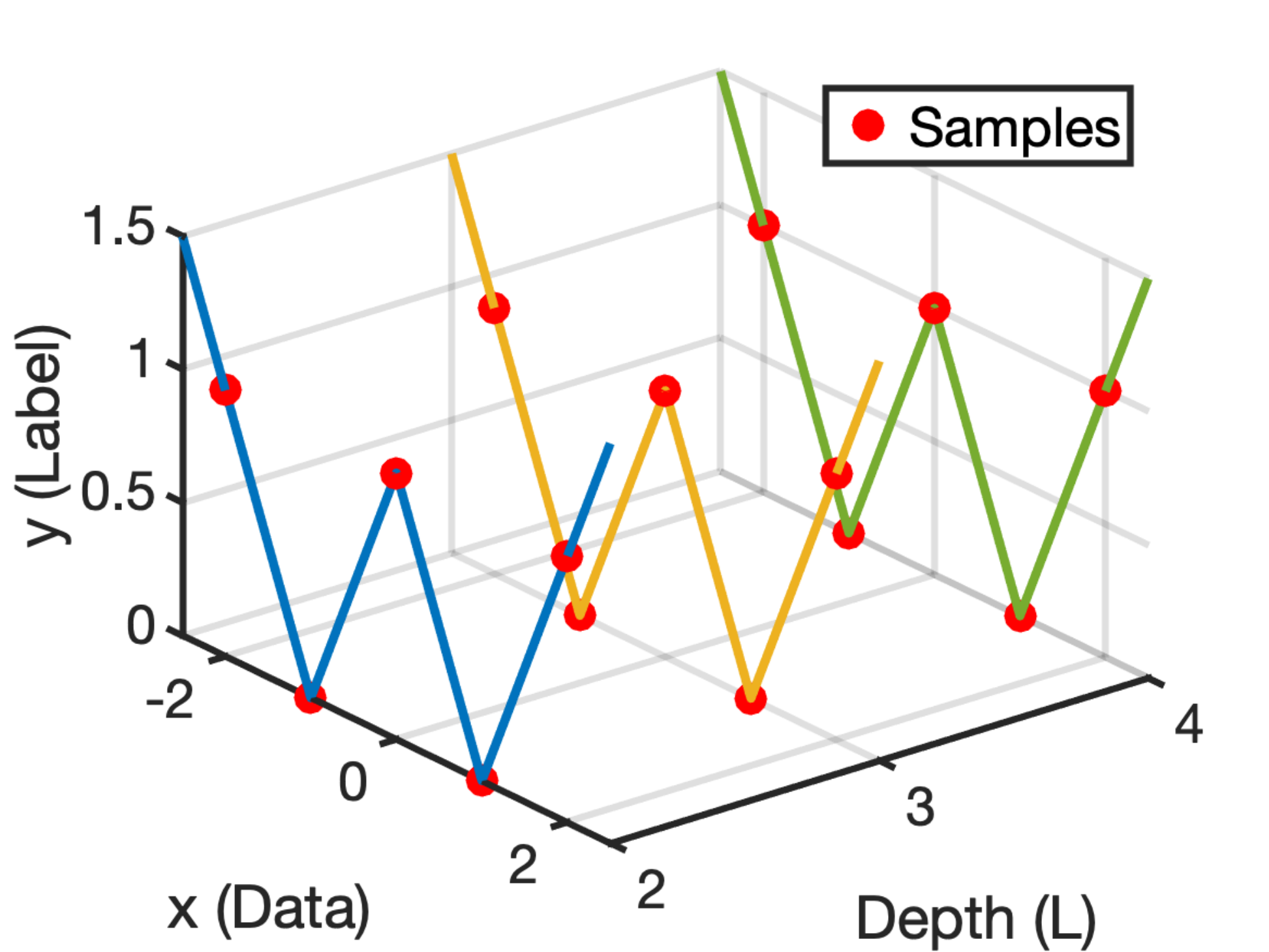} 
    \qquad
  \resizebox{1.4\columnwidth}{0.25\columnwidth}{
  \begin{tabular}[b]{|c|c|c|c|c|}
    \hline
    & \textbf{Width} ($m$) & \textbf{Assumption}& \textbf{Depth} ($L$) & \textbf{\# of outputs} ($K$) \\  \hline
   {\cite{infinite_width}}& $\infty$ & \text{1D data ($d=1$)}& 2 & \xmark \ ($K=1$) \\ \hline
   {\cite{parhi_minimum}} & $\infty$ & \text{1D data ($d=1$)} & 2 & \xmark \ ($K=1$) \\ \hline
   {\cite{ergen2020aistats,ergen2020journal}} &  finite &$ \text{rank-one/whitened} $    & 2&   \cmark \ ($K\geq 1$)  \\ \hline
    \textbf{Our results} & finite & \begin{tabular}{@{}c@{}}\text{rank-one/whitened } \\  or BatchNorm \end{tabular} & $L\geq 2$ & \cmark \  ($K\geq 1$) \\
    \hline
  \end{tabular}
  } 
  \vspace{0.1in}
    \captionlistentry[table]{One dimensional interpolation using $L$-layer ReLU networks with $20$ neurons in each hidden layer. As predicted by Corollary \ref{cor:multilayer_rankone_kinks}, the optimal solution is given by piecewise linear splines for any $L\ge 2$.Comparison with previous studies for the spline interpolation characterization}\label{tab:comparison}
    \captionsetup{labelformat=andtable}
    \caption{One dimensional interpolation using $L$-layer ReLU networks with $20$ neurons in each hidden layer. As predicted by Corollary \ref{cor:multilayer_rankone_kinks}, the optimal solution is given by piecewise linear splines for any $L\ge 2$. Additionally, we provide a comparison with previous studies about this characterization.} \label{fig:deeprelu_function}
  \end{figure*}

\subsection{Overview of our results}

\textbf{Notation:} We denote matrices/vectors as uppercase/lowercase bold letters. We use $\vec{0}_k$ (or $\vec{1}_k$) and $\vec{I}_k$ to denote a vector of zeros (or ones) and the identity matrix of size $ k \times k$, respectively. We denote the set of integers from $1$ to $n$ as $[n]$. To denote Frobenius, operator, and nuclear norms, we use $\|\cdot \|_{F}$, $\|\cdot\|_2$, and $\| \cdot\|_{*}$, respectively. We also use $\trace$ to denote the trace of a matrix. Furthermore, $\sigma_{max}(\cdot)$ and $\sigma_{min}(\cdot)$ represent the maximum and minimum singular values, respectively and the unit $\ell_2$-ball $\ball_2$ is defined as $\ball_2=\{\vec{u} \in \mathbb{R}^d\,\vert\,\|\vec{u}\|_2\le 1\}$. We also provide further explanations about our notation in Table \ref{tab:notations} in Appendix.

We consider an $L$-layer network with layer weights $\weightmat_{l,j} \in \mathbb{R}^{m_{l-1} \times m_l}$ and $\weight_L \in \mathbb{R}^m$, $\forall l \in [L],\, \forall j \in [m]$, where $m_0=d$ and $m_{L-1}=1$, respectively. Then, given a data matrix $\data \in \mathbb{R}^{n \times d}$, the output is 
$    f_{\theta,L}(\data)=\vec{A}_{L-1}\weight_L,\;\vec{A}_{l,j}=g(\vec{A}_{l-1,j}\weightmat_{l,j})\;\forall l\in [L-1],   
$
where $\vec{A}_{0,j}=\data$, $\vec{A}_{L-1} \in \mathbb{R}^{n \times m}$, and $g(\cdot)$ is the activation function. Given labels $\vec{y} \in \mathbb{R}^n$, the training problem is as follows
\begin{align}
    \min_{\{\theta_l\}_{l=1}^L} \mathcal{L}( f_{\theta,L}(\data),\vec{y})+\beta  \mathcal{R}(\theta) \,, \label{eq:problem_statement}
\end{align}
where $\mathcal{L}(\cdot,\cdot)$ is an arbitrary  loss function, $\mathcal{R}(\theta)$ is regularization for the layer weights, $\beta>0$ is a regularization parameter, $\theta_l=\{\{\weightmat_{l,j}\}_{j=1}^m,m_l\}$, and $\theta=\{\theta_l\}_{l=1}^L$. 
 In the paper, for the sake of presentation simplicity, we illustrate the conventional training setup with squared loss and $\ell_2^2$-norm regularization. However, our analysis is valid for arbitrary convex loss functions as proven in Appendix \ref{sec:supp_general_loss}. 
Thus, we consider the following optimization problem
\begin{align}
    P^*=& \min_{\{\theta_l\}_{l=1}^L} \mathcal{L}( f_{\theta,L}(\data),\vec{y})+ \frac{\beta}{2}\sum_{j=1}^m\sum_{l=1}^L \|\weightmat_{l,j}\|_F^2 . \label{eq:problem_def_overview}
\end{align}
Next, we show that the minimum $\ell_2^2$-norm is equivalent to minimum $\ell_1$-norm after a rescaling.

\begin{restatable}{lem}{equivalenceoverview} \label{lemma:scaling_deep_overview}
The following problems are equivalent :
\begin{align*}
\begin{split}
      &\min_{\{\theta_l\}_{l=1}^L}  \mathcal{L}( f_{\theta,L}(\data),\vec{y})+\frac{\beta}{2}\sum_{j=1}^m\sum_{l=1}^L \|\weightmat_{l,j}\|_F^2 
\end{split}\\
     \begin{split}
      =   &\min_{\{\theta_l\}_{l=1}^L, \{t_j\}_{j=1}^m} \mathcal{L}( f_{\theta,L}(\data),\vec{y})+ \beta \|\weight_L\|_1 \\
         &\hspace{4cm}+ \frac{\beta}{2}(L-2)\sum_{j-1}^m t_j^2 \\&\text{ s.t. } \weight_{L-1,j} \in \ball_2, \|\weightmat_{l,j} \|_F \leq t_j, \; \forall l \in [L-2]   
\end{split}.
\end{align*}
\end{restatable}

Using Lemma \ref{lemma:scaling_deep_overview}\footnote{The proof is presented in Appendix \ref{sec:supp_equivalence}.}, we first take the dual with respect to the output layer weights $\weight_L$ and then change the order of min-max to achieve the following dual as a lower bound \footnote{For the definitions and details see Appendix \ref{sec:supp_general_loss} and \ref{sec:supp_dual_derivations}.}
\begin{align}\label{eq:dual_deeplinear_overview}
        P^*\geq &D^* =\min_{\{t_j\}_{j=1}^m}\max_{\vec{\dual}} - \mathcal{L}^*(\dual) + \frac{\beta}{2}(L-2)\sum_{j-1}^m t_j^2 \nonumber \\
        &\text{ s.t. } \max_{\substack{ \weight_{L-1,j} \in \ball_2 \\\|\weightmat_{l,j} \|_F \leq t_j}}\| \vec{A}_{L-1,j}^T \vec{\dual}\|_{\infty}\leq \beta\,.
\end{align}
To the best of our knowledge, the dual DNN characterization \eqref{eq:dual_deeplinear_overview} is novel. Using this result, we first characterize a set of weights that minimize the objective via the optimality conditions and active constraints in \eqref{eq:dual_deeplinear_overview}. We then prove the optimality of these weights by proving strong duality, i.e., $P^*=D^*$, for DNNs. We then show that, for deep linear networks, optimal weight matrices align with the previous layers. 

More importantly, the same analysis and conclusions also apply to deep ReLU networks when the input is whitened and/or rank-one. Here, we even obtain closed-form solutions for the optimal layer weights. As a corollary, we show that deep ReLU networks fit a linear spline interpolation when the input is one-dimensional. We also provide an experiment in Figure \ref{fig:deeprelu_function} to verify this claim. Note that this result was previously known only for two-layer networks \cite{infinite_width,parhi_minimum,ergen2020aistats,ergen2020journal} and here we extend it to arbitrary depth $L$ (see Table \ref{tab:comparison} for details). We also show that the whitened/rank-one
assumption can be removed by introducing batch normalization in between layers, which reflects the training setup in practice.

\section{Warmup: Two-layer linear networks} 
As a warmup, we first consider the simple case of two-layer linear networks with the output $f_{\theta,2}(\data)=\data \weightmat_1 \weight_2$ and the parameters as
$\theta \in \Theta= \{(\weightmat_1,\weight_2,m)\,\vert\, \weightmat_1 \in \mathbb{R}^{d \times m},  \weight_2 \in \mathbb{R}^m, m \in \mathbb{Z}_+\}$. Motivated by recent results \cite{neyshabur_reg,infinite_width,parhi_minimum,ergen2020aistats,ergen2020journal}, we first focus on a minimum norm\footnote{This corresponds to weak regularization, i.e., $\beta\rightarrow 0$ in \eqref{eq:problem_statement} (see e.g. \cite{margin_theory_tengyu}.). } variant of \eqref{eq:problem_statement} with squared loss, which can be written as
\begin{align}
    & \min_{\theta \in \Theta}  \|\weightmat_1\|_F^2 + \| \weight_2\|_2^2\; \text{ s.t. } f_{\theta,2}(\data)=\vec{y}. \label{eq:problem_def1}
\end{align}
Using Lemma \ref{lemma:scaling}\footnote{All the equivalence lemmas are presented in Appendix \ref{sec:supp_equivalence}.}, we equivalently have
\begin{align}
    P^*=\min_{\theta \in \Theta} \| \weight_2 \|_1 \;\text{ s.t. } f_{\theta,2}(\data)=\vec{y},  \weight_{1,j} \in \ball_2, \forall j,\label{eq:problem_def_linear}
\end{align}
which has the following dual form.

\begin{restatable}{theo}{theotwolayerequalitydual} \label{theo:twolayer_equality_dual}
The dual of the problem in \eqref{eq:problem_def_linear} is given by 
\begin{align}
    &P^*\ge D^*=\max_{\dual\in \reals^n} \dual^T \vec{y} \;\text{ s.t. } \max_{\weight_1 \in \ball_2}\big\vert\dual^T \data \weight_1 \big \vert \le 1\, . \label{eq:twolayer_equality_dual}
\end{align}
For \eqref{eq:problem_def_linear}, $\exists m^*\leq n+1$ such that strong duality holds, i.e., $P^*=D^*$, $\forall m \geq m^*$ and $\weightmat_1^*$ satisfies
%
$    \|(\data \weightmat_1^*)^T \dual^* \|_{\infty} = 1\,,
$
where $\dual^*$ is the dual optimal parameter.
\end{restatable}

Using Theorem \ref{theo:twolayer_equality_dual}, we now characterize the optimal neurons as the extreme points of a convex set.
\begin{restatable}{cor}{cortwolayerextremeoptimality} \label{cor:twolayer_extreme_optimality}
By Theorem \ref{theo:twolayer_equality_dual}, the optimal neurons are extreme points which solve
%
$    \argmax_{\weight_1 \in \ball_2 } |{\dual^*}^T \data\weight_1\,|.
$
\end{restatable}
\begin{defns}
We call the maximizers of the constraint in Corollary \ref{cor:twolayer_extreme_optimality} \textit{extreme points} throughout the paper.
\end{defns}

From Theorem \ref{theo:twolayer_equality_dual}, we have the following dual problem 
\begin{align}
        &\max_{\dual} \dual^T \vec{y} \;\mbox{ s.t. } \max_{\weight_1 \in \ball_2}|\dual^T \data \weight_1 | \leq 1. \label{eq:dual_linear}
\end{align}
Let $\data=\vec{U}_x\vec{\Sigma}_x\vec{V}_x^T$ be the singular value decomposition (SVD) of $\data$\footnote{In this paper, we use full SVD unless otherwise stated.}. If we assume that there exists $\planted$ such that $\data \planted=\vec{y}$ due to Proposition \ref{prop:planted_model_twolayer}, then \eqref{eq:dual_linear} is equivalent to
\begin{align}
            &\max_{\tilde{\dual}} \tilde{\dual}^T \vec{\Sigma}_x \tilplanted \; \mbox{ s.t. } \|\vec{\Sigma}_x^T\tilde{\dual} \|_2 \leq 1, \label{eq:dual_linearv2}
\end{align}
where $\tilde{\dual}=\vec{U}_x^T \dual$, $\tilplanted=\vec{V}_x^T \planted$, and we changed the constraint since the extreme point is achieved when $\weight_1=\data^T\dual/\|\data^T \dual\|_2$. Given $\rank(\data)=r$, we have
\begin{align}\label{eq:w_rs}
    \tilde{\dual}^T \vec{\Sigma}_x \tilplanted&= \tilde{\dual}^T \vec{\Sigma}_x \underbrace{\begin{bmatrix} \vec{I}_r & \vec{0}_{r\times d-r}\\ \vec{0}_{d-r \times r} & \vec{0}_{d-r \times d-r}\end{bmatrix}\tilplanted}_{\planted_r} \nonumber \\
    &\leq \|  \vec{\Sigma}_x^T \tilde{\dual}\|_2 \| \tilplanted_r \|_2\leq\| \tilplanted_r \|_2,
\end{align}
which shows that the maximum objective value is achieved when $ \vec{\Sigma}_x^T \tilde{\dual}=c_1 \tilplanted_r$. Thus, we have
\begin{align*}
    \weight_1^*=\frac{\vec{V}_x\vec{\Sigma}_x^T\tilde{\dual}}{\|\vec{V}_x\vec{\Sigma}_x^T\tilde{\dual}\|_2}=\frac{\vec{V}_x\tilplanted_r}{\|\tilplanted_r\|_2}=\frac{\mathcal{P}_{\data^T}(\planted)}{\| \mathcal{P}_{\data^T}(\planted)\|_2},
\end{align*}
 where $\mathcal{P}_{\data^T}(\cdot)$ projects its input onto the range of $\data^T$. In the sequel, we first show that one can consider a planted model without loss of generality and then prove strong duality for \eqref{eq:problem_def_linear}.
\begin{restatable}{prop}{propplantedmodeltwolayer}[\cite{du_deeplinear}]\label{prop:planted_model_twolayer}
Given $\weight^*=\argmin_{\weight}\|\data\weight-\vec{y}\|_2$, we have
\begin{align*}
    \argmin_{\weightmat_1,\weight_2}\| \data \weightmat_1 \weight_2-\data \weight^*\|_2^2=\argmin_{\weightmat_1,\weight_2}\| \data \weightmat_1 \weight_2-\vec{y}\|_2^2.
\end{align*}
\end{restatable}
\begin{restatable}{theo}{theostrongdualitytwolayer}\label{theo:strong_duality_twolayer}
Let $\{ \data,\vec{y}\}$ be feasible for \eqref{eq:problem_def_linear}, then strong duality holds for finite width networks.
\end{restatable}

\subsection{Regularized training problem}\label{sec:regularized}
In this section, we define the regularized version of \eqref{eq:problem_def_linear} as
\begin{align}
    \min_{\theta \in \Theta} \frac{1}{2}\| f_{\theta,2}(\data) -\vec{y} \|_2^2+\beta \| \weight_2 \|_1 \;\text{ s.t. }  \weight_{1,j} \in \ball_2, \label{eq:problem_defreg_linear}
\end{align}
which has the following dual form
\begin{align*}
        &\max_{\dual} - \frac{1}{2}\| \dual- \vec{y}\|_2^2 +\frac{1}{2}\|\vec{y}\|_2^2 \; \mbox{ s.t. } \max_{\weight_1 \in \ball_2}|\dual^T \data \weight_1 | \leq \beta. 
\end{align*}
Then, an optimal neuron needs to satisfy the condition
\begin{align*}
    \weight_1^*=      \frac{ \data^T \mathcal{P}_{\data,\beta}(\vec{y})}{\|\data^T \mathcal{P}_{\data,\beta}(\vec{y}) \|_2}
\end{align*}
where $\mathcal{P}_{\data,\beta}(\cdot) \text{ projects to }\{\vec{u}\in \mathbb{R}^n \;| \;\| \data^T\vec{u}\|_2 \leq \beta\}$. We now prove strong duality.
\begin{restatable}{theo}{theostrongdualitytwolayerregularized}\label{theo:strong_duality_twolayer_regularized}
Strong duality holds for \eqref{eq:problem_defreg_linear} with finite width.
\end{restatable}

\subsection{Training problem with vector outputs}
Here, our model is $f_{\theta,2}(\data)=\data \weightmat_1 \weightmat_2$ to estimate $\vec{Y} \in \mathbb{R}^{n \times K}$, which can be optimized as follows
\begin{align}\label{eq:problem_def_linear_vector1}
        \min_{\theta \in \Theta}  \|\weightmat_1\|_F^2+\|\weightmat_2\|_F^2 \; \text{ s.t. } f_{\theta,2}(\data)=\vec{Y}.
\end{align}
Using Lemma \ref{lemma:scaling_vector}, we reformulate \eqref{eq:problem_def_linear_vector1} as 
\begin{align}
    \min_{\theta \in \Theta} \sum_{j=1}^m\| \weight_{2,j} \|_2 \;\text{ s.t. } f_{\theta,2}(\data)=\vec{Y},  \weight_{1,j} \in \ball_2,\label{eq:problem_def_linear_vector}
\end{align}
which has the following dual with respect to $\weightmat_2$
\begin{align}
        \hspace{-0.2cm}\max_{\dualmat} \trace(\dualmat^T \vec{Y} ) \;\mbox{ s.t. } \|\dualmat^T \data \weight_1 \|_2 \leq 1, \; \forall \weight_1 \in \ball_2. \label{eq:dual_linear_vector}
\end{align}
Since we can assume $\vec{Y}=\data \plantedmat$ due to Proposition \ref{prop:planted_model_twolayer},
\begin{align}
\label{eq:vectoroutput_obj_upperbound}
   \trace(\dualmat^T \vec{Y})&=\trace(\dualmat^T \data \plantedmat)= \trace(\dualmat \vec{U}_x \vec{\Sigma}_x \tilplantedmat_r) \nonumber \\&\leq \sigma_{max}(\dualmat^T \vec{U}_x \vec{\Sigma}_x ) \left\|\tilplantedmat_r\right\|_{*} \leq \| \tilplantedmat_r\|_{*}
\end{align}
where $\sigma_{max}(\dualmat^T \data)\leq 1$ due to \eqref{eq:dual_linear_vector} and
$
   \tilplantedmat_r=\begin{bmatrix} \vec{I}_r & \vec{0}_{r\times d-r}\\ \vec{0}_{d-r \times r} & \vec{0}_{d-r \times d-r}\end{bmatrix}\vec{V}_x^T\plantedmat .
$
Given the SVD of $\tilplantedmat_r$, i.e., $\vec{U}_{\plantedsub} \vec{\Sigma}_{\plantedsub} \vec{V}_{\plantedsub}^T$, choosing
\begin{align*}
  \dualmat^T \vec{U}_x \vec{\Sigma}_x=\vec{V}_{\plantedsub} \begin{bmatrix} \vec{I}_{r_{\plantedsub}} & \vec{0}_{r_{\plantedsub}\times d-r_{\plantedsub}}\\ \vec{0}_{K-r_{\plantedsub}\times r_{\plantedsub}} & \vec{0}_{K-r_{\plantedsub} \times d-r_{\plantedsub}}\end{bmatrix}\vec{U}_{\plantedsub}^T  
\end{align*}
achieves the upper-bound above, where $r_{\plantedsub}=\rank(\tilplantedmat_r)$. Thus, optimal neurons are a subset of the first $r_{\plantedsub}$ right singular vectors of $\dualmat^T \data$. We next prove strong duality.

\begin{restatable}{theo}{theostrongdualitylinearvector}\label{theo:strong_duality_linear_vector}
Let $\{ \data,\vec{Y}\}$ be feasible for \eqref{eq:problem_def_linear_vector}, then strong duality holds for finite width networks.
\end{restatable}

\subsubsection{Regularized case}
Here, we define the regularized version of \eqref{eq:problem_def_linear_vector} as follows
\begin{align*}
    \min_{\theta \in \Theta} \frac{1}{2}\| f_{\theta,2}(\data) -\vec{Y}\|_F^2 +\beta \sum_{j=1}^m\| \weight_{2,j} \|_2\;\text{ s.t. }  \weight_{1,j} \in \ball_2,
\end{align*}
which has the following dual with respect to $\weightmat_2$
\begin{align*}
        &\max_{\dualmat} - \frac{1}{2}\| \dualmat- \vec{Y}\|_F^2 +\frac{1}{2}\|\vec{Y}\|_F^2 \; \mbox{ s.t. } \sigma_{max}(\dualmat^T \data) \leq \beta. 
\end{align*}
Then, the optimal neurons are a subset of the maximal right singular vectors of $\mathcal{P}_{\data,\beta}(\vec{Y})^T \data$, where $\mathcal{P}_{\data,\beta}(\cdot)$ projects its input to the set $\{\vec{U}\in \mathbb{R}^{n\times K}\;|\; \sigma_{max}( \vec{U}^T \data ) \leq \beta\}$.

\bremark \label{rem:twolayer_rank}
Note that the optimal neurons are the right singular vectors of $\mathcal{P}_{\data,\beta}(\vec{Y})^T \data$ that achieve $\|\mathcal{P}_{\data,\beta}(\vec{Y})^T \data \weight_1^*\|_2=\beta$, where $\|\weight_1^*\|_2=1$. This implies that $\|\vec{Y}^T \data \weight_1^*\|_2 \geq \beta$, therefore, the number of optimal neurons and $\text{rank}(\weightmat_1^*)$ are determined by $\beta$.
\eremark

\bremark\label{rem:twolayer_nonunique}
The right singular vectors of $\mathcal{P}_{\data,\beta}(\vec{Y})^T \data$ are not the only solutions. Consider $\vec{u}_1$ and $\vec{u}_2$ as the optimal right singular vectors. Then, $\vec{u}=\alpha_1 \vec{u}_1+\alpha_2 \vec{u}_2$ with $\alpha_1^2+\alpha_2^2=1$ also achieves the upper-bound, thus, optimal.
\eremark

\section{Deep linear networks\footnote{Since the derivations are similar, we present the details in Appendix \ref{sec:proofs_deep_linear}. }}
We now consider an $L$-layer linear network with the output function $f_{\theta,L}(\data)=\sum_{j=1}^m\data\weightmat_{1,j} \ldots \weightscalar_{L,j}$, and the training problem
\begin{align}\label{eq:problemdef_deeplinear}
    \min_{\{\theta_l\}_{l=1}^L} \frac{1}{2}\sum_{j=1}^m\sum_{l=1}^L\| \weightmat_{l,j} \|_F^2 \; \text{ s.t. } f_{\theta,L}(\data)=\vec{y}.
\end{align}
\begin{restatable}{prop}{propdeepweightnormeq}\label{prop:deep_weightnorm_eq}
First $L-2$ layer weight matrices in \eqref{eq:problemdef_deeplinear} have the same operator and Frobenius norms, i.e., $t_j=\|\weightmat_{l,j}\|_F=\|\weightmat_{l,j}\|_2, \forall l \in [L-2]$, $\forall j \in [m]$.
\end{restatable}
This result shows that the layer weights obey an alignment condition. After using the scaling in Lemma \ref{lemma:scaling_deep} and the same convex duality arguments, a set of optimal solutions to the training problem can be described as follows.
\begin{restatable}{theo}{theodeepweightsform}\label{theo:deep_weights}
Optimal layer weights for \eqref{eq:problemdef_deeplinear} are
\begin{align*}
    & \weightmat_{l,j}^* =\begin{cases}
     t_j^*\frac{\vec{V}_x\tilplanted_r  }{\|\tilplanted_r \|_2}\boldsymbol{\rho}_{1,j}^T\; \text{ if } l=1\\
     t_j^*\boldsymbol{\rho}_{l-1,j} \boldsymbol{\rho}_{l,j}^T\;  \text{ if } 1<l \leq L-2\\
     \boldsymbol{\rho}_{L-2,j}\; \text{ if } l=L-1
    \end{cases},
\end{align*}
where $\boldsymbol{\rho}_{l,j} \in \mathbb{R}^{m_l}$ such that $\|\boldsymbol{\rho}_{l,j}\|_2=1, \; \forall l \in [L-2], \; \forall j \in [m]$ and $\tilplanted_r $ is defined in \eqref{eq:w_rs}. 
\end{restatable}
Next, we prove strong duality holds. 
\begin{restatable}{theo}{theodeepstrongduality}\label{theo:deep_strong_duality}
Let $\{ \data,\vec{y}\}$ be feasible for \eqref{eq:problemdef_deeplinear}, then strong duality holds for finite width networks.
\end{restatable}

\begin{restatable}{cor}{cordeeplinear}\label{cor:deeplinear}
Theorem \ref{theo:deep_weights} implies that deep linear networks can obtain a scaled version of $\vec{y}$ using only the first layer, i.e., $\data\weightmat_1 \boldsymbol{\rho}_1=c\vec{y}$, where $c>0$. Therefore, the remaining layers do not contribute to the expressive power of the network.
\end{restatable}

\subsection{Regularized training problem}
We now present the regularized training problem as follows
\begin{align}\label{eq:problemdef_deeplinear_regularized}
   \min_{\{\theta_l\}_{l=1}^L} \frac{1}{2}\|f_{\theta,L}(\data)-\vec{y}\|_2^2+ \frac{\beta}{2}\sum_{j=1}^m\sum_{l=1}^L\| \weightmat_{l,j} \|_F^2 .
\end{align}
Next result provides a set of optimal solutions to \eqref{eq:problemdef_deeplinear_regularized}.
\begin{restatable}{theo}{theodeepweightsregform}\label{theo:deep_weights_reg}
Optimal layer weights for \eqref{eq:problemdef_deeplinear_regularized} are
\begin{align*}
    & \weightmat_{l,j}^*= \begin{cases} t_j^*\frac{ \data^T \mathcal{P}_{\data,\beta}(\vec{y})}{\|\data^T \mathcal{P}_{\data,\beta}(\vec{y}) \|_2}\boldsymbol{\rho}_{1,j}^T\; \text{ if } l=1\\ t_j^*\boldsymbol{\rho}_{l-1,j} \boldsymbol{\rho}_{l,j}^T\; \text{ if } 1<l\leq L-2 \\
    \boldsymbol{\rho}_{L-2,j}\; \text{ if } l=L-1\end{cases},
\end{align*}
where $\mathcal{P}_{\data,\beta}(\cdot) \text{ projects to }\left\{\vec{u}\in \mathbb{R}^n \; |\; \| \data^T\vec{u}\|_2 \leq \beta t_j^{*^{2-L}}\right\}$. 
\end{restatable}

\begin{restatable}{cor}{cordeepstrongdualityregularized}\label{cor:deep_strong_duality_regularized}
Theorem \ref{theo:deep_strong_duality} also shows that strong duality holds for the training problem in \eqref{eq:problemdef_deeplinear_regularized}.
\end{restatable}

\subsection{Training problem with vector outputs}
Here, we consider vector output deep networks with the output function $f_{\theta,L}(\data)=\sum_{j=1}^m\data\weightmat_{1,j} \ldots \weight_{L,j}^T$. In this case, we have the following training problem
\begin{align}\label{eq:problemdef_deeplinear_vector}
    &\min_{\{\theta_l\}_{l=1}^L}  \sum_{j=1}^m\sum_{l=1}^L\| \weightmat_{l,j} \|_F^2  \;\text{ s.t. } f_{\theta,L}(\data)=\vec{Y}.
\end{align}
Using the scaling in Lemma \ref{lemma:scaling_deep_vector} and the same convex duality arguments, optimal layer weights for \eqref{eq:problemdef_deeplinear_vector} are as follows.

\begin{restatable}{theo}{theodeepweightsformvector}\label{theo:deep_weights_vector}
Optimal layer weight for \eqref{eq:problemdef_deeplinear_vector} are
\begin{align*}
    & \weightmat_{l,j}^*=  \begin{cases}t_j^* \tilde{\vec{v}}_{\plantedsub,j}\boldsymbol{\rho}_{1,j}^T\;\text{ if }l=1\\ t_j^*\boldsymbol{\rho}_{l-1,j} \boldsymbol{\rho}_{l,j}^T \; \text{ if } 1<l\leq L-2\\
  \boldsymbol{\rho}_{L-2,j} \;\text{ if } l=L-1\end{cases},
\end{align*}
where $j \in [K]$, $\tilde{\vec{v}}_{\plantedsub,j}$ is the $j\textsuperscript{th}$ maximal right singular vector of $\dualmat^{*^T} \data$ and  $\{\boldsymbol{\rho}_{l,j}\}_{l=1}^{L-2}$ are arbitrary unit norm vectors such that $\boldsymbol{\rho}_{l,j}^T\boldsymbol{\rho}_{l,k}=0,\; \forall j \neq k$.
\end{restatable}

The next theorem formally proves that strong duality holds for the primal problem in \eqref{eq:problemdef_deeplinear_vector}.
\begin{restatable}{theo}{theostrongdualitydeepvector}\label{theo:strong_duality_deep_vector}
Let $\{ \data,\vec{Y}\}$ be feasible for \eqref{eq:problemdef_deeplinear_vector}, then strong duality holds for finite width networks.
\end{restatable}

\subsubsection{Regularized Case}\label{sec:regularized_deeplinear_vector}
We now examine the following regularized problem
\begin{align}\label{eq:problemdef_deeplinear_vector_regularized}
   \min_{\{\theta_l\}_{l=1}^L} \frac{1}{2}\|f_{\theta,L}(\data)-\vec{y}\|_2^2+ \frac{\beta}{2}\sum_{j=1}^m\sum_{l=1}^L\| \weightmat_{l,j} \|_F^2 .
\end{align}
Next result provides a set of optimal solutions to \eqref{eq:problemdef_deeplinear_vector_regularized}.
\begin{restatable}{theo}{theodeepweightsvectorregform}\label{theo:deep_weights_vector_reg}
Optimal layer weights for \eqref{eq:problemdef_deeplinear_vector_regularized} are
\begin{align*}
    & \weightmat_{l,j}^*=\begin{cases} t_j^*\tilde{\vec{v}}_{x,j}\boldsymbol{\rho}_{1,j}^T\; \text{ if } l=1\\
  t_j^* \boldsymbol{\rho}_{l-1,j} \boldsymbol{\rho}_{l,j}^T\;\text{ if } 1<l\leq L-2\\
 \boldsymbol{\rho}_{L-2,j}\; \text{ if } l=L-1\\
  \end{cases},
\end{align*}
where $j \in [K]$, $\tilde{\vec{v}}_{x,j}$ is a maximal right singular vector of $\mathcal{P}_{\data,\beta}(\vec{Y})^T \data$ and $\mathcal{P}_{\data,\beta}(\cdot) \text{ projects to } \{\vec{U}\in \mathbb{R}^{n\times k}\;|\; \sigma_{max}( \vec{U}^T \data ) \leq \beta t_j^{*^{2-L}}\}$. Additionally, $\boldsymbol{\rho}_{l,j}$'s is an orthonormal set. Therefore, the rank of each hidden layer is determined by $\beta$ as in Remark \ref{rem:twolayer_rank}.
\end{restatable}

\section{Deep ReLU networks}
Here, we consider an $L$-layer ReLU network with the output function $f_{\theta,L}(\data)=\vec{A}_{L-1}\weight_L$,
where $\vec{A}_{l,j}=(\vec{A}_{l-1,j}\weightmat_{l,j})_+, $, $\vec{A}_{0,j}=\data$, $\forall l,j$, and $(x)_+=\max\{0,x\}$. Below, we first state the minimum norm training problem and then present our results
\begin{align}\label{eq:problemdef_deeprelu}
    &\min_{\{\theta_l\}_{l=1}^L}  \sum_{j=1}^m\sum_{l=1}^L\| \weightmat_{l,j} \|_F^2  \;\text{ s.t. } f_{\theta,L}(\data)=\vec{y}.
\end{align}


\begin{restatable}{theo}{theodeeprelu}\label{theo:deeprelu}
Let $\data$ be a rank-one matrix such that $\data=\vec{c}\vec{a}_{0}^T$, where $\vec{c}\in\mathbb{R}_+^n$ and $\vec{a}_0 \in \mathbb{R}^d$, then strong duality holds and the optimal weights are
\begin{align*}
    \weightmat_{l,j}=\frac{\boldsymbol{\phi}_{l-1,j}}{\| \boldsymbol{\phi}_{l-1,j}\|_2}\boldsymbol{\phi}_{l,j}^T, \: \forall l \in [L-2],\;\weight_{L-1,j}=\frac{\boldsymbol{\phi}_{L-2,j}}{\| \boldsymbol{\phi}_{L-2,j}\|_2},
\end{align*}
where $\boldsymbol{\phi}_{0,j}=\vec{a}_0$ and $\{\boldsymbol{\phi}_{l,j}\}_{l=1}^{L-2}$ is a set of vectors such that $\boldsymbol{\phi}_{l,j}\in \mathbb{R}_+^{m_l}\,$ and  $\|\boldsymbol{\phi}_{l,j}\|_2=t_j^*,\; \forall l \in [L-2], \forall j \in [m]$.
\end{restatable}

In the sequel, we first examine a two-layer network training problem with bias and then extend this to multi-layer.

\begin{restatable}{theo}{theodeeprelurankone}\label{theo:deeprelu_rankone}
Let $\data$ be a matrix such that $\data=\vec{c}\vec{a}_0^T$, where $\vec{c} \in \mathbb{R}^n$ and $\vec{a}_0 \in \mathbb{R}^d$. Then, when $L=2$, a set of optimal solutions to \eqref{eq:problemdef_deeprelu} is $\{(\weight_i, \bias_i) \}_{i=1}^{m}$, where  $\weight_i=s_i\frac{\vec{a}_0}{\|\vec{a}_0\|_2}, \bias_{i}=- s_i c_{i}\|\vec{a}_0\|_2$ with $s_i=\pm 1, \forall i \in [m]$. 
\end{restatable}

\begin{restatable}{cor}{corrankonekinks}\label{cor:rankone_kinks}
As a result of  Theorem \ref{theo:deeprelu_rankone}, when we have one dimensional data, i.e., $\sample \in \mathbb{R}^n$, an optimal solution to \eqref{eq:problemdef_deeprelu} can be formulated as $\{(\weightscalar_i, \bias_i) \}_{i=1}^{m}$, where  $\weightscalar_i=s_i$, $\bias_{i}=- s_i \samplescalar_{i}$ with $s_i=\pm 1, \forall i \in [m]$. Therefore, the optimal network output has kinks only at the input data points, i.e., the output function is in the following form: $f_{\theta,2}(\hat{\samplescalar})=\sum_i\relu{\hat{\samplescalar}-\samplescalar_i}$. Hence, the network output becomes a linear spline interpolation.
\end{restatable}

We now extend the results in Theorem \ref{theo:deeprelu_rankone} and Corollary \ref{cor:rankone_kinks} to multi-layer ReLU networks.

\begin{restatable}{prop}{propeffectbias}\label{prop:effect_bias}
Theorem \ref{theo:deeprelu} still holds when we add a bias term to the last hidden layer, i.e., $\sum_j\relu{\vec{A}_{L-2,j}\weight_{L-1,j}+\vec{1}_n\bias_j}\weightscalar_{L,j}=\vec{y}$.
\end{restatable}

\begin{restatable}{cor}{cormultilayerrankonekinks}\label{cor:multilayer_rankone_kinks}
As a result of  Theorem \ref{theo:deeprelu_rankone} and Proposition \ref{prop:effect_bias}, for one dimensional data, i.e., $\sample \in \mathbb{R}^n$, the optimal network output has kinks only at the input data points, i.e., the output function is in the following form: $f_{\theta,L}(\hat{\samplescalar})=\sum_i\relu{\hat{\samplescalar}-\samplescalar_i}$. Therefore, the optimal network output is a linear spline interpolation.
\end{restatable}

In Corollary \ref{cor:rankone_kinks} and \ref{cor:multilayer_rankone_kinks}, the optimal output function for multi-layer ReLU networks are linear spline interpolators for rank-one data, which generalizes the two-layer results for one-dimensional data in \cite{infinite_width,parhi_minimum,ergen2020aistats,ergen2020journal} to arbitrary depth.


\subsection{Regularized problem with vector outputs}
We now extend the analysis to regularized training problems with $K$ outputs, i.e., $\vec{Y} \in \mathbb{R}^{n \times K}$.

The result in Theorem \ref{theo:deeprelu} also holds for vector output multi-layer ReLU networks as shown below.
\begin{restatable}{prop}{propdeeprelustrongdualityvector}\label{prop:deeprelu_strong_duality_vector}
Theorem \ref{theo:deeprelu} extends to deep ReLU networks with vector outputs, therefore, the optimal layer weights can be formulated as in Theorem \ref{theo:deeprelu}.
\end{restatable}

\begin{figure*}[t]
\centering
\captionsetup[subfigure]{oneside,margin={1cm,0cm}}
	\begin{subfigure}[t]{0.25\textwidth}
	\centering
	\includegraphics[width=1.0\textwidth, height=1\textwidth]{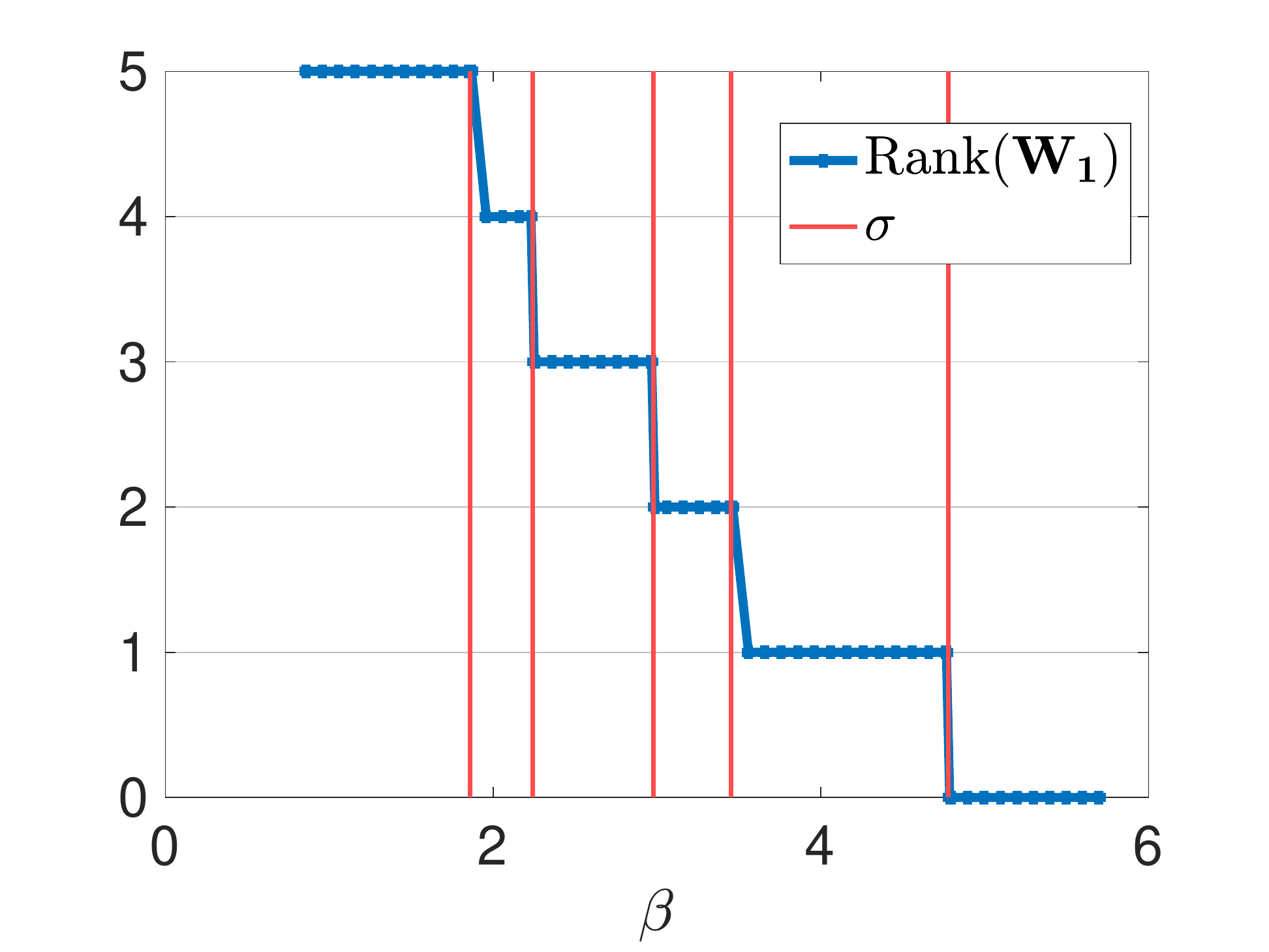}
	\caption{\centering} \label{fig:twolayer_rank}
\end{subfigure} 
	\begin{subfigure}[t]{0.70\textwidth}
	\centering
	\includegraphics[width=1\textwidth, height=0.35\textwidth]{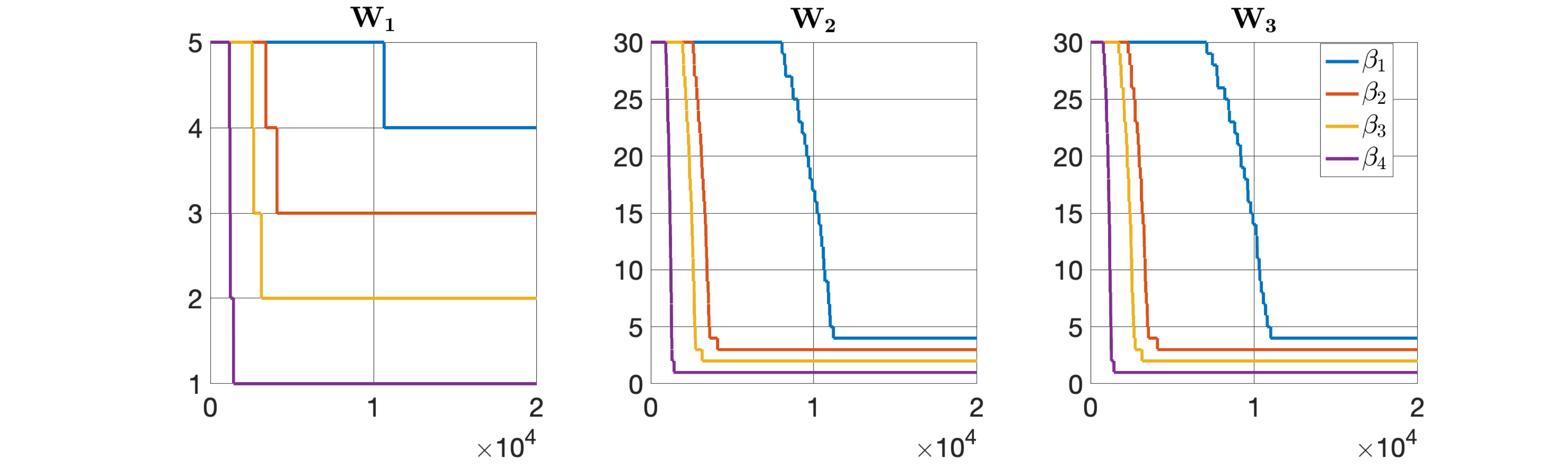}
	\caption{\centering} \label{fig:deeprank_linear}
\end{subfigure} \hspace*{\fill}\vspace{0.1in}
\caption{Verification of Remark \ref{rem:twolayer_rank}. (a) Rank of the hidden layer weight matrix as a function of $\beta$ and (b) rank of the hidden layer weights for different regularization parameters, i.e., $\beta_1<\beta_2<\beta_3<\beta_4$.}\label{fig:linear}
\end{figure*}

\begin{figure*}[t]
\centering
\captionsetup[subfigure]{oneside,margin={1cm,0cm}}
	\begin{subfigure}[t]{0.45\textwidth}
	\centering
	\includegraphics[width=.85\textwidth, height=0.5\textwidth]{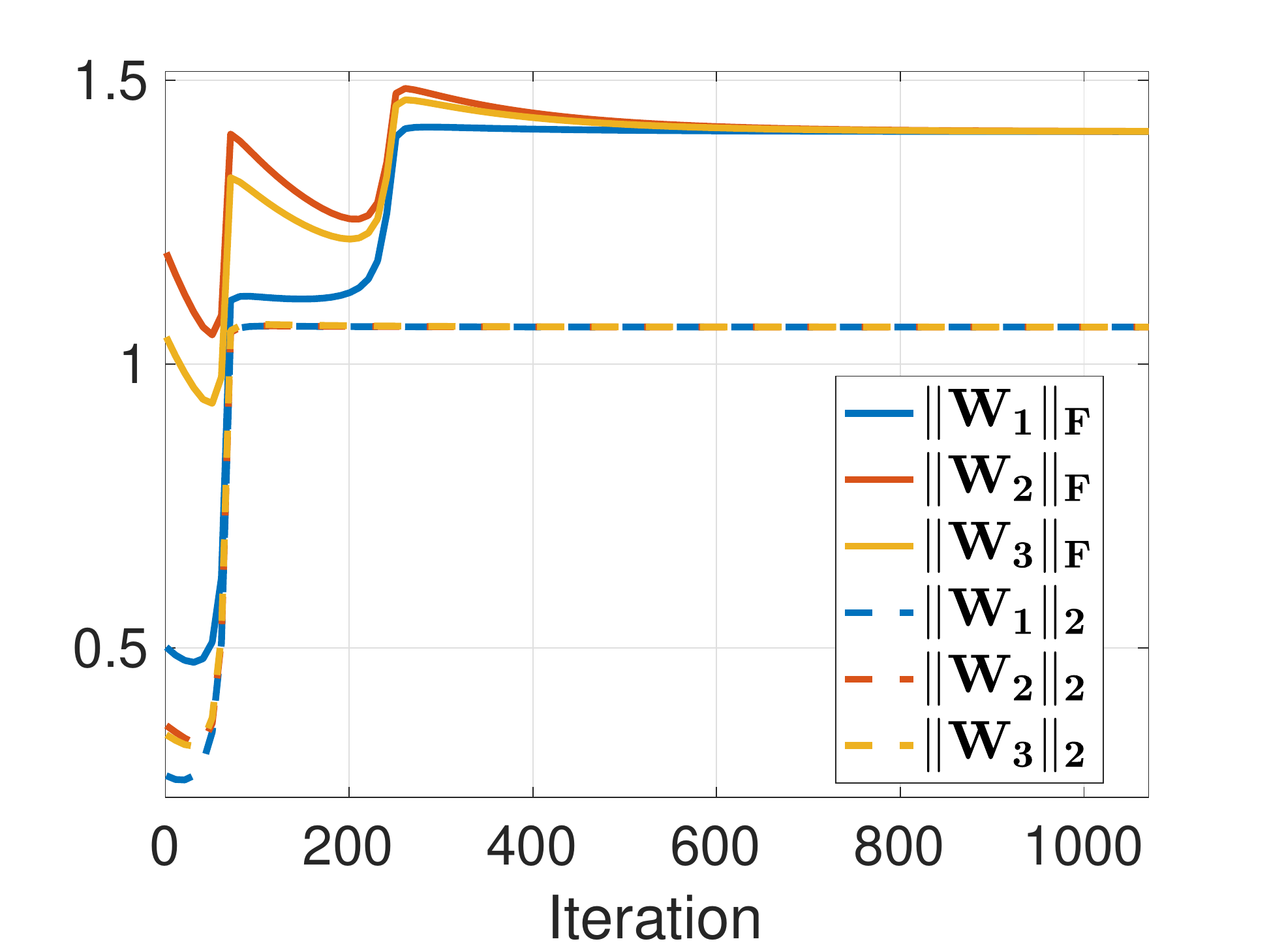}
	\caption{\centering} \label{fig:deepnorm_linear}
\end{subfigure} \hspace*{\fill}
	\begin{subfigure}[t]{0.45\textwidth}
	\centering
	\includegraphics[width=.85\textwidth, height=0.5\textwidth]{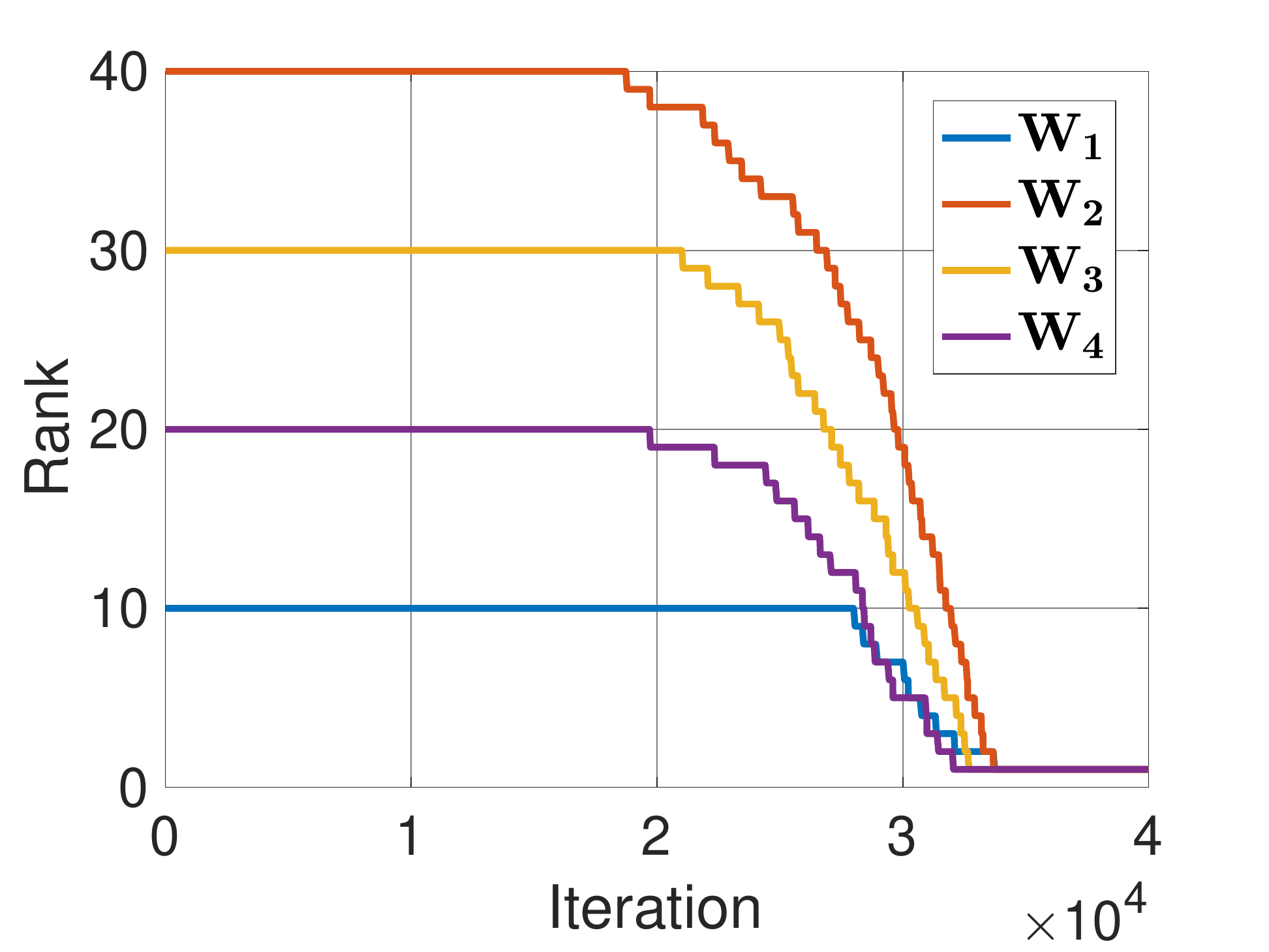}
	\caption{\centering} \label{fig:deeprank_relu}
\end{subfigure} \hspace*{\fill}\vspace{0.1in}
\caption{Verification of Proposition \ref{prop:deep_weightnorm_eq} and \ref{prop:effect_bias}. (a) Evolution of the operator and Frobenius norms for the layer weights of a linear network and (b) Rank of the layer weights of a ReLU network with $K=1$.}\label{fig:newfig}
\end{figure*} 


Now, we extend our characterization to arbitrary rank whitened data matrices and fully characterize the optimal layer weights of a deep ReLU network with $K$ outputs. We also note that one can even obtain closed-form solutions for all the layers weights as proven in the next result.

\begin{restatable}{theo}{theodeepreluwhitevectorclosedform} \label{theo:closedform_regularized_multiclass}
Let $\{\data,\vec{Y}\}$ be a dataset such that $\data\data^T=\vec{I}_n$ and $\vec{Y}$ is one-hot encoded, then a set of optimal solutions for the following regularized training problem
\begin{align} \label{eq:problem_def_regularized_multiclass}
       &\min_{\theta \in \Theta } \frac{1}{2}\|f_{\theta,L}(\data)-\vec{Y}\|_F^2+\frac{\beta}{2}\sum_{j=1}^m\sum_{l=1}^L \|\weightmat_{l,j}\|_F^2
\end{align}
can be formulated as follows
\begin{align*}
    & \weightmat_{l,j}=  \begin{cases}\frac{\boldsymbol{\phi}_{l-1,j}}{\| \boldsymbol{\phi}_{l-1,j}\|_2}\boldsymbol{\phi}_{l,j}^T, \; &\text{ if } l \in [L-1]\\
 \left(\| \boldsymbol{\phi}_{0,j}\|_2-\beta\right)_+\boldsymbol{\phi}_{l-1,j} \vec{e}_r^T\;&\text{ if } l=L\\
    \end{cases},
\end{align*}
 where $\boldsymbol{\phi}_{0,j}=\data^T \vec{y}_j$, $\{\boldsymbol{\phi}_{l,j}\}_{l=1}^{L-2}$ are vectors such that $\boldsymbol{\phi}_{l,j}\in \mathbb{R}_+^{m_l}$, $\|\boldsymbol{\phi}_{l,j}\|_2=t_j^* \text{, and } \boldsymbol{\phi}_{l,i}^T\boldsymbol{\phi}_{l,j}=0, \; \forall i\neq j$, Moreover, $\boldsymbol{\phi}_{L-1,j}=\vec{e}_j$ is the $j^{th}$ ordinary basis vector.
\end{restatable}
\bremark 
We note that the whitening assumption $\data\data^T=\vec{I}_n$ necessitates that $n \leq d$, which might appear to be restrictive. However, this case is common in few-shot classification problems with limited labels \cite{fewshot}. Moreover, it is challenging to obtain reliable labels in problems involving high dimensional data such as in medical imaging \cite{highdim_medical} and genetics \cite{highdim_bio}, where $n\leq d$ is typical. More importantly, SGD employed in deep learning frameworks, e.g., PyTorch and Tensorflow, operate in mini-batches rather than the full dataset. Therefore, even when $n>d$, each gradient descent update can only be evaluated on small batches, where the batch size $n_b$ satisfies $n_b \ll d$. Hence, the $n\leq d$ case implicitly occurs during the training phase.
\eremark
\bremark
We also note that the conditions in Theorem \ref{theo:closedform_regularized_multiclass} are common in practical frameworks. As an example, for image classification, it has been shown that whitening significantly improves the classification accuracy of the state-of-the-art architectures, e.g., ResNets, on benchmark datasets such as ImageNet \cite{whitening}. Furthermore, the label matrix is one hot encoded in image classification. Therefore, in such cases, there is no need to train a deep ReLU network in an end-to-end manner. Instead one can directly use the closed-form formulas in Theorem \ref{theo:closedform_regularized_multiclass}. 
\eremark

\subsection{Regularized problem with Batch Normalization}
We now consider a more practical setting with an arbitrary $L$-layer network and batch normalization \cite{batch_norm}. We first define batch normalization as follows. For the activation matrix $\vec{A}_{l-1} \in \mathbb{R}^{n \times m_{l-1}}$, batch normalization applies to each column $j$ independently as follows
\begin{align*}
    &\bn{\vec{A}_{l-1,j} \weight_{l,j}}= \\
    &\hspace{2cm}\frac{(\vec{I}_{n}-\frac{1}{n}\vec{1}_{n\times n} ) \vec{A}_{l-1,j}\weight_{l,j}}{\|(\vec{I}_{n}-\frac{1}{n}\vec{1}_{n\times n} ) \vec{A}_{l-1,j}\weight_{l,j}\|_2}\bnvarl{l}_j +\frac{\vec{1}_n}{\sqrt{n}}\bnmeanl{l}_j,
\end{align*}
where $\bnvarl{l}_j$ and $\bnmeanl{l}_j$ scales and shifts the normalized value, respectively. The following theorem presents a complete characterization for the last two layers' weights.

\begin{restatable}{theo}{theodeeprelubatchnormclosedform}\label{theo:deep_vector_closedform}
Suppose $\vec{Y}$ is one hot encoded and the network is overparameterized such that the range of $\vec{A}_{L-2,j}$ is $\mathbb{R}^{n}$, then an optimal solution to the following problem\footnote{Notice here we only regularize the last layer's parameters, however, regularizing all the parameters does not change the analysis and conclusion as proven in Appendix \ref{sec:supp_regularization}.}
\begin{align*}
    &\min_{\theta \in \Theta} \frac{1}{2} \normf{ \sum_{j=1}^{m}\relu{\bn{\vec{A}_{L-2,j} \weight_{L-1,j}}}{\weight_{L,j}}^T- \vec{Y} }^2 \nonumber\\
    &+ \frac{\beta}{2} \sum_{j=1}^{m}\left({\bnvarl{L-1}_j}^2+{\bnmeanl{L-1}_j}^2+\norm{\weight_{L,j}}^2 \right),
\end{align*}
can be formulated in closed-form as follows
\begin{align*}
   \begin{split}
       &\left(\weight_{L-1,j}^* ,\weight_{L,j}^*\right)= 
   \left(\vec{A}_{L-2,j}^\dagger \vec{y}_j, \left(\|\vec{y}_j\|_2-\beta\right)_+\vec{e}_j\right) 
     \\
      &\begin{bmatrix} {\bnvarl{L-1}_j}^* \\{ \bnmeanl{L-1}_j}^*\end{bmatrix} =\frac{1}{\|\vec{y}_j\|_2}\begin{bmatrix} \|\vec{y}_j-\frac{1}{n}\vec{1}_{n\times n}\vec{y}_j\|_2\\ \frac{1}{{\sqrt{n}} }\vec{1}_n^T\vec{y}_j \end{bmatrix}    
   \end{split}
\end{align*}
$\forall j \in [K]$, where $\vec{e}_j$ is the $j^{th}$ ordinary basis vector.
\end{restatable}

\begin{figure*}[h]
        \centering
        \begin{subfigure}[b]{0.3\textwidth}
            \centering
            \includegraphics[width=\textwidth]{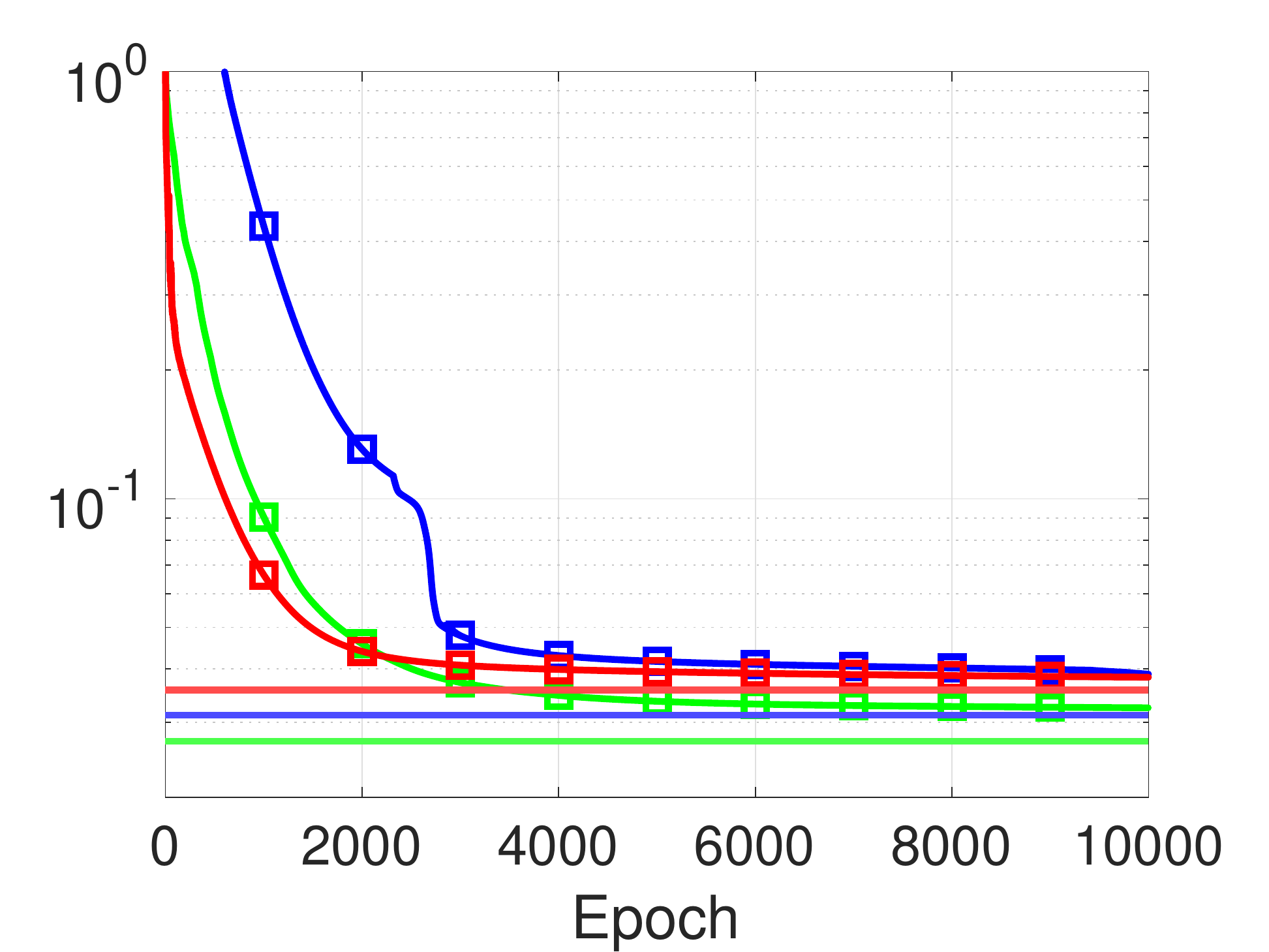}
            \caption{MNIST-Training objective}
        \end{subfigure}
        \hfill
        \begin{subfigure}[b]{0.3\textwidth}  
            \centering 
            \includegraphics[width=\textwidth]{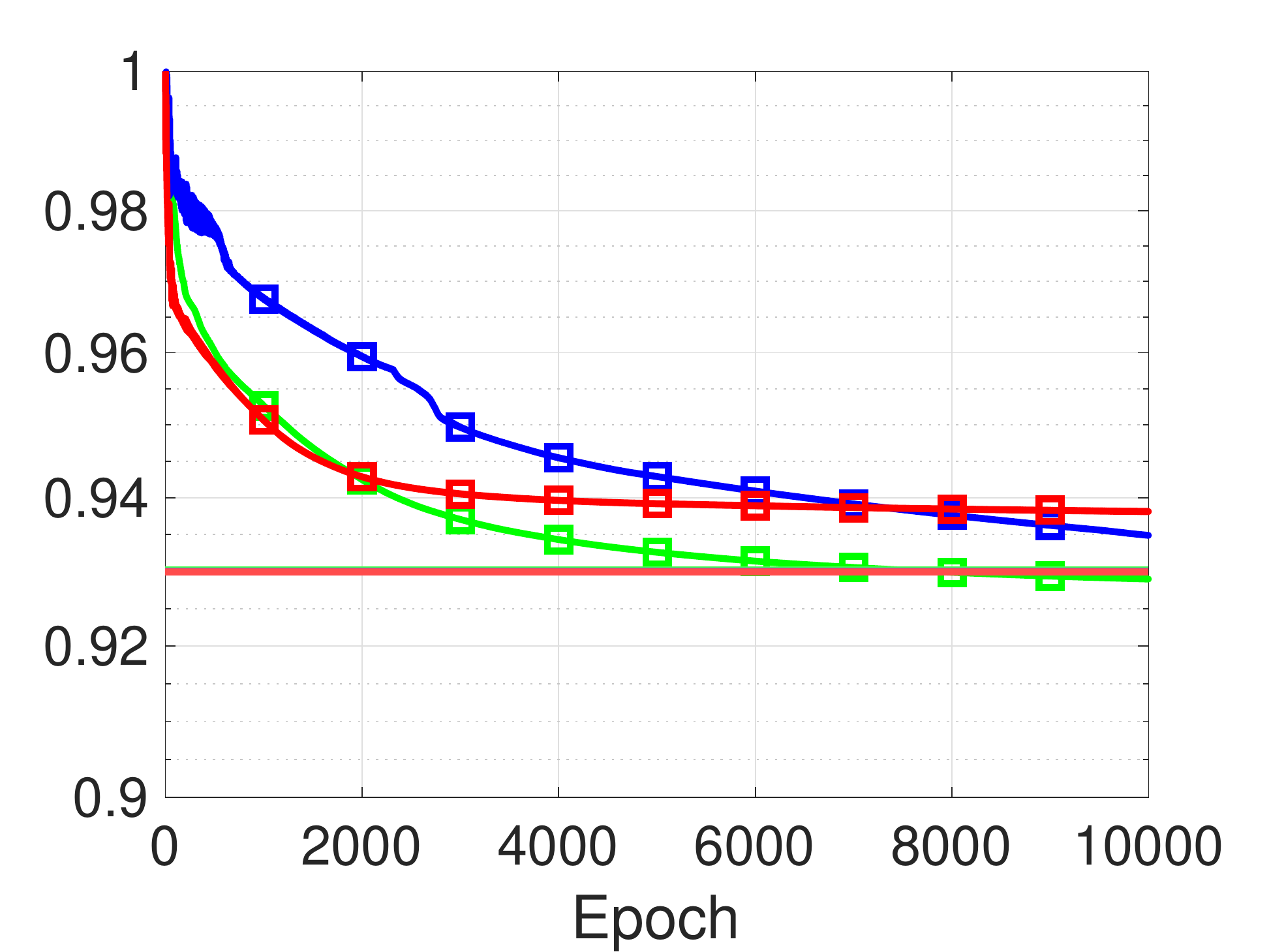}
            \caption{MNIST-Test error}
        \end{subfigure}        \hfill
        \begin{subfigure}[b]{0.3\textwidth}   
            \centering 
            \includegraphics[width=\textwidth]{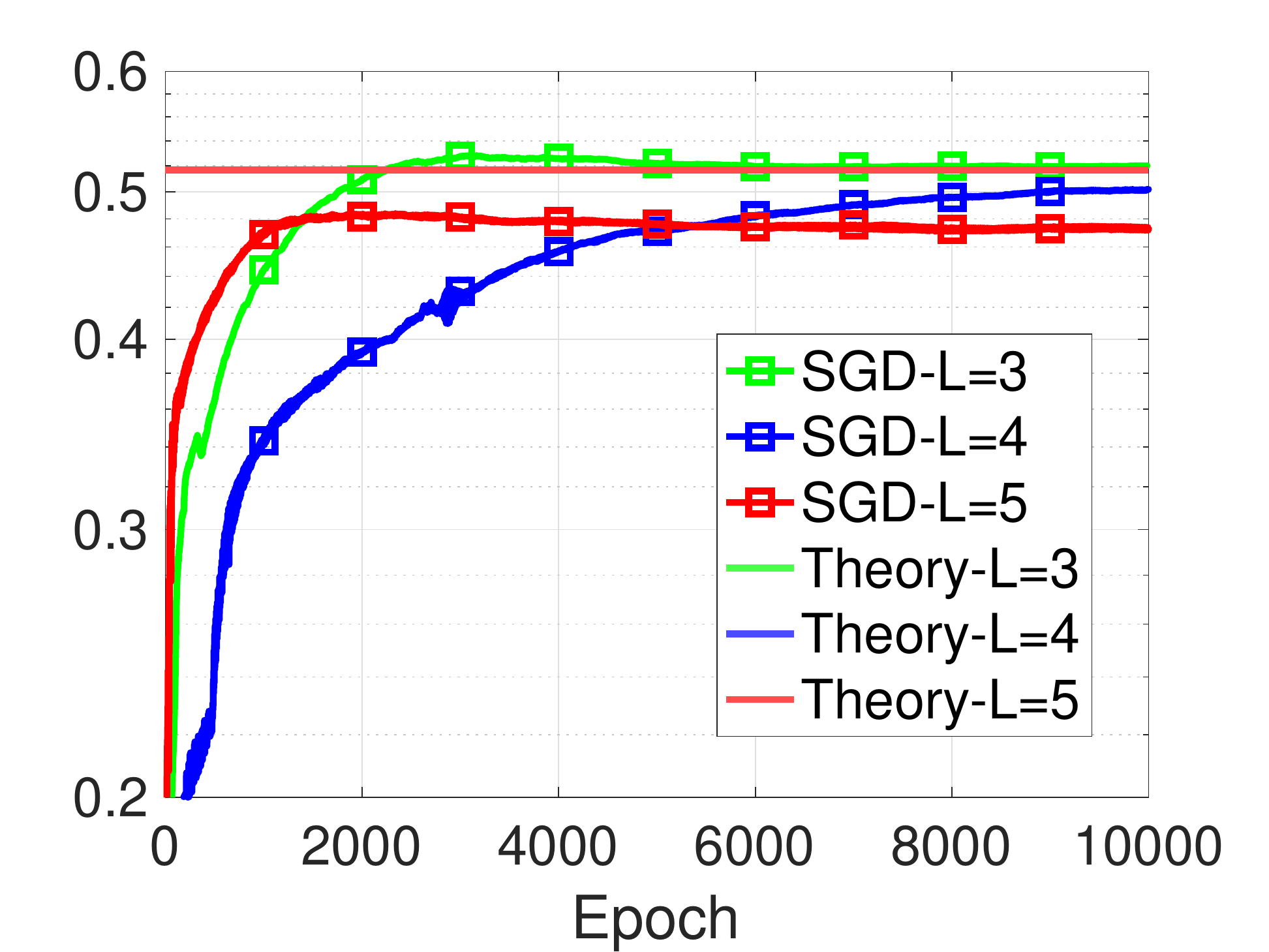}
            \caption{MNIST-Test accuracy}
        \end{subfigure}
        \centering
        \begin{subfigure}[b]{0.3\textwidth}
            \centering
            \includegraphics[width=\textwidth]{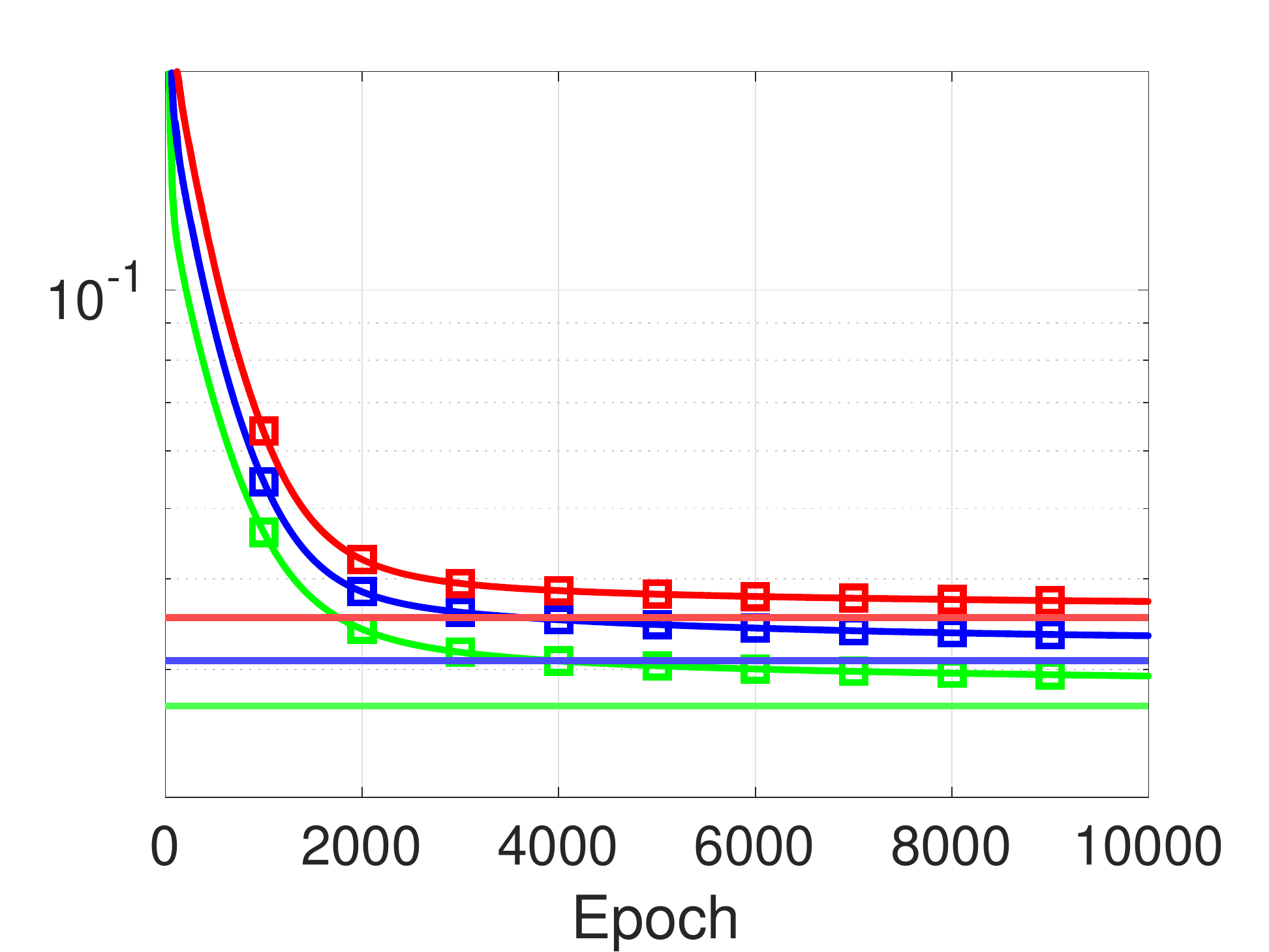}
            \caption{CIFAR-10-Training objective}
        \end{subfigure}
        \hfill
        \begin{subfigure}[b]{0.3\textwidth}  
            \centering 
            \includegraphics[width=\textwidth]{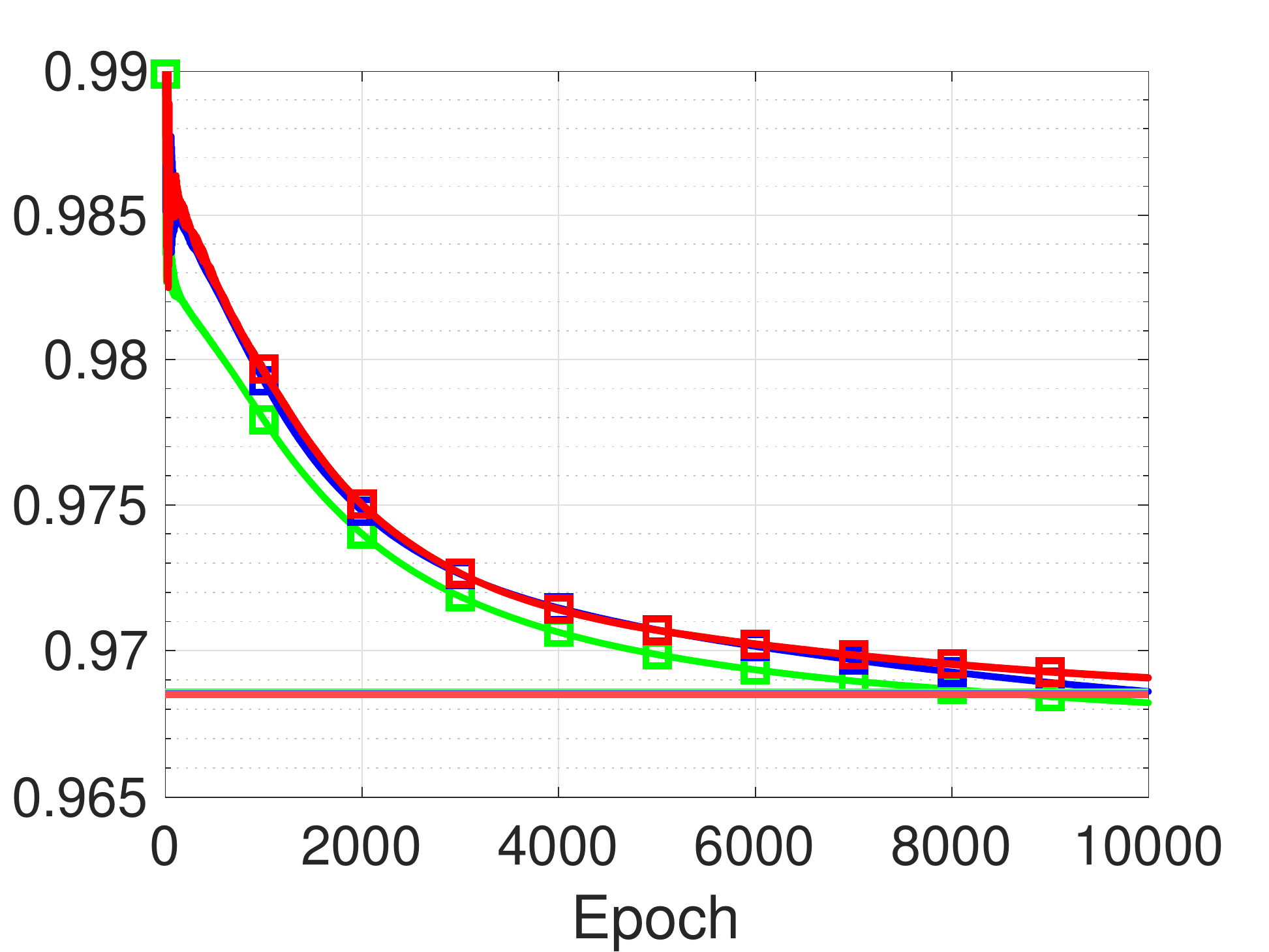}
            \caption{CIFAR-10-Test error}
        \end{subfigure}        \hfill
        \begin{subfigure}[b]{0.3\textwidth}   
            \centering 
            \includegraphics[width=\textwidth]{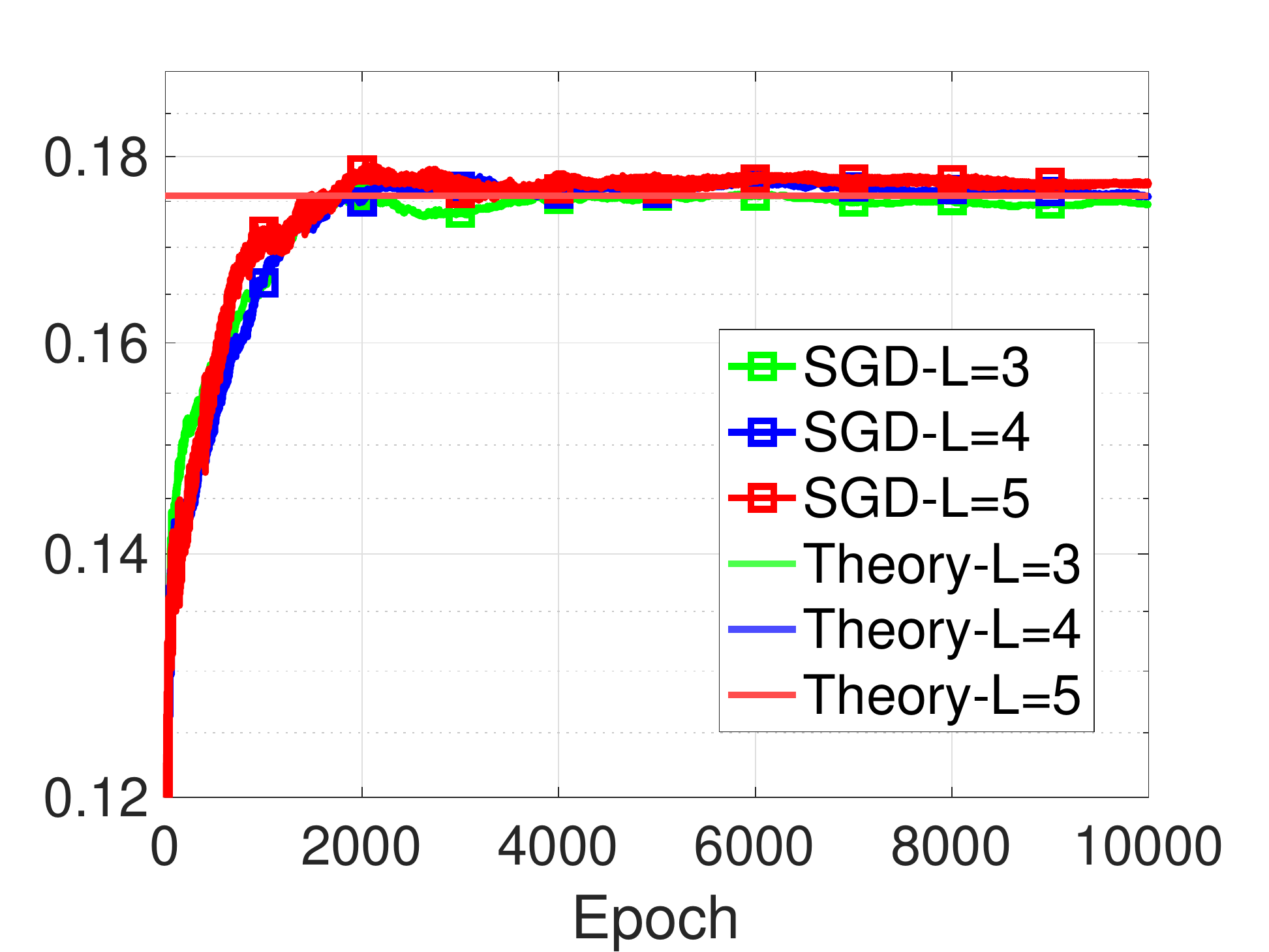}
            \caption{CIFAR-10-Test accuracy}
        \end{subfigure}\vspace{0.1in}
	\caption{Training and test performance on whitened and sampled datasets, where $(n,d)=(60,90)$, $K=10$, $L=3,4,5$ with $50$ neurons per layer and we use squared loss with one hot encoding. For Theory, we use the layer weights in Theorem \ref{theo:closedform_regularized_multiclass}, which achieves the optimal performance as guaranteed by Theorem \ref{theo:closedform_regularized_multiclass}. }\label{fig:mnist_cifar}
    \end{figure*}

\begin{remark}
We note that the results in Theorem \ref{theo:closedform_regularized_multiclass} and \ref{theo:deep_vector_closedform} indicate that whitened data and arbitrary data trained with batch normalization effectively yield the same results in the last layer, i.e., both achieve a scaled version of the labels. The difference is that in Theorem \ref{theo:closedform_regularized_multiclass}, the labels are obtained after the first layer and carried out to the last layer by aligned layer weights. However, in Theorem \ref{theo:deep_vector_closedform}, since batch normalization normalizes layers individually, the scaled labels are obtained after the last hidden layer.
\end{remark}

One-hot encoding is one of the common strategies to convert categorical variables into a binary representation that can be processed by DNNs. Although \cite{papyan2020neuralcollapse} empirically verified the emergence of certain patterns, termed Neural Collapse, for one-hot encoded labels trained with batch normalization, the theory behind these findings are still unknown. Therefore, we first define a new notion of simplex Equiangular Tight Frame (ETF) and then explain the Neural Collapse phenomenon where class means collapse to the vertices of a simplex ETF. We also note that all of our derivations hold for arbitrary convex loss functions, therefore, are also valid for the commonly adopted cross entropy loss as proven in Appendix \ref{sec:supp_general_loss}.

\begin{defns}\label{def:simplex_etf}
A standard simplex ETF is a set of points in $\mathbb{R}^{K}$ selected from the columns of the following matrix
\begin{align*}
    \vec{S}= \sqrt{\frac{K}{K-1}} \left(\vec{I}_K-\frac{1}{K}\vec{1}_{K \times K}\right).
\end{align*}
However, \cite{papyan2020neuralcollapse} also allows rescaling and rotations of $\vec{S}$, i.e., they define a general simplex ETF as $\vec{S}_g=\alpha \vec{U}\vec{S} \in \mathbb{R}^{p \times K}$, where $\vec{U}^T\vec{U}=\vec{I}_K$ and $\alpha \in \mathbb{R}_+$.
\end{defns}

\begin{restatable}{cor}{corneuralcollapse} \label{cor:neural_collapse}
Computing the last hidden layer activations after BN, i.e., $\vec{A}_{L-1} \in \mathbb{R}^{n \times K}$, using the optimal layer weight in Theorem \ref{theo:deep_vector_closedform} and then subtracting their global mean as in \cite{papyan2020neuralcollapse} yields
\begin{align*}
    \left(\vec{I}_n-\frac{1}{n}\vec{1}_{n \times n} \right)\vec{A}_{L-1}=  \sqrt{\frac{K}{n}}\left(\vec{I}_K \otimes \vec{1}_{\frac{n}{K}}-\frac{1}{K}\vec{1}_{n \times K} \right),
\end{align*}
where we assume that samples are ordered, i.e., the first $n/K$ samples belong to class 1, next $n/K$ samples belong to class 2 and so on. Therefore, all the activations for a certain class $k$ are the same and their mean is given by $ (\sqrt{K/n})(\vec{e}_k- \vec{1}_K/K)$, which is the $k^{th}$ column of a general simplex ETF with $\alpha=\sqrt{(K-1)/n}$ and $\vec{U}=\vec{I}_K$.
\end{restatable}

\section{Numerical experiments}
 
Here, we present numerical results to verify our theoretical analysis. We first use synthetic datasets generated from a random data matrix with zero mean and identity covariance and the corresponding output vector is obtained via a randomly initialized teacher network\footnote{Additional numerical results can be found in Appendix \ref{sec:supp_additional_exp}.}. We first consider a two-layer linear network with $\weightmat_1 \in \mathbb{R}^{20 \times 50}$ and $\weightmat_2 \in \mathbb{R}^{50 \times 5}$. To prove our claim in Remark \ref{rem:twolayer_rank}, we train the network using GD with different $\beta$. In Figure \ref{fig:twolayer_rank}, we plot the rank of $\weightmat_1$ as a function of $\beta$, as well as the location of the singular values of $\vec{Y}^T \data $ using vertical red lines. This shows that the rank of the layer changes when $\beta$ is equal to one of the singular values, which verifies Remark \ref{rem:twolayer_rank}. We also consider a four-layer linear network with $\weightmat_{1,j}\in \mathbb{R}^{5 \times 50}$, $\weightmat_{2,j}\in \mathbb{R}^{50 \times 30}$, $\weightmat_{3,j}\in \mathbb{R}^{30 \times 40}$, and $\weightmat_{4,j}\in \mathbb{R}^{40 \times 5}$. We then select different regularization parameters as $\beta_1<\beta_2<\beta_3<\beta_4$. As illustrated in Figure \ref{fig:deeprank_linear}, $\beta$ determines the rank of each weight matrix and the rank is same for all the layers, which matches with our results. Moreover, to verify Proposition \ref{prop:deep_weightnorm_eq}, we choose $\beta$ such that the weights are rank-two. In Figure \ref{fig:deepnorm_linear}, we numerically show that all the hidden layer weight matrices have the same operator and Frobenius norms. We also conduct an experiment for a five-layer ReLU network with $\weightmat_{1,j}\in \mathbb{R}^{10 \times 50}$, $\weightmat_{2,j}\in \mathbb{R}^{50 \times 40}$, $\weightmat_{3,j}\in \mathbb{R}^{40 \times 30}$, $\weightmat_{4,j}\in \mathbb{R}^{30 \times 20}$, and $\weight_{5,j}\in \mathbb{R}^{20 \times 1}$. Here, we use data such that $\data=\vec{c}\vec{a}_0^T$, where $\vec{c} \in \mathbb{R}^n_+$ and $\vec{a}_0 \in \mathbb{R}^d$. In Figure \ref{fig:deeprank_relu}, we plot the rank of each weight matrix, which converges to one as claimed Proposition \ref{prop:effect_bias}.

We also verify our theory on two real benchmark datasets, i.e., MNIST \cite{mnist} and CIFAR-10 \cite{cifar10}. We first randomly undersample and whitened these datasets. We then convert the labels into one hot encoded form. Then, we consider a ten class classification/regression task using three multi-layer ReLU network architectures with $L=3,4,5$. For each architecture, we use SGD with momentum for training and compare the training/test performance with the corresponding network constructed via the closed-form solutions (without any sort of training) in Theorem \ref{theo:closedform_regularized_multiclass}, i.e., denoted as ``Theory''. In Figure \ref{fig:mnist_cifar}, Theory achieves the optimal training objective, which also yields smaller error and higher test accuracy. Thus, we numerically verify the claims in Theorem \ref{theo:closedform_regularized_multiclass}.

\section{Concluding remarks}
We studied regularized DNN training problems and developed an analytic framework to characterize the optimal solutions. We showed that optimal weights can be explicitly formulated as the extreme points of a convex set via the dual problem. We then proved that strong duality holds for both deep linear and ReLU networks and provided a set of optimal solutions. We also extended our derivations to the vector outputs and many other loss functions. More importantly, our analysis shows that when the input data is whitened or rank-one, instead of training an $L$-layer deep ReLU network in an end-to-end manner, one can directly use the closed-form solutions provided in Theorem \ref{theo:deeprelu} and \ref{theo:closedform_regularized_multiclass}. Furthermore, we showed that whitening/rank-one assumptions can be removed via batch normalization (see Theorem \ref{theo:deep_vector_closedform}). After our work, this was also realized by \cite{ergen2021bn}, where the authors proved that batch normalization effectively whitens the input data matrix. As a corollary, we uncovered theoretical reasons behind a recent empirical observation termed Neural Collapse \cite{papyan2020neuralcollapse}. As another corollary, we proved that the kinks of ReLU occur exactly at the input data so that the optimal network outputs linear spline interpolations for one-dimensional datasets, which was previously known only for two-layer networks \cite{infinite_width,parhi_minimum,ergen2020aistats,ergen2020journal}.

As the limitation of this work, we note that for networks with more than two-layers (i.e., $L>2$), we use a non-standard architecture, where each layer consists of $m$ weight matrices. Thus, we are able to achieve strong duality which is essential for our analysis. We leave the strong duality analysis of standard deep networks as an open research problem for future work.

\section*{Acknowledgements}
This work was partially supported by the National Science Foundation under grants IIS-1838179 and ECCS-2037304, Facebook Research, Adobe Research and Stanford SystemX Alliance.

\bibliographystyle{icml2021}
\bibliography{references}

\clearpage
\onecolumn
\appendix
\addcontentsline{toc}{section}{Appendix} 
\part{Appendix} 
\parttoc 

\section{Appendix}
Here, we present additional materials and proofs of the main results that are not included in the main paper due to the page limit. We also restate each result before the corresponding proof for the convenience of the reader. We also provide a table for notations below.
\begin{table}[h]
\caption{Notations and variables in this paper.}
\label{tab:notations}
\begin{center}
\begin{small}
\begin{tabular}{lcccr}
\toprule
Notation & Description \\
\midrule
$\data \in \mathbb{R}^{n \times d}$    & Data matrix  \\
$\vec{y} \in \mathbb{R}^n$,$\vec{Y} \in \mathbb{R}^{n \times K}$ & Label vector and matrix \\
$\weightmat_{l,j} \in \mathbb{R}^{m_{l-1} \times m_l}$    & $l^{th}$ layer weight matrix\\
$\vec{A}_l \in \mathbb{R}^{n \times m_l}$    & $l^{th}$ layer activation matrix\\
$\dual \in \mathbb{R}^{n}$,$\dualmat \in \mathbb{R}^{n \times K}$ & Dual vector and matrix\\
$\weight^* \in \mathbb{R}^d$,$\weightmat^* \in \mathbb{R}^{d \times K}$ & Optimal weight vector and matrix\\
$r$& Rank of $\data$\\
$\vec{U}_x \boldsymbol{\Sigma}_x \vec{V}_x^T$ & Full SVD of $\data$\\
$\vec{e}_j$ & $j^{th}$ ordinary basis vector \\
$\mathcal{L}(\cdot,\cdot)$ & Arbitrary convex loss function\\
$f_{\theta,L}(\data)$ & Output of an $L$-layer network\\
\bottomrule
\end{tabular}
\end{small}
\end{center}
\end{table}
\subsection{General loss functions}\label{sec:supp_general_loss}
In this section, we show that our extreme point characterization holds for arbitrary convex loss functions including cross entropy and hinge loss. We first restate the primal training problem after applying the rescaling in Lemma \ref{lemma:scaling_deep_overview} as follows
\begin{align} \label{eq:deep_general_loss}
   \min_{\{\theta_l\}_{l=1}^L, t_j} \mathcal{L}( f_{\theta,L}(\data),\vec{y})+ \beta \|\weight_L\|_1+ \frac{\beta}{2}(L-2)\sum_{j=1}^m t_j^2 \text{ s.t. } \weight_{L-1,j} \in \ball_2, \|\weightmat_{l,j} \|_F \leq t_j, \; \forall l \in [L-2]  , \forall j \in [m],
\end{align}
where $\mathcal{L}(\cdot,\vec{y})$ is a convex loss function. 
\begin{restatable}{theo}{theodeepgenericdual}\label{theo:deep_generic_dual}
The dual of \eqref{eq:deep_general_loss} is given by
\begin{align*}
\min_{t_j}\max_{\vec{\dual}}- \mathcal{L}^*(\dual) +\frac{\beta}{2}(L-2)\sum_{j=1}^m t_j^2   \text{ s.t. }\max_{\substack{ \weight_{L-1,j} \in \ball_2 \\\|\weightmat_{l,j} \|_F \leq t_j}} \| \vec{A}_{L-1}^T \vec{\dual}\|_{\infty}\leq \beta\,,\end{align*}
where $\mathcal{L}^*$ is the Fenchel conjugate function defined as
\begin{align*}
\mathcal{L}^*(\dual) = \max_{\vec{z}} \vec{z}^T \dual - \mathcal{L}(\vec{z},\vec{y})\,.
\end{align*}
\end{restatable}

\begin{proof}[\textbf{Proof of Theorem \ref{theo:deep_generic_dual}}]
The proof directly follows from the dual derivation in Appendix \ref{sec:supp_dual_derivations}.
\end{proof}
Theorem \ref{theo:deep_generic_dual} proves that our extreme point characterization applies to arbitrary loss function. Therefore, optimal parameters for \eqref{eq:deep_general_loss} are a subset of the same extreme point set, i.e., determined by the input data matrix $\data$, independent of loss function.
\bremark
Since our characterization is generic in the sense that it holds for vector output, deep linear and deep ReLU networks (see the main paper for details), Theorem \ref{theo:deep_generic_dual} is also valid for all of these cases.
\eremark

\subsection{Derivations for the dual problem in \eqref{eq:dual_deeplinear_overview}}\label{sec:supp_dual_derivations}
We first restate the scaled primal problem in Lemma \ref{lemma:scaling_deep_overview}
\begin{align}\label{eq:overview_primal_supp}
   P^*= &\min_{\{\theta_l\}_{l=1}^L, t_j,\hat{\vec{y}}} \mathcal{L}(\hat{\vec{y}} ,\vec{y})+ \beta \|\weight_L\|_1+ \frac{\beta}{2}(L-2)\sum_{j=1}^m t_j^2 ~ \text{ s.t. } \begin{split}
     & \weight_{L-1,j} \in \ball_2, \|\weightmat_{l,j} \|_F \leq t_j, \; \forall l \in [L-2],\forall j \in [m]\; \\ &\hat{\vec{y}}=f_{\theta,L}(\data).
   \end{split}
\end{align}
Then, the corresponding Lagrangian is 
\begin{align*}
    L(\dual,\hat{\vec{y}},\weight_L)=\mathcal{L}(\hat{\vec{y}},\vec{y})- \dual^T \hat{\vec{y}} +\dual^T f_{\theta,L}(\data)+ \beta \|\weight_L\|_1+ \frac{\beta}{2}(L-2)\sum_{j=1}^m t_j^2.
\end{align*}
Based on the Lagrangian above, we now obtain the dual function as follows
\begin{align*}
    g(\dual)&= \min_{\hat{\vec{y}},\weight_L}L(\dual,\hat{\vec{y}},\weight_L)\\
    &=\min_{\hat{\vec{y}},\weight_L} \mathcal{L}(\hat{\vec{y}},\vec{y})- \dual^T \hat{\vec{y}} +\dual^T f_{\theta,L}(\data)+ \beta \|\weight_L\|_1+ \frac{\beta}{2}(L-2)\sum_{j=1}^m t_j^2 \\
  &=\min_{\hat{\vec{y}},\weight_L} \mathcal{L}(\hat{\vec{y}},\vec{y})- \dual^T \hat{\vec{y}} +\dual^T  \vec{A}_{L-1}\weight_L+ \beta \|\weight_L\|_1+ \frac{\beta}{2}(L-2)\sum_{j=1}^m t_j^2 \\ 
    &=-\mathcal{L}^*(\dual)+\frac{\beta}{2}(L-2)\sum_{j=1}^m t_j^2   \text{ s.t. } \| \vec{A}_{L-1}^T \vec{\dual}\|_{\infty}\ \leq \beta,
\end{align*}
where $\mathcal{L}^*$ is the Fenchel conjugate function defined as \cite{boyd_convex}
\begin{align*}
\mathcal{L}^*(\dual) = \max_{\vec{z}} \vec{z}^T \dual - \mathcal{L}(\vec{z},\vec{y})\,.
\end{align*}
Thus, taking the dual of \eqref{eq:overview_primal_supp} in terms of $\weight_L$ and $\hat{\vec{y}}$ yield
\begin{align*}
        \begin{split}
             &P^*= \min_{\{\theta_l\}_{l=1}^{L-1}, t_j}\max_{\dual} g(\dual)\\
             &\text{ s.t. } \weight_{L-1,j} \in \ball_2, \|\weightmat_{l,j} \|_F \leq t_j, \; \forall l, j
        \end{split}=
             \begin{split}
            &\min_{\{\theta_l\}_{l=1}^{L-1}, t_j} \max_{\dual} -\mathcal{L}^*(\dual)+\frac{\beta}{2}(L-2)\sum_{j=1}^m t_j^2   \\
            &\text{ s.t. } \weight_{L-1,j} \in \ball_2, \|\weightmat_{l,j} \|_F \leq t_j, \; \forall l , j ,\;\| \vec{A}_{L-1}^T \vec{\dual}\|_{\infty}\ \leq \beta.
        \end{split}
\end{align*}
To achieve the lower bound in the main paper, we now change the order of min (for the layer weights)-max as follows
\begin{align*}
        P^*\geq D^* &=\min_{ t_j}\max_{\vec{\dual}}\min_{\substack{ \weight_{L-1,j} \in \ball_2 \\\|\weightmat_{l,j} \|_F \leq t_j}} - \mathcal{L}^*(\dual) +\frac{\beta}{2}(L-2)\sum_{j=1}^m t_j^2    ~\text{ s.t. } \| \vec{A}_{L-1}^T \vec{\dual}\|_{\infty}\ \leq \beta\\
        &=\min_{ t_j}\max_{\vec{\dual}} - \mathcal{L}^*(\dual) +\frac{\beta}{2}(L-2)\sum_{j=1}^m t_j^2   ~ \text{ s.t. }\max_{\substack{ \weight_{L-1,j} \in \ball_2 \\\|\weightmat_{l,j} \|_F \leq t_j}} \| \vec{A}_{L-1}^T \vec{\dual}\|_{\infty}\ \leq \beta,
\end{align*}
which completes the derivation.

\subsection{Equivalence (Rescaling) lemmas for the non-convex objectives }\label{sec:supp_equivalence}
 In this section, we present all the equivalence (scaling transformation) lemmas we used in the main paper and the the proofs are presented in Appendix \ref{sec:proofs_twolayer_linear}, \ref{sec:proofs_deep_linear}, and \ref{sec:proofs_deep_relu}, two-layer, deep linear, and deep ReLU networks, respectively. We also note that similar scaling techniques were also utilized in \cite{neyshabur_reg,infinite_width,ergen2019cutting,ergen2020aistats,ergen2020journal,ergen2020workshop}.
 

 \equivalenceoverview*

 \begin{proof}[\textbf{Proof of Lemma \ref{lemma:scaling_deep_overview}}]
 For any $\theta \in \Theta$, we can rescale the parameters as $\bar{\weight}_{L-1,j}=\alpha_j\weight_{L-1,j}$ and $\bar{\weightscalar}_{L,j}= \weightscalar_{L,j}/\alpha_j$, for any $\alpha_j>0$. Then, the network output becomes
\begin{align*}
    f_{\bar{\theta},L}(\data)=\sum_{j=1}^m\relu{\relu{\data \weightmat_{1,j}} \ldots \bar{\weight}_{L-1,j}}\bar{\weightscalar}_{L,j}=\sum_{j=1}^m\relu{\relu{\data \weightmat_{1,j}} \ldots \weight_{L-1,j}}\weightscalar_{L,j}=f_{\theta,L}(\data),
\end{align*}
which proves that this scaling does not change the output of the network. In addition to this, we have the following basic inequality
\begin{align*}
   \sum_{j=1}^m\sum_{l=1}^L \|\weightmat_{l,j}\|_F^2 \geq \sum_{j=1}^m\sum_{l=1}^{L-2} \|\weightmat_l\|_F^2+2\sum_{j=1}^m |\weightscalar_{L,j}| \text{ }\| \weight_{L-1,j}\|_2,
\end{align*}
where the equality is achieved with the scaling choice $\alpha_j=\big(\frac{|\weightscalar_{L,j}|}{\| \weight_{L-1,j}\|_2}\big)^{\frac{1}{2}}$ is used. Since the scaling operation does not change the right-hand side of the inequality, we can set $\|\weight_{L-1,j} \|_2=1, \forall j$. Therefore, the right-hand side becomes $\| \weight_L\|_1$.

Now, let us consider a modified version of the problem, where the unit norm equality constraint is relaxed as $\| \weight_{L-1,j} \|_2 \leq 1$. Let us also assume that for a certain index $j$, we obtain  $\| \weight_{L-1,j} \|_2 < 1$ with $\weightscalar_{L,j}\neq 0$ as an optimal solution. This shows that the unit norm inequality constraint is not active for $\weight_{L-1,j}$, and hence removing the constraint for $\weight_{L-1,j}$ will not change the optimal solution. However, when we remove the constraint, $\| \weight_{L-1,j}\|_2 \rightarrow \infty$ reduces the objective value since it yields $\weightscalar_{L,j}=0$. Therefore, we have a contradiction, which proves that all the constraints that correspond to a nonzero $\weightscalar_{L,j}$ must be active for an optimal solution. This also shows that replacing $\|\weight_{L-1,j}\|_2=1$ with $\| \weight_{L-1,j} \|_2 \leq 1$ does not change the solution to the problem.

Then, we use the epigraph form for the sum of the norm of the first $L-2$ layers to achieve the equivalence, i.e., we introduce $\|\weightmat_{l,j}\|_F \leq t_j$ constraint and replace $\sum_{l=1}^{L-2}\|\weightmat_{l,j}\|_F^2$ with $(L-2)t_j^2$ in the objective. We also note that since the optimal layer weights have the same Frobenius norm as proven in Proposition \ref{prop:deep_weightnorm_eq}, we can replace the Frobenius norm of each layer weight matrix with the same variable $t$ without loss of generality.
 \end{proof}


\begin{restatable}{lem}{lemmascaling} \label{lemma:scaling}
The following two problems are equivalent:
\begin{align*}
  \begin{split}
      &\min_{\theta \in \Theta} \|\weightmat_1\|_F^2 + \| \weight_2\|_2^2 \\
      &\text{ s.t. } f_{\theta,2}(\data)=\vec{y}
  \end{split}\qquad\qquad = \begin{split}
     &\min_{\theta \in \Theta} \| \weight_2 \|_1 \\& \text{ s.t. } f_{\theta,2}(\data)=\vec{y},  \weight_{1,j} \in \ball_2
  \end{split}.
\end{align*}
\end{restatable}

\begin{proof}[\textbf{Proof of Lemma \ref{lemma:scaling}}]
For any $\theta \in \Theta$, we can rescale the parameters as $\bar{\weight}_{1,j}=\alpha_j\weight_{1,j}$ and $\bar{\weightscalar}_{2,j}= \weightscalar_{2,j}/\alpha_j$, for any $\alpha_j>0$. Then, the network output becomes
\begin{align*}
    f_{\bar{\theta},2}(\data)=\sum_{j=1}^m \bar{\weightscalar}_{2,j} \data \bar{\weight}_{1,j}=\sum_{j=1}^m \frac{\weightscalar_{2,j}}{\alpha_j} \alpha_j \data \weight_{1,j}=\sum_{j=1}^m \weightscalar_{2,j} \data \weight_{1,j},
\end{align*}
which proves $f_{\theta,2}(\data)=f_{\bar{\theta},2}(\data)$. In addition to this, we have the following basic inequality
\begin{align*}
   \frac{1}{2} \sum_{j=1}^m (\weightscalar_{2,j}^2+\| \weight_{1,j}\|_2^2) \geq \sum_{j=1}^m (|\weightscalar_{2,j}| \text{ }\| \weight_{1,j}\|_2),
\end{align*}
where the equality is achieved with the scaling choice $\alpha_j=\big(\frac{|\weightscalar_{2,j}|}{\| \weight_{1,j}\|_2}\big)^{\frac{1}{2}}$ is used. Since the scaling operation does not change the right-hand side of the inequality, we can set $\|\weight_{1,j} \|_2=1, \forall j$. Therefore, the right-hand side becomes $\| \weight_2\|_1$.

Now, let us consider a modified version of the problem, where the unit norm equality constraint is relaxed as $\| \weight_{1,j} \|_2 \leq 1$. Let us also assume that for a certain index $j$, we obtain  $\| \weight_{1,j} \|_2 < 1$ with $\weightscalar_{2,j}\neq 0$ as an optimal solution. This shows that the unit norm inequality constraint is not active for $\weight_{1,j}$, and hence removing the constraint for $\weight_{1,j}$ will not change the optimal solution. However, when we remove the constraint, $\| \weight_{1,j}\|_2 \rightarrow \infty$ reduces the objective value since it yields $\weightscalar_{2,j}=0$. Therefore, we have a contradiction, which proves that all the constraints that correspond to a nonzero $\weightscalar_{2,j}$ must be active for an optimal solution. This also shows that replacing $\|\weight_{1,j}\|_2=1$ with $\| \weight_{1,j} \|_2 \leq 1$ does not change the solution to the problem.
\end{proof}


\begin{restatable}{lem}{lemmascalingvector}\label{lemma:scaling_vector}
The following problems are equivalent:
\begin{align*}
\begin{split}
    &\min_{\theta \in \Theta}  \|\weightmat_1\|_F^2+\|\weightmat_2\|_F^2 \\&\text{ s.t. } f_{\theta,2}(\data)=\vec{Y}
\end{split}
    \qquad\qquad = \begin{split}
         &\min_{\theta \in \Theta}\sum_{j=1}^m\| \weight_{2,j} \|_2 \\ &\text{ s.t. }f_{\theta,2}(\data)=\vec{Y}, \weight_{1,j} \in \ball_2,\forall j
\end{split}.
\end{align*}
\end{restatable}
\begin{proof}[\textbf{Proof of Lemma \ref{lemma:scaling_vector}}]
The proof directly follows from Proof of Lemma \ref{lemma:scaling} using the following inequality
\begin{align*}
   \frac{1}{2} \sum_{j=1}^m (\|\weight_{2,j}\|_2^2+\| \weight_{1,j}\|_2^2) \geq \sum_{j=1}^m (\|\weight_{2,j}\|_2 \text{ }\| \weight_{1,j}\|_2).
\end{align*}
Then, if we set $\|\weight_{1,j} \|_2=1, \forall j$, the right-hand side becomes $\sum_{j=1}^m\| \weight_{2,j}\|_2$.
\end{proof}


\begin{restatable}{lem}{lemmascalingdeep}\label{lemma:scaling_deep}
The following problems are equivalent:
\begin{align*}
\begin{split}
      &\min_{\{\theta_l\}_{l=1}^L} \frac{1}{2} \sum_{j=1}^m\sum_{l=1}^L\| \weightmat_{l,j} \|_F^2  \\&\text{ s.t. } f_{\theta,L}(\data)=\vec{y} 
\end{split}
    \qquad = \begin{split}
         &\min_{\{\theta_l\}_{l=1}^L, \{t_j\}_{j=1}^{m}} \|\weight_L\|_1+ \frac{1}{2}(L-2)\sum_{j=1}^m t_j^2 \\&\text{ s.t. }f_{\theta,L}(\data)=\vec{y} ,\; \weight_{L-1,j} \in \ball_2,  \|\weightmat_{l,j} \|_F \leq t_j, \; \forall l \in [L-2], \forall j \in [m]  
\end{split}.
\end{align*}

\end{restatable}
\begin{proof}[\textbf{Proof of Lemma \ref{lemma:scaling_deep}}]
Applying the rescaling in Lemma \ref{lemma:scaling} to the last two layers of the $L$-layer network in \eqref{eq:problemdef_deeplinear} gives
\begin{align*}
\begin{split}
         &\min_{\{\theta_l\}_{l=1}^L} \|\weight_L\|_1+ \frac{1}{2}\sum_{j=1}^m\sum_{l=1}^{L-2}\| \weightmat_{l,j} \|_F^2  \\&\text{ s.t. }\|\weight_{L-1,j}\|_2 \leq 1,\forall j \in [m],\;\sum_{j=1}^m\data \weightmat_{1,j} \ldots \weight_{L-1,j}\weightscalar_{L,j}=\vec{y} 
\end{split}.
\end{align*}
Then, we use the epigraph form for the norm of the first $L-2$ to achieve the equivalence.
\end{proof}


\begin{restatable}{lem}{lemmascalingdeepvector}\label{lemma:scaling_deep_vector}
The following problems are equivalent:
\begin{align*}
\begin{split}
    &\min_{\{\theta_l\}_{l=1}^L}  \frac{1}{2}\sum_{j=1}^m \sum_{l=1}^L\| \weightmat_{l,j} \|_F^2 \\& \text{ s.t. } f_{\theta,L}(\data)=\vec{Y}
\end{split}
    \qquad = \begin{split}
         &\min_{\{\theta_l\}_{l=1}^L, \{t_l\}_{l=1}^{L-2}} \sum_{j=1}^{m}\|\weight_{L,j}\|_2+ \frac{1}{2}(L-2)\sum_{j=1}^{m}t_j^2 \\&\text{ s.t. }f_{\theta,L}(\data)=\vec{Y} ,\;  \weight_{L-1,j} \in \ball_2, \|\weightmat_{l,j} \|_F \leq t_j, \; \forall l \in [L-2] ,\, \forall j \in [m]
\end{split}.
\end{align*}
\end{restatable}
\begin{proof}[\textbf{Proof of Lemma \ref{lemma:scaling_deep_vector}}]
Applying the rescaling in Lemma \ref{lemma:scaling} to the last two layer of the $L$-layer network in \eqref{eq:problemdef_deeplinear_vector} gives
\begin{align*}
\begin{split}
         &\min_{\{\theta_l\}_{l=1}^L} \sum_{j=1}^{m}\|\weight_{L,j}\|_2+\frac{1}{2} \sum_{j=1}^m\sum_{l=1}^{L-2}\| \weightmat_l \|_F^2  \\&\text{ s.t. }\|\weight_{L-1,j}\|_2 \leq 1,\forall j \in [m],\; \sum_{j=1}^m\data \weightmat_{1,j} \ldots \weight_{L-1,j}\weight_{L,j}^T=\vec{Y} 
\end{split}.
\end{align*}
Then, we use the epigraph form for the norm of the first $L-2$ to achieve the equivalence.
\end{proof}

\subsection{Regularization in Theorem \ref{theo:deep_vector_closedform}} \label{sec:supp_regularization}
In this section, we prove that regularizing the all the parameters do not alter the claims in Theorem \ref{theo:deep_vector_closedform}. We first state the primal problem, where all the parameters are regularized, as follows
\begin{align}\label{eq:deep_primal_fullreg}
    &P_{r}^*=\min_{\theta \in \Theta} \frac{1}{2}\normf{\sum_{j=1}^{m}\relu{\bn{\vec{A}_{L-2,j} \weight_{L-1,j}}}\weight_{L,j}^T-  \vec{Y}}^2+ \frac{\beta}{2}\sum_{j=1}^m\sum_{l=1}^{L}\left(\norm{\bnvarvecl{l}_j}^2+\norm{\bnmeanvecl{l}_j}^2+\normf{\weightmat_{l,j}}^2 \right),
\end{align}
where we use $\bnvarvecl{L}=\bnmeanvecl{L}=\vec{0}$ as dummy variables for notational simplicity. Now, we rewrite \eqref{eq:deep_primal_fullreg} as 
\begin{align*}
    &P_{r}^*=\min_{t\geq0}\min_{\theta \in \Theta} \frac{1}{2}\normf{\sum_{j=1}^{m}\relu{\bn{\vec{A}_{L-2,j} \weight_{L-1,j}}}\weight_{L,j}^T-  \vec{Y}}^2+  \frac{\beta}{2}\sum_{j=1}^{m}\left({\bnvarl{L-1}_j}^2+{\bnmeanl{L-1}_j}^2+\norm{\weight_{L,j}}^2 \right) +\frac{\beta}{2}t\\
    &\text{s.t. } \sum_{j=1}^m\sum_{l=1}^{L-2}\left(\norm{\bnvarvecl{l}_j}^2+\norm{\bnmeanvecl{l}_j}^2+\normf{\weightmat_{l,j}}^2 \right)+\normf{\weightmat_{L-1}}^2 \leq t.
\end{align*}
After applying the scaling between $\weightmat_L$ and $(\bnvarvecl{L-1},\bnmeanvecl{L-1})$ as in Lemma \ref{lemma:scaling_deep_vector}, we take the dual with respect to $\weightmat_{L}$ to obtain the following problem
\begin{align}\label{eq:deep_dual_fullreg}
       &P_{r}^* \geq  D_{r}^*=\max_{t\geq 0}\max_{\dualmat}   -\frac{1}{2}\normf{\dualmat - \vec{Y}}^2 + \frac{1}{2}\normf{\vec{Y}}^2 +\frac{\beta}{2}t\\ \nonumber
       &\text{ s.t. } \max_{\theta \in \Theta_{r} }  \norm{ \dualmat^T\relu{\bn{\vec{A}_{L-2,j} \weight_{L-1,j}}}} \leq \beta,
\end{align}
where $\Theta_{r}=\{\theta \in \Theta :  {\bnvarl{L-1}_j}^2+{\bnmeanl{L-1}_j}^2 =1,\, \forall j \in [m],\, \sum_{l=1}^{L-2}\left(\norm{\bnvarvecl{l}}^2+\norm{\bnmeanvecl{l}}^2+\normf{\weightmat_{l,j}}^2 \right)+\normf{\weightmat_{L-1}}^2 \leq t \}$. Since
\begin{align*} 
    \bn{\vec{A}_{L-2,j} \weight_{L-1,j}}=  \underbrace{\frac{(\vec{I}_{n}-\frac{1}{n}\vec{1}_{n \times n} ) \vec{A}_{L-2,j}\weight_{L-1,j}}{\|(\vec{I}_{n}-\frac{1}{n}\vec{1}_{n \times n}  ) \vec{A}_{L-2,j}\weight_{L-1,j}\|_2}}_{\vec{h}(\theta^\prime)}\bnvarl{L-1}_j +\frac{\vec{1}_n}{\sqrt{n}}\bnmeanl{L-1}_j,
\end{align*}
where $\theta^\prime$ denotes all the parameters except ${\bnvarvecl{L-1}},{\bnmeanvecl{L-1}},{\weightmat_{L}} $. Then, independent of the value $t$, $\vec{h}(\theta^\prime)$ is always a unit norm vector. Therefore, the maximization constraint in \eqref{eq:deep_dual_fullreg} is independent of the norms of the parameters in $\theta^\prime$, which also proves that regularizing the weights in $\theta^\prime$ does not affect the dual characterization in \eqref{eq:deep_dual_fullreg}.

\subsection{Additional numerical results}\label{sec:supp_additional_exp}

Here, we present additional numerical results that are not included in the main paper. In Figure \ref{fig:twolayer_proj}, we perform an experiment to check whether the hidden neurons of a two-layer linear network align with the proposed right singular vectors. For this experiment, we select a certain $\beta$ such that $\weightmat_1$ becomes rank-two. After training, we first normalize each neuron to have unit norm, i.e., $\|\weight_{1,j}\|_2=1,\forall j$, and then compute the sum of the projections of each neuron onto each right singular vector, i.e., denoted as $\vec{v}_i$. Since we choose $\beta$ such that $\weightmat_1$ is a rank-two matrix, most of the neurons align with the first two right singular vectors as expected. Therefore, this experiment verifies our analysis and claims in Remark \ref{rem:twolayer_rank}. Furthermore, as an alternative to Figure \ref{fig:twolayer_rank}, we plot the singular values of $\weightmat_1$ with respect to the regularization parameter $\beta$ in Figure \ref{fig:twolayer_singvals}.

\begin{figure*}[ht]
\centering
\captionsetup[subfigure]{oneside,margin={1cm,0cm}}
	\begin{subfigure}[t]{0.45\textwidth}
	\centering
	\includegraphics[width=1.\textwidth, height=0.8\textwidth]{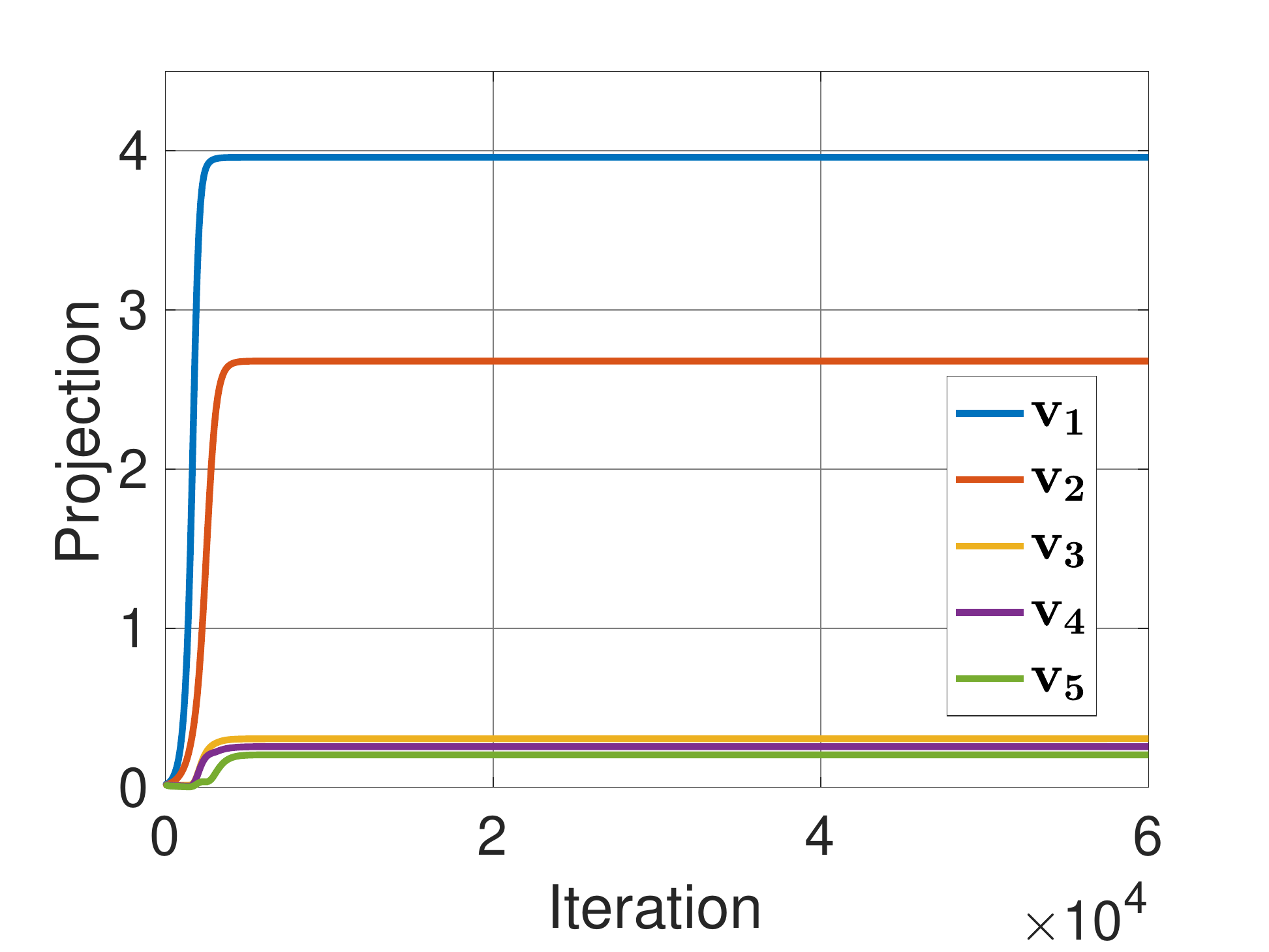}
	\caption{\centering} \label{fig:twolayer_proj}
\end{subfigure} \hspace*{\fill}
	\begin{subfigure}[t]{0.45\textwidth}
	\centering
	\includegraphics[width=1.\textwidth, height=0.8\textwidth]{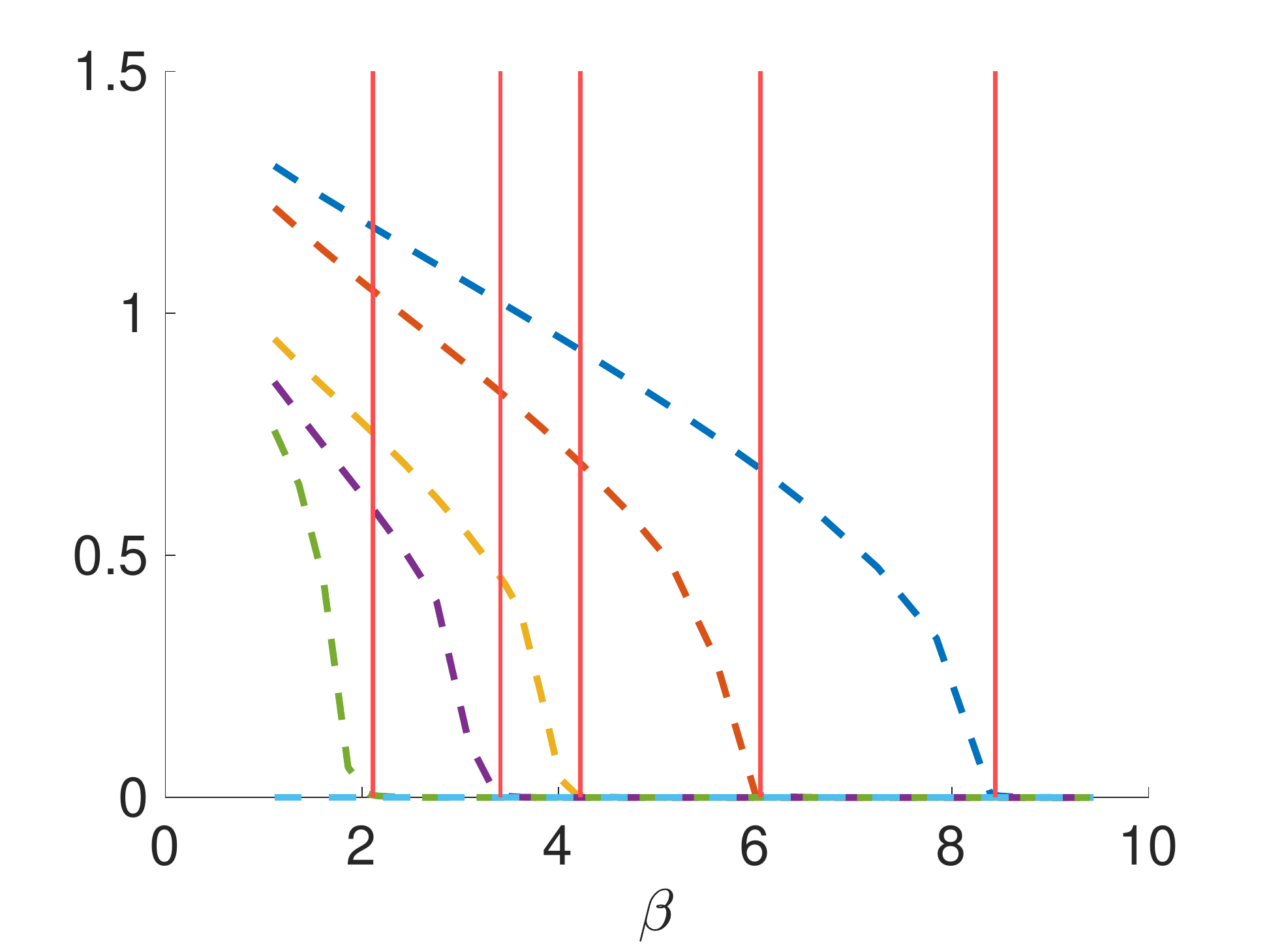}
	\caption{\centering} \label{fig:twolayer_singvals}
\end{subfigure} \hspace*{\fill}
\caption{(a) Projection of the hidden neurons to the right singular vectors claimed in Remark \ref{rem:twolayer_rank} and (b) singular values of $\weightmat_1$ with respect to $\beta$.}\label{fig:twolayer_supp}
\end{figure*}

\begin{figure*}[h]
        \centering
        \begin{subfigure}[b]{0.45\textwidth}
            \centering
            \includegraphics[width=\textwidth]{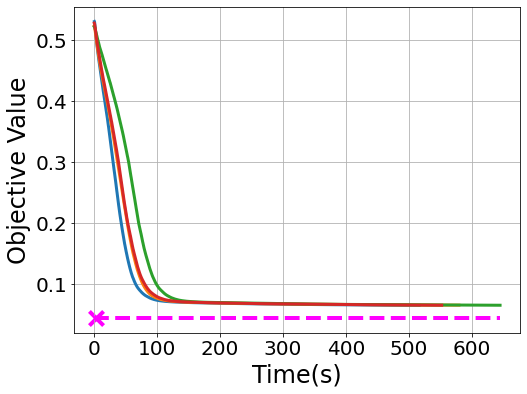}
            \caption{CIFAR-100-Training objective}\medskip
        \end{subfigure}
        \hfill
        \begin{subfigure}[b]{0.45\textwidth}   
            \centering 
            \includegraphics[width=\textwidth]{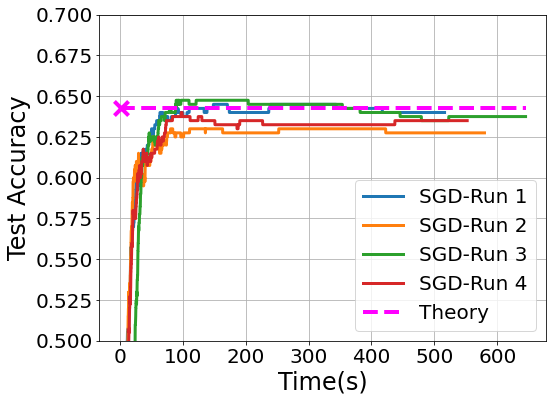}
            \caption{CIFAR-100-Test accuracy} \medskip
        \end{subfigure}
	\caption{Training and test performance of full batch SGD (4 initializations) on the CIFAR-100 datasets for a four class classification tasks, where $(n,d)=(2000,3072)$, $K=4$, $L=2$ with $100$ neurons and we use squared loss with one hot encoding. For Theory, we use the layer weights in Theorem \ref{theo:deep_vector_closedform}, which achieves the optimal performance as guaranteed by Theorem \ref{theo:deep_vector_closedform}. We also use a marker to denote the time required to compute the closed-form solution. }\label{fig:cifar_bn}
    \end{figure*}

We also conduct an experiment on CIFAR-100 \cite{cifar10} datasets, for which we consider a four class classification task. In order to verify our results in Theorem \ref{theo:deep_vector_closedform}, we train a two-layer regularized ReLU networks with batch normalization using four different initializations and then plot the results with respect to wall-clock time. As demonstrated in Figure \ref{fig:cifar_bn}, our closed form solution, i.e., denoted as Theory, achieves lower objective value as proven in Theorem \ref{theo:deep_vector_closedform} and higher test accuracy.

\subsection{Proofs for two-layer networks}\label{sec:proofs_twolayer_linear}


\theotwolayerequalitydual*
\cortwolayerextremeoptimality*

\begin{proof}[\textbf{Proof of Theorem \ref{theo:twolayer_equality_dual} and Corollary \ref{cor:twolayer_extreme_optimality}}]

We first note that the dual of \eqref{eq:problem_def_linear} with respect to $\weight_2$ is
\begin{align*}
    \min_{\theta \in \Theta \backslash\{\weight_2\}} \max_{\dual} \dual^T \vec{y} \text{ s.t. } \|(\data \weightmat_1)_{+}^T \dual \|_{\infty} \leq 1, \;\| \weight_{1,j} \|_2\leq1 , \forall j.
\end{align*}
Then, we can reformulate the problem as follows
\begin{align*}
P^*=\min_{\theta \in \Theta \backslash \{\weight_2\}}\max_{\dual} \dual^T \vec{y}  +\mathcal{I}(\|(\data \weightmat_1)_{+}^T \dual \|_{\infty} \leq 1) , \text{ s.t. } \| \weight_{1,j} \|_2\leq1 , \forall j.
\end{align*}
where $\mathcal{I}(\|(\data \weightmat_1)^T \dual \|_{\infty} \leq 1)$ is the characteristic function of the set $\|(\data \weightmat_1)^T \dual \|_{\infty} \leq 1 $, which is defined as
\begin{align*}
    \mathcal{I}(\|(\data \weightmat_1)^T \dual \|_{\infty} \leq 1) = \begin{cases} 0 & \text{if }  \|(\data \weightmat_1)^T \dual \|_{\infty} \leq 1 \\
    -\infty & \mbox{otherwise}\end{cases}.
\end{align*}
Since the set $\|(\data \weightmat_1)^T \dual \|_{\infty}\le 1$ is closed, the function $\Phi(\dual,\weightmat_1)=\dual^T \vec{y}  +\mathcal{I}(\|(\data \weightmat_1)^T \dual \|_{\infty} \leq 1)$ is the sum of a linear function and an upper-semicontinuous indicator function and therefore upper-semicontinuous. The constraint on $\weightmat_1$ is convex and compact. We use $P^*$ to denote the value of the above min-max program. Exchanging the order of min-max we obtain the dual problem given in \eqref{eq:twolayer_equality_dual}, which establishes a lower bound $D^*$ for the above problem:
\begin{align*}
P^*\ge D^*&=\max_{\dual}\min_{\theta \in \Theta \backslash \{\weight_2\}} \dual^T \vec{y}  +\mathcal{I}(\|(\data \weightmat_1)^T \dual \|_{\infty} \leq 1) , \text{ s.t. } \| \weight_{1,j} \|_2\leq1 , \forall j,\\
&= \max_{\dual} \dual^T \vec{y}, \text{ s.t. } \|(\data \weightmat_1)^T \dual \|_{\infty} \leq 1 ~\forall \weight_{1,j} \mbox{ : } \| \weight_{1,j} \|_2\leq1 , \forall j, \\
&= \max_{\dual} \dual^T \vec{y}, \text{ s.t. } \|(\data \weight_1)^T \dual \|_{\infty} \leq 1 ~\forall \weight_1 \mbox{ : } \| \weight_1 \|_2\leq 1 , 
\end{align*}
We now show that strong duality holds for infinite size NNs. The dual of the semi-infinite program in \eqref{eq:twolayer_equality_dual} is given by (see Section 2.2 of \cite{semiinfinite_goberna} and also \cite{bach2017breaking})
\begin{align*}
    &\min \|\boldsymbol{\mu}\|_{TV}\\
    &\mbox{s.t.} \int_{\weight_1\in \ball_2} \data\weight_1 d\boldsymbol{\mu}(\weight_1) = \vec{y}\,,
\end{align*}
where TV is the total variation norm of the Radon measure $\boldsymbol{\mu}$. This expression coincides with the infinite-size NN as given in \cite{bach2017breaking}, and therefore strong duality holds. We also note that although the above formulation involves an infinite dimensional integral form, by Caratheodory's theorem, the integral can be represented as a finite summation of at most $n+1$ Dirac delta functions \cite{rosset2007}. 
%
Next we invoke the semi-infinite optimality conditions for the dual problem in \eqref{eq:twolayer_equality_dual}, in particular we apply Theorem 7.2 of \cite{semiinfinite_goberna}. We first define the set
\begin{align*}
    \mathbf{K}=\mathbf{cone}\left\{ \left( \begin{array}{c}s\data \weight_1 \\ 1 \end{array}  \right), \weight_1 \in \ball_2, s\in\{-1,+1\}; \left(\begin{array}{c} \vec{0}_n \\ -1\end{array}\right) \right\}\,.
\end{align*}
Note that $\mathbf{K}$ is the union of finitely many convex closed sets, since the function $\data\weight_1$ can be expressed as the union of finitely many convex closed sets. Therefore the set $\mathbf{K}$ is closed. By Theorem 5.3 \cite{semiinfinite_goberna}, this implies that the set of constraints in \eqref{eq:twolayer_equality_dual} forms a Farkas-Minkowski system. By Theorem 8.4 of \cite{semiinfinite_goberna}, primal and dual values are equal, given that the system is consistent. Moreover, the system is discretizable, i.e., there exists a sequence of problems with finitely many constraints whose optimal values approach to the optimal value of \eqref{eq:twolayer_equality_dual}. 
The optimality conditions in Theorem 7.2 \cite{semiinfinite_goberna} implies that $\vec{y}=\data\weightmat_1^*\weight_2^*$ for some vector $\weight_2^*$. Since the primal and dual values are equal, we have ${\dual^*}^T \vec{y} ={\dual^*}^T \data\weightmat_1^*\weight_2^*= \|\weight_2^*\|_1$, which shows that the primal-dual pair $\left(\{\weight_2^*,\weightmat_1^*\} ,\dual^*\right)$ is optimal. Thus, the optimal neuron weights $\weightmat_1^*$ satisfy $\|(\data \weightmat_1^*)^T \dual^* \|_{\infty} = 1$.
\end{proof}


\propplantedmodeltwolayer*
\begin{proof}[\textbf{Proof of Proposition \ref{prop:planted_model_twolayer}}]
Let us first define a variable $\vec{w}^*$ that minimizes the following problem
\begin{align*}
      \planted=\min_{\plantednostar} \| \data \plantednostar-\vec{y}\|_2^2.
\end{align*}
Thus, the following relation holds
\begin{align*}
    \data^T (\data\planted-\vec{y})=\vec{0}_d.
\end{align*}
Then, for any $\plantednostar \in \mathbb{R}^d$, we have
\begin{align*}
    f(\plantednostar)&=\|\data \plantednostar-\data \planted+\data \planted-\vec{y}\|_2^2\\&=\|\data \plantednostar-\data \planted\|_2^2+2(\plantednostar-\planted)^T \underbrace{\data^T (\data \planted-\vec{y})}_{=\vec{0}_d}+\|\data\planted-\vec{y}\|_2^2\\
    &=\|\data \plantednostar-\data \planted\|_2^2+\|\data \planted-\vec{y}\|_2^2.
\end{align*}
Notice that $\|\data \planted-\vec{y}\|_2^2$ does not depend on $\plantednostar$, thus, the relation above proves that minimizing $f(\plantednostar)$ is equivalent to minimizing $\|\data \plantednostar-\data \planted\|_2^2$, where $\planted$ is the planted model parameter. Therefore, the planted model assumption does not change solution to the linear network training problem in \eqref{eq:problem_def_linear}.
\end{proof}


\theostrongdualitytwolayer*

\begin{proof}[\textbf{Proof of Theorem \ref{theo:strong_duality_twolayer}}]
Since there exists a single extreme point, we can construct a weight vector $\weight_e \in \mathbb{R}^{d }$ that is the extreme point. Then, the dual of \eqref{eq:problem_def_linear} with $\weightmat_1=\weight_e$ is
\begin{align}\label{eq:strong_dual_problemw}
    D_e^*= \max_{\dual} \dual^T \vec{y} \text{ s.t. } \|(\data \weight_e)^T \dual \|_{\infty} \leq 1.
\end{align}

Then, we have 
\begin{align}\label{eq:strong_dual_problemw_result}
   P^*=\nonumber &\min_{\theta \in \Theta \backslash \{\weight_2\}}\max_{\dual} \dual^T \vec{y} \hspace{3.3cm}\geq &&\max_{\dual}\min_{\theta \in \Theta \backslash \{\weight_2\}}  \dual^T \vec{y}  \\    \nonumber&\text{s.t }\|(\data \weightmat_1)^T \dual \|_{\infty} \leq 1, \;\| \weight_{1,j} \|_2\leq1 , \forall j &&\text{s.t. } \|(\data \weightmat_1)^T \dual \|_{\infty} \leq 1, \;\| \weight_{1,j} \|_2\leq1 , \forall j  \\ \nonumber
       &\hspace{6.25cm}  =  &&\max_{\dual}  \dual^T \vec{y} \\ \nonumber &\hfill &&\text{s.t. } \|(\data \weight_e)^T \dual \|_{\infty} \leq 1\\
       &\hspace{6.25cm}  = &&D_e^*=D^*
\end{align}
where the first inequality follows from changing order of min-max to obtain a lower bound and the equality in the second line follows from Corollary \ref{cor:twolayer_extreme_optimality}. 

From the fact that an infinite width NN can always find a solution with the objective value lower than or equal to the objective value of a finite width NN, we have
\begin{align}\label{eq:strong_original_upperbound}
  P^*_e=&\min_{\theta \in \Theta \backslash \{\weightmat_1,m\}} | \weightscalar_2| \hspace{2cm}\geq \hspace{1cm} P^*= \hspace{-1cm}&&\min_{\theta \in \Theta  } \| \weight_2\|_1 \\\nonumber &\text{s.t. } \data \weight_e \weightscalar_2  =\vec{y}   &&\text{s.t. } \data \weightmat_1\weight_2  =\vec{y}, \;\| \weight_{1,j} \|_2\leq1 , \forall j,
\end{align}
where $P^*$ is the optimal value of the original problem with infinitely many neurons. Now, notice that the optimization problem on the left hand side of \eqref{eq:strong_original_upperbound} is convex since it is an $\ell_1$-norm minimization problem with linear equality constraints. Therefore, strong duality holds for this problem, i.e., $P^*_e=D^*_e$. Using this result along with \eqref{eq:strong_dual_problemw_result}, we prove that strong duality holds for a finite width NN, i.e., $P_e^* = P^* = D^* =D_e^*$.


\end{proof}




\theostrongdualitytwolayerregularized*
\begin{proof}[\textbf{Proof of Theorem \ref{theo:strong_duality_twolayer_regularized}}]
Since there exists a single extreme point, we can construct a weight vector $\weight_e \in \mathbb{R}^{d }$ that is the extreme point. Then, the dual of \eqref{eq:problem_defreg_linear} with $\weightmat_1=\weight_e$
\begin{align}\nonumber 
    D_e^*=  &\max_{\dual} - \frac{1}{2}\| \dual- \vec{y}\|_2^2 +\frac{1}{2}\|\vec{y}\|_2^2 \mbox{ s.t. } |\dual^T \data \weight_e | \leq \beta.
\end{align}
Then the rest of the proof directly follows Proof of Theorem \ref{theo:strong_duality_twolayer}.
\end{proof}


\theostrongdualitylinearvector*
\begin{proof}[\textbf{Proof of Theorem \ref{theo:strong_duality_linear_vector}}]
Since there exist $r_\plantedsub$ possible extreme points, we can construct a weight matrix $\weightmat_e \in \mathbb{R}^{d \times r_\plantedsub}$ that consists of all the possible extreme points. Then, the dual of \eqref{eq:problem_def_linear_vector} with $\weightmat_1=\weightmat_e$
\begin{align}\nonumber 
    D_e^*=  &\max_{\dualmat}\trace(\dualmat^T \vec{Y})  \mbox{ s.t. } \|\dualmat^T \data \weight_{e,j} \|_2 \leq 1, \forall j \in [r_{\plantedsub}].
\end{align}
Then the rest of the proof directly follows Proof of Theorem \ref{theo:strong_duality_twolayer}.
\end{proof}


\subsection{Proofs for deep linear networks}\label{sec:proofs_deep_linear}


\propdeepweightnormeq*
\begin{proof}[\textbf{Proof of Proposition \ref{prop:deep_weightnorm_eq}}]
Let us first rescale the first $L-2$ layer weights as $\bar{\weightmat}_{l,j}=\frac{t_{l,j}}{\|\weightmat_{l,j}\|_F}\weightmat_{l,j}$, where $t_{l,j} > 0$. Defining $t_j^{L-2}=\prod_{l=1}^{L-2} \|\weightmat_{l,j}\|_F$, if $t_{l,j}$'s are chosen such that $\prod_{l=1}^{L-2}t_{l,j}=t_j^{L-2}$, then the rescaling does not alter the output of the network, i.e., $f_{\theta,L}(\data)=f_{\bar{\theta},L}(\data)$. Therefore, we can optimize $\{t_{l,j}\}_{l=1}^{L-2}$ as follows
\begin{align*}
    \min_{\{t_{l,j}\}_{l=1}^{L-2}} \frac{1}{2}\sum_{l=1}^{L-2}t_{l,j}^2 \text{ s.t. } \prod_{l=1}^{L-2}t_{l,j}=t_j^{L-2},
\end{align*}
for each $j \in [m]$. Apparently, the optimal scaling parameters obey $t_{1,j}=t_{2,j}= \ldots=t_{L-2,j}=t_j$. We also note that the optimal layer weights satisfy $t_{j}=\| \weightmat_{l,j}\|_2=\| \weightmat_{l,j}\|_F$, $\forall l \in [L-2]$, since the upper-bound is achieved when the matrices are rank-one (see \eqref{eq:deeplinear_equality_weights}).

\end{proof}


\theodeepweightsform*
\begin{proof}[\textbf{Proof of Theorem \ref{theo:deep_weights}}]
Using Lemma \ref{lemma:scaling_deep} and Proposition \ref{prop:deep_weightnorm_eq}, we have the following dual problem for \eqref{eq:problemdef_deeplinear}
\begin{align}\label{eq:dual_deeplinear}
        P^*= \min_{\{\theta_l\}_{l=1}^{L-1},\{t_j\}_{j=1}^m}\max_{\vec{\dual}}\vec{\dual}^T \vec{y} + \frac{1}{2}(L-2) \sum_{j=1}^m t_j^2 ~ \text{ s.t. } \begin{split}
           & | (\data \weightmat_{1,j} \ldots \weight_{L-1,j})^T \vec{\dual}|\leq 1,\; \weight_{L-1,j} \in \ball_2\\
       &  \|\weightmat_{l,j} \|_F \leq t_j, \; \forall l \in [L-2], \; \forall j \in [m].
        \end{split}
\end{align}
Now, let us assume that the optimal Frobenius norm for each layer $l$ is $t_j^*$ \footnote{With this assumption, $(L-2)\sum_{j=1}^m t_j^2$ becomes constant so we ignore this term for the rest of our derivations.}. Then, if we define $\Theta_{L-1}=\{\theta_1, \ldots , \theta_{L-1}| \|\weight_{L-1,j}\|_2\leq 1,\;  \|\weightmat_{l,j} \|_F \leq t_j^*, \; \forall l \in [L-2],\forall j \in [m] \}$, \eqref{eq:dual_deeplinear} reduces to the following problem
\begin{align}\label{eq:dual_deeplinear_withoutnorm}
\begin{split}
       &P^*\geq D^* =\max_{\vec{\dual}}\vec{\dual}^T \vec{y}  \text{ s.t. } | (\data \weightmat_{1,j} \ldots \weight_{L-1,j})^T \vec{\dual}|\leq 1,\; \forall \theta_l \in \Theta_{L-1},\; \forall l 
       \end{split},
\end{align}
where we change the order of min-max to obtain a lower bound for \eqref{eq:dual_deeplinear}. The dual of the semi-infinite problem in \eqref{eq:dual_deeplinear_withoutnorm} is given by
\begin{align}
    \label{eq:deep_primal_infinite}
            &\min \|\boldsymbol{\mu}\|_{TV} 
\mbox{ s.t.} \int_{\{\theta_l\}_{l=1}^{L-1}\in \Theta_{L-1} } \data\weightmat_1 \ldots \weight_{L-1}d\boldsymbol{\mu}(\theta_1,\ldots, \theta_{L-1}) = \vec{y}\,,    
\end{align}
where $\boldsymbol{\mu}$ is a signed Radon measure and $\|\cdot\|_{TV}$ is the total variation norm. We emphasize that \eqref{eq:deep_primal_infinite} has infinite width in each layer, however, an application of Caratheodory's theorem shows that the measure $\boldsymbol{\mu}$ in the integral can be represented by finitely many (at most $n+1$) Dirac delta functions \cite{rosset2007}.
Such selection of $\boldsymbol{\mu}$ yields the following problem 
\begin{align}
\label{eq:deep_primal_finite}
            &P^*_m=\min_{\{\theta_l\}_{l=1}^L} \|\weight_L\|_{1} \mbox{ s.t.} \sum_{j=1}^{m^*} \data\weightmat_{1,j} \ldots \weight_{L-1,j}\weightscalar_{L,j} = \vec{y},\; \theta_l \in \Theta_{L-1},\; \forall l
\end{align}
We first note that since the model in \eqref{eq:deep_primal_finite} has the same expressive power with the network in \eqref{eq:problemdef_deeplinear} as long as $m \geq m^*$, we have $P^* = P^*_m$. Since the dual of \eqref{eq:problemdef_deeplinear} and \eqref{eq:deep_primal_finite} are the same, we also have $D_m^*=D^*$, where $D_m^*$ is the optimal dual value for \eqref{eq:deep_primal_finite}.



We now apply the variable change in \eqref{eq:dual_linearv2} to \eqref{eq:dual_deeplinear_withoutnorm} as follows
\begin{align}\label{eq:dual_deeplinearv2}
       &  \max_{\vec{\dual}}\tilde{\vec{\dual}}^T \vec{\Sigma}_x\tilplanted_r  \text{ s.t. } \|   \weightmat_{L-2,j}^T\ldots \weightmat_{1,j}^T \vec{V}_x\vec{\Sigma}_x^T \tilde{\dual}  \|_2\leq 1,\; \forall \theta_l \in \Theta_{L-1}, \; \forall l.
\end{align}
We note that an upper-bound for the constraint in \eqref{eq:dual_deeplinearv2} can be achieved as follows
\begin{align*}
  \|   \weightmat_{L-2,j}^T\ldots \weightmat_{1,j}^T \vec{V}_x\vec{\Sigma}_x^T \tilde{\dual}  \|_2 \leq  \|\weightmat_{L-2,j}\|_2 \ldots \|\weightmat_{1,j}\|_2 \|\vec{V}_x\vec{\Sigma}_x^T \tilde{\dual}   \|_2 \leq t_j^{*^{L-2}}\|\vec{\Sigma}_x^T \tilde{\dual}   \|_2 ,
\end{align*}
where the last inequality follows from the constraint on each layer weight's norm, i.e., $\|\weightmat_{l,j}\|_F\leq t_j^*$. The equality can be reached when the layer weights are
\begin{align*}
    \weightmat_l= t_j^* \boldsymbol{\rho}_{l-1,j}  \boldsymbol{\rho}_{l,j}^T \; \forall l \in [L-2],
\end{align*}
where $\{\boldsymbol{\rho}_{l,j}\}_{l=1}^{L-2}$ is a set of arbitrary unit norm vectors and $\boldsymbol{\rho}_0=\vec{V}_x\vec{\Sigma}_x^T \tilde{\dual}/ \|\vec{V}_x\vec{\Sigma}_x^T \tilde{\dual} \|_2$. Hence, we can rewrite \eqref{eq:dual_deeplinearv2} as
\begin{align}\label{eq:dual_deeplinearv3}
       &  \max_{\vec{\dual}}\tilde{\vec{\dual}}^T \vec{\Sigma}_x\tilplanted_r  \text{ s.t. } t_j^{*^{L-2}}\|\vec{\Sigma}_x^T \tilde{\dual}   \|_2\leq 1, \forall j \in [m].
\end{align}
Therefore, the maximum objective value is achieved when $ \vec{\Sigma}_x^T \tilde{\dual}=c_1 \tilplanted_r$ for some $c_1 >0$, which yields the following set of optimal layer weight matrices
\begin{align}
\label{eq:deeplinear_equality_weights}
    &      \weightmat_{l,j}^* =\begin{cases}
     t_j^*\frac{\vec{V}_x\tilplanted_r  }{\|\tilplanted_r \|_2}\boldsymbol{\rho}_{1,j}^T\; \text{ if } l=1\\
     t_j^*\boldsymbol{\rho}_{l-1,j} \boldsymbol{\rho}_{l,j}^T\;  \text{ if } 1<l \leq L-2\\
     \boldsymbol{\rho}_{L-2,j}\; \text{ if } l=L-1
    \end{cases},
\end{align}
where $\boldsymbol{\rho}_{l,j} \in \mathbb{R}^{m_l}$ such that $\|\boldsymbol{\rho}_{l,j} \|_2=1, \; \forall l \in [L-2], \forall j \in [m]$. This shows that the weight matrices are rank-one and align with each other. Therefore, an arbitrary set of unit norm vectors, i.e., $\{\boldsymbol{\rho}_{l,j}\}_{l=1}^{L-2}$ can be chosen to achieve the maximum dual objective.

\end{proof}


\theodeepstrongduality*
\begin{proof}[\textbf{Proof of Theorem \ref{theo:deep_strong_duality}}]
We first select a set of unit norm vectors, i.e., $\{\boldsymbol{\rho}_{l,j}\}_{l=1}^{L-2}$, to construct weight matrices $\{\weightmat_{l,j}^e \}_{l=1}^{L-1}$ that satisfies \eqref{eq:deeplinear_equality_weights}. Then, the dual of \eqref{eq:problemdef_deeplinear} can be written as
\begin{align*}
\begin{split}
    &D_{e}^*= \max_{\dual} \dual^T \vec{y}  \\
       &\text{ s.t. } | (\data \weightmat_{1,j}^e \ldots \weight_{L-1,j}^e)^T \vec{\dual}|\leq 1, \, \forall j \in [m]
       \end{split}.
\end{align*}
Then, we have 
\begin{align}\label{eq:deep_strong_duality1}
   &P^*= \min_{\{\theta_l\}_{l=1}^{L-1}\in \Theta_{L-1}}\max_{\dual}  \dual^T \vec{y}  \hspace{1.9cm}\geq &&\max_{\dual} \dual^T \vec{y}  \\    \nonumber&\text{ s.t. } | (\data \weightmat_1 \ldots \weight_{L-1})^T \vec{\dual}|\leq 1 &&\text{s.t. } | (\data \weightmat_1 \ldots \weight_{L-1})^T \vec{\dual}|\leq 1,\; \forall \theta_l \in \Theta_{L-1}\\
     \nonumber &\hspace{6.25cm}  =  &&\max_{\dual}  \dual^T \vec{y}  \\    \nonumber&\hfill &&\text{s.t. }| (\data \weightmat_{1,j}^e \ldots \weight_{L-1,j}^e)^T \vec{\dual}|\leq 1 ,\, \forall j \\ \nonumber
       &\hspace{6.25cm}  = &&D_{e}^*=D^*=D^*_m,
\end{align}
where the first inequality follows from changing the order of min-max to obtain a lower bound and the first equality follows from the fact that $\{\weightmat_{l,j}^e \}_{l=1}^{L-1}$ maximizes the dual problem. Furthermore, we have the following relation between the primal problems
\begin{align}\label{eq:deep_strong_duality2}
  &P_e^*=\min_{\weight_L} \| \weight_L\|_1 \hspace{2cm}\geq \hspace{1cm} &&P^*= \min_{\{\theta_l\}_{l=1}^{L}\in \Theta_{L-1}  } \| \weight_L\|_1 \\\nonumber &\text{s.t. } \sum_{j=1}^m\data\weightmat_{1,j}^e \ldots \weight_{L-1.j}^e\weightscalar_{L,j}=\vec{y}  &&\text{s.t. } \sum_{j=1}^m \data \weightmat_{1,j} \ldots \weight_{L-1,j}\weightscalar_{L,j}=\vec{y}  ,
\end{align}
where the inequality follows from the fact that the original problem has infinite width in each layer. Now, notice that the optimization problem on the left hand side of \eqref{eq:deep_strong_duality2} is convex since it is an $\ell_1$-norm minimization problem with linear equality constraints. Therefore, strong duality holds for this problem, i.e., $P^*_e=D^*_e$ and we have $P_e^* \geq P^* = P_m^* \geq D_e^* =D^*=D_m^*$. Using this result along with \eqref{eq:deep_strong_duality1}, we prove that strong duality holds, i.e., $P_e^* = P^* =P_m^*= D_e^* =D^*=D_m^*$.


\end{proof}


\cordeeplinear*
\begin{proof}[\textbf{Proof of Corollary \ref{cor:deeplinear}}]
The proof directly follows from \eqref{eq:deeplinear_equality_weights}.
\end{proof}


\theodeepweightsregform*
\begin{proof}[\textbf{Proof of Theorem \ref{theo:deep_weights_reg}}]
Using Lemma \ref{lemma:scaling_deep} and Proposition \ref{prop:deep_weightnorm_eq}, we have the following dual for \eqref{eq:problemdef_deeplinear_regularized}
\begin{align*}
   &\max_{\dual} -\frac{1}{2} \|  \dual- \vec{y}\|_2^2+\frac{1}{2}\|\vec{y}\|_2 \mbox{ s.t. } \| (\data\weightmat_{1,j} \ldots \weightmat_{L-2,j})^T \dual\|_2 \leq \beta,\;\forall \theta_l \in \Theta_{L-1}, \; \forall l,j,
\end{align*}
where $\Theta_{L-1}=\{\theta_1, \ldots , \theta_{L-1}| \|\weight_{L-1,j}\|_2\leq 1,\;  \|\weightmat_{l,j} \|_F \leq t_j^*, \; \forall l \in [L-2],\forall j \in [m] \}$. Then, the weight matrices that maximize the value of the constraint can be described as
\begin{align*}
    & \weightmat_{l,j}^*= \begin{cases} t_j^*\frac{ \data^T \mathcal{P}_{\data,\beta}(\vec{y})}{\|\data^T \mathcal{P}_{\data,\beta}(\vec{y}) \|_2}\boldsymbol{\rho}_{1,j}^T\; \text{ if } l=1\\ t_j^*\boldsymbol{\rho}_{l-1,j} \boldsymbol{\rho}_{l,j}^T\; \text{ if } 1<l\leq L-2 \\
    \boldsymbol{\rho}_{L-2,j}\; \text{ if } l=L-1\end{cases}.
\end{align*}
where $\mathcal{P}_{\data,\beta}(\cdot)$ projects its input to $\left\{\vec{u}\in \mathbb{R}^n \; |\; \| \data^T\vec{u}\|_2 \leq \beta t_j^{*^{2-L}}\right\}$. 
\end{proof}


\cordeepstrongdualityregularized*
\begin{proof}[\textbf{Proof of Corollary \ref{cor:deep_strong_duality_regularized}}]
The proof directly follows from the analysis in this section and Theorem \ref{theo:deep_strong_duality}.
\end{proof}


\theodeepweightsformvector*

\begin{proof}[Proof of Theorem \ref{theo:deep_weights_vector}]
Using Proposition \ref{prop:deep_weightnorm_eq} and Lemma \ref{lemma:scaling_deep_vector}, we obtain the following dual problem
\begin{align}\label{eq:dual_deeplinear_vector}
        D&=\max_{\dualmat} \trace(\dualmat^T \vec{Y} )  \mbox{ s.t. } \|\dualmat^T \data \weightmat_{1,j} \ldots \weightmat_{L-2,j} \weight_{L-1,j}\|_2 \leq 1,\;\forall \theta_l \in \Theta_{L-1},\, \forall j \in[m] \nonumber \\
        &=\max_{\dualmat} \trace(\dualmat^T \vec{Y} )  \mbox{ s.t. } \sigma_{max}(\dualmat^T \data \weightmat_{1,j} \ldots \weightmat_{L-2,j} ) \leq 1,\;\forall \theta_l \in \Theta_{L-1},\, \forall j \in[m],
\end{align}
where $\Theta_{L-1}=\{\theta_1, \ldots , \theta_{L-1}| \|\weight_{L-1,j}\|_2\leq 1,\;  \|\weightmat_{l,j} \|_F \leq t_j^*, \; \forall l \in [L-2],\forall j \in [m] \}$.

It is straightforward to show that the optimal layer weights are the extreme points of the constraint in \eqref{eq:dual_deeplinear_vector}, which achieves the following upper-bound
\begin{align*}
    &\max_{\{\theta_l\}_{l=1}^{L-2} \in \Theta_{L-1}}\sigma_{max}(\dualmat^T\data \weightmat_{1,j} \ldots \weightmat_{L-2,j} ) \leq \sigma_{max}(\dualmat^T\data )t_j^{*^{L-2}}. 
\end{align*}
This upper-bound is achieved when the first $L-2$ layer weights are rank-one with the singular value $t_j^*$ by Proposition \ref{prop:deep_weightnorm_eq}. Additionally, the left singular vector of $\weightmat_{1,j}$ needs to align with one of the maximum right singular vectors of $\dualmat^T \data$. Since the upper-bound for the objective is achievable for any $\dualmat$, we can maximize the objective value, as in \eqref{eq:vectoroutput_obj_upperbound}, by choosing a matrix $\dualmat$ such that 
\begin{align*}
  \dualmat^T \vec{U}_x\vec{\Sigma}_x=\vec{V}_{\plantedsub} \begin{bmatrix} t_j^{*^{2-L}} \vec{I}_{r_{\plantedsub}} & \vec{0}_{r_x\times d-r_{\plantedsub}}\\ \vec{0}_{k-r_{\plantedsub} \times r_x} & \vec{0}_{k-r_{\plantedsub} \times d-r_{\plantedsub}}\end{bmatrix}\vec{U}_{\plantedsub}^T  ,
\end{align*}
where $\tilplantedmat_r=\vec{V}_{x}^T\plantedmat_r=\vec{U}_{\plantedsub} \vec{\Sigma}_{\plantedsub} \vec{V}_{\plantedsub}^T$. Thus, a set of optimal layer weights can be formulated as follows
\begin{align}
  \label{eq:deeplinear_equality_vectoroutput_weights}
    & \weightmat_{l,j}= \begin{cases}t_j^* \tilde{\vec{v}}_{\plantedsub,j}\boldsymbol{\rho}_{1,j}^T\;\text{ if } l=1\\
    t_j^*\boldsymbol{\rho}_{l-1,j} \boldsymbol{\rho}_{l,j}^T \;\text{ if } 1<l\leq L-2\\
    \boldsymbol{\rho}_{L-2,j}\;\text{ if } l=L-1\end{cases},
\end{align}
where $\tilde{\vec{v}}_{\plantedsub,j}$ is the $j\textsuperscript{th}$ maximal right singular vector of $\dualmat^{T^*}\data$ and we select a set of unit norm vectors $\{\boldsymbol{\rho}_{l,j}\}_{l=1}^{L-2}$ such that $\boldsymbol{\rho}_{l,j}^T\boldsymbol{\rho}_{l,k}=0,\; \forall j \neq k$. We now note that since there exist at most $K$ singular vectors of $\dualmat^T \data$ with non-zeros singular values, we can replace $m$ with $K$ without loss of generality.

\end{proof}


\theostrongdualitydeepvector*
\begin{proof}[\textbf{Proof of Theorem \ref{theo:strong_duality_deep_vector}}]
We first select a set of unit norm vectors, i.e., $\{\boldsymbol{\rho}_{l,j}\}_{l=1}^{L-2}$, to construct weight matrices $\{\weightmat_{l,j}^{e} \}_{l=1}^{L-1}$ that satisfies \eqref{eq:deeplinear_equality_vectoroutput_weights}. 
Then, we have 
\begin{align}\label{eq:deep_strong_duality1_vector}
   &P^*= \min_{\{\theta_l\}_{l=1}^{L-1}\in \Theta_{L-1}}\max_{\dualmat}  \trace(\dualmat^T \vec{Y})  \hspace{0.8cm}\geq &&\max_{\dualmat} \trace(\dualmat^T \vec{Y})  \\    \nonumber&\text{ s.t. } \sigma_{max}( \dualmat^T\data \weightmat_{1,j} \ldots \weightmat_{L-2,j})\leq 1, \forall j &&\text{s.t. } \sigma_{max}( \dualmat^T\data \weightmat_{1,j} \ldots \weightmat_{L-2,j})\leq 1,\; \forall j, \forall \theta_l \in \Theta_{L-1}\\
     \nonumber &\hspace{6.25cm}  =  &&\max_{\dualmat} \trace(\dualmat^T \vec{Y})  \\    \nonumber&\hfill && \text{s.t. }\sigma_{max}(\vec{\dualmat}^T\data \weightmat_{1,j}^{e} \ldots \weightmat_{L-2,j}^{e}) \leq 1,\; \forall j \\ \nonumber
       &\hspace{6.25cm}  = &&D_{e}^*=D^*=D^*_m,
\end{align}
where the first inequality follows from changing the order of min-max to obtain a lower bound and the first equality follows from the fact that $\{\weightmat_{l,j}^{e} \}_{l=1}^{L-1}$ maximizes the dual problem. Furthermore, we have the following relation between the primal problems
\begin{align}\label{eq:deep_strong_duality2_vector}
  &P_e^*=\min_{\weightmat_L} \sum_{j=1}^{m}\|\weight_{L,j}\|_2 \hspace{2cm}\geq \hspace{1cm} &&P^*= \min_{\{\theta_l\}_{l=1}^{L}\in \Theta_{L-1}  } \sum_{j=1}^{m}\|\weight_{L,j}\|_2 \\\nonumber &\text{s.t. } \sum_{j=1}^{m}\data\weightmat_{1,j}^{e} \ldots \weightmat_{L-1,j}^{e}\weight_{L,j}^T=\vec{Y}  &&\text{s.t. } \sum_{j=1}^m \data \weightmat_{1,j} \ldots \weight_{L-1,j}\weight_{L,j}^T=\vec{Y}  ,
\end{align}
where the inequality follows from the fact that the original problem has infinite width in each layer. Now, notice that the optimization problem on the left hand side of \eqref{eq:deep_strong_duality2_vector} is convex since it is an $\ell_2$-norm minimization problem with linear equality constraints. Therefore, strong duality holds for this problem, i.e., $P^*_e=D^*_e$ and we have $P_e^* \geq P^* = P_m^* \geq D_e^* =D^*=D_m^*$. Using this result along with \eqref{eq:deep_strong_duality1_vector}, we prove that strong duality holds, i.e., $P_e^* = P^* =P_m^*= D_e^* =D^*=D_m^*$.

\end{proof}


\theodeepweightsvectorregform*
\begin{proof}[\textbf{Proof of Theorem \ref{theo:deep_weights_vector_reg}}]
Using Lemma \ref{lemma:scaling_deep_vector} and Proposition \ref{prop:deep_weightnorm_eq}, we have the following dual for \eqref{eq:problemdef_deeplinear_vector_regularized}
\begin{align*}
   &\max_{\dualmat} -\frac{1}{2} \| \dualmat- \vec{Y}\|_F^2+\frac{1}{2}\|\vec{Y}\|_F^2  \mbox{ s.t. } \sigma_{max}(\dualmat^T \data\weightmat_{1,j} \ldots \weightmat_{L-2,j}) \leq \beta,\;\forall \theta_l \in \Theta_{L-1} ,\, \forall j \in[m],
\end{align*}
where we define $\Theta_{L-1}=\{\theta_1, \ldots , \theta_{L-1}| \|\weight_{L-1,j}\|_2\leq 1,\;  \|\weightmat_{l,j} \|_F \leq t_j^*, \; \forall l \in [L-2],\forall j \in [m] \}$. Then, as in \eqref{eq:deeplinear_equality_vectoroutput_weights}, a set of optimal layer weights is 
\begin{align*}
  \weightmat_{l,j}^*=\begin{cases} t_j^*\tilde{\vec{v}}_{x,j}\boldsymbol{\rho}_{1,j}^T\; \text{ if } l=1\\
  t_j^* \boldsymbol{\rho}_{l-1,j} \boldsymbol{\rho}_{l,j}^T\;\text{ if } 1<l\leq L-2\\
\boldsymbol{\rho}_{L-2,j}\; \text{ if } l=L-1\\
  \end{cases},
\end{align*}
where $\tilde{\vec{v}}_{x,j}$ is a maximal right singular vector of $\mathcal{P}_{\data,\beta}(\vec{Y})^T \data$ and $\mathcal{P}_{\data,\beta}(\cdot)$ projects its input to the set $\{\vec{U}\in \mathbb{R}^{n\times k}\;|\; \sigma_{max}( \vec{U}^T \data ) \leq \beta t_j^{*^{2-L}}\}$. Additionally, $\boldsymbol{\rho}_{l,j}$'s is an orthonormal set.
\end{proof}

\subsection{Proofs for deep ReLU networks}\label{sec:proofs_deep_relu}


\theodeeprelu*
\begin{props}\label{prop:deep_weightnorm_eq_relu}
First $L-2$ hidden layer weight matrices in \eqref{eq:problemdef_deeprelu} have the same operator and Frobenius norms.
\end{props}
\begin{proof}[\textbf{Proof of Proposition \ref{prop:deep_weightnorm_eq_relu}}]
Let us first denote the sum of the norms for the first $L-2$ layer as $t_j$, i.e., $t_j=\sum_{l=1}^{L-2} t_{l,j}$, where $t_{l,j}=\| \weightmat_{l,j}\|_2=\| \weightmat_{l,j}\|_F$ since the upper-bound is achieved when the matrices are rank-one. Then, to find the extreme points (see the details in Proof of Theorem \ref{theo:deeprelu}), we need to solve the following problem
\begin{align*}
\begin{split}
     &\argmax_{\{\theta_l\}_{l=1}^{L-2} } |\dual^{*^T} \vec{c}|\, \| \vec{a}_{L-2}\|_2  =    \argmax_{\{\theta_l\}_{l=1}^{L-2} \in \Theta_{L-1}} |\dual^{*^T} \vec{c}|\, \|  (\vec{a}_{L-3,j}^T \weightmat_{L-2,j})_+\|_2  
\end{split}
\end{align*}
where we use $\vec{a}_{L-2,j}^T= (\vec{a}_{L-3,j}^T \weightmat_{L-2,j})_+$. Since $\| \weightmat_{L-2,j}\|_F =t_{L-2,j}=t_j-\sum_{l=1}^{L-3}t_{l,j}$, the objective value above becomes $|\dual^{*^T} \vec{c}|\, \|  (\vec{a}_{L-3,j}\|_2 \left(t_j-\sum_{l=1}^{L-3}t_{l,j}\right)$. Applying this step to all the remaining layer weights gives the following problem
\begin{align*}
\begin{split}
     &\argmax_{\{t_{l,j}\}_{l=1}^{L-3} } |\dual^{*^T} \vec{c}|\, \| \vec{a}_{0}\|_2  \left(t_j-\sum_{l=1}^{L-3}t_{l,j}\right) \prod_{j=1}^{L-3}t_{l,j} \text{ s.t. }\sum_{l=1}^{L-3} t_{l,j} \leq t_j,\; t_{l,j}\geq 0
\end{split}.
\end{align*}
Then, the proof directly follows from Proof of Proposition \ref{prop:deep_weightnorm_eq}.
\end{proof}

\begin{proof}[\textbf{Proof of Theorem \ref{theo:deeprelu}}]
Using Lemma \ref{lemma:scaling_deep} and Proposition \ref{prop:deep_weightnorm_eq_relu}, this problem can be equivalently stated as
\begin{align}
\label{eq:primal_generic}
\begin{split}
         \min_{\{\theta_l\}_{l=1}^{L} \in \Theta_{L-1}} \|\weight_L\|_1 \text{ s.t. }&\vec{A}_{l,j}=(\vec{A}_{l-1,j}\weightmat_{l,j})_+, \text{ } \forall l\in [L-1], \forall j \in [m]\\
    &\vec{A}_{L-1}\weight_L=\vec{y}
\end{split},
\end{align}
which also has the following dual form
\begin{align}
\label{eq:dual_generic1}
  \begin{split}
        P^*= &\min_{\{\theta_l\}_{l=1}^{L-1} \in \Theta_{L-1}} \max_{\dual}\dual^T \vec{y}\\
    & \text{s.t. } \| \vec{A}_{L-1}^T \dual\|_{\infty}\leq 1
\end{split}  .
\end{align}
Notice that we remove the recursive constraint in \eqref{eq:dual_generic1} for notational simplicity, however, $\vec{A}_{L-1}$ is  still a function of all the layer weights except $\weight_L$. Changing the order of min-max in \eqref{eq:dual_generic1} gives
\begin{align}
\label{eq:dual_generic2}
  \begin{split}
         P^* \geq D^*=&\max_{\dual}\dual^T \vec{y} \text{ s.t. }  \|\vec{A}_{L-1}^T \dual\|_{\infty}\leq 1, \;\forall \theta_l \in \Theta_{L-1}, \; \forall l \in [L-1]
\end{split}.
\end{align}
 The dual of the semi-infinite problem in \eqref{eq:dual_generic2} is given by
\begin{align}
    \label{eq:deeprelu_primal_infinite}
    \begin{split}
            &\min \|\boldsymbol{\mu}\|_{TV}\\
    &\mbox{s.t.} \int_{\{\theta_l\}_{l=1}^{L-1}\in \Theta_{L-1} } \relu{\vec{A}_{L-2}\weight_{L-1}}d\boldsymbol{\mu}(\vec{\theta_1,\ldots, \theta_{L-1}}) = \vec{y}\,,  
    \end{split}
\end{align}
where $\boldsymbol{\mu}$ is a signed Radon measure and $\|\cdot\|_{TV}$ is the total variation norm. We emphasize that \eqref{eq:deeprelu_primal_infinite} has infinite width in each layer, however, an application of Caratheodory's theorem shows that the measure $\boldsymbol{\mu}$ in the integral can be represented by finitely many (at most $n+1$) Dirac delta functions \cite{rosset2007}. Thus, we choose $$\boldsymbol{\mu}= \sum_{j=1}^{m} \delta(\weightmat_1-\weightmat_{1,j}, \ldots, \weight_{L-1}-\weight_{L-1,j})\weightscalar_{L,j},$$ where $\delta(\cdot)$ is the Dirac delta function and the superscript indicates a particular choice for the corresponding layer weight. This selection of $\boldsymbol{\mu}$ yields the following problem 
\begin{align}
\label{eq:deeprelu_primal_finite}
\begin{split}
            &P^*_m=\min_{\{\theta_l\}_{l=1}^L} \|\weight_L\|_{1}\\
    &\mbox{s.t.} \sum_{j=1}^{m} \relu{\vec{A}_{L-2,j}\weight_{L-1,j}} \weightscalar_{L,j} = \vec{y},\; \theta_l \in \Theta_{L-1},\; \forall l \in [L-1]
\end{split}.
\end{align}
Here, we note that the model in \eqref{eq:deeprelu_primal_finite} has the same expressive power with ReLU networks, thus, we have $P^* = P^*_m$. 

As a consequence of \eqref{eq:dual_generic2}, we can characterize the optimal layer weights for \eqref{eq:deeprelu_primal_finite} as the extreme points that solve
\begin{align} \label{eq:extreme_Llayer}
    &\argmax_{\{\theta_l\}_{l=1}^{L-1} \in \Theta_{L-1}}|\dual^{*^T}(\vec{A}_{L-2,j}\weight_{L-1,j})_+ | 
\end{align}
where $\dual^*$ is the optimal dual parameter. Since we assume that $\data=\vec{c}\vec{a}_{0}^T$ with $\vec{c} \in \mathbb{R}^n_+$, we have $\vec{A}_{L-2,j}=\vec{c}\vec{a}_{L-2,j}^T$, where $\vec{a}_{l,j}^T=(\vec{a}_{l-1,j}^T\weightmat_{l,j})_+$, $\vec{a}_{l,j} \in \mathbb{R}^{m_{l}}_+$ and $\forall l \in [L-1],\,\forall j \in [m]$. Based on this observation, we have $\weight_{L-1,j}=\vec{a}_{L-2,j}/\| \vec{a}_{L-2,j}\|_2$, which reduces  \eqref{eq:extreme_Llayer} to the following
\begin{align}\label{eq:extreme_Llayerv2}
       &\argmax_{\{\theta_l\}_{l=1}^{L-2} \in \Theta_{L-1}} |\dual^{*^T} \vec{c}|\, \| \vec{a}_{L-2,j}\|_2  
\end{align}
We then apply the same approach to all the remaining layer weights. However, notice that each neuron for the first $L-2$ layers must have bounded Frobenius norms due to the norm constraint. If we denote the optimal $\ell_2$ norms vector for the neuron in the $l\textsuperscript{th}$ layer as $\boldsymbol{\phi}_{l,j} \in \mathbb{R}^{m_l}_+$, then we have the following formulation for the layer weights that solve \eqref{eq:extreme_Llayer}
\begin{align}\label{eq:deeprelu_equality_weights}
    \weightmat_{l,j}=\frac{\boldsymbol{\phi}_{l-1,j}}{\| \boldsymbol{\phi}_{l-1,j}\|_2}\boldsymbol{\phi}_{l,j}^T, \: \forall l \in [L-2],\;\weight_{L-1,j}=\frac{\boldsymbol{\phi}_{L-2,j}}{\| \boldsymbol{\phi}_{L-2}\|_2},
\end{align}
where $\boldsymbol{\phi}_{0,j}=\vec{a}_0$, $\{\boldsymbol{\phi}_{l,j}\}_{l=1}^{L-2}$ is a set of nonnegative vectors satisfying $\|\boldsymbol{\phi}_{l,j}\|_2=t_j^*,\; \forall l \in [L-2]$. Therefore, the set of weights in \eqref{eq:deeprelu_equality_weights} are optimal for \eqref{eq:problemdef_deeprelu}. Moreover, as a direct consequence of Theorem \ref{theo:deep_strong_duality}, strong duality holds for this case as well.

\end{proof}


\theodeeprelurankone*
 \begin{proof}[\textbf{Proof of Theorem \ref{theo:deeprelu_rankone}}]
 Given $\data=\vec{c}\vec{a}_0^T$, all possible extreme points can be characterized as follows
\begin{align*}
        \argmax_{\bias,\weight:\|\weight\|_2= 1 } |{\dual}^T \relu{\data \weight+\bias \vec{1}}\,|&=\argmax_{\bias,\weight:\|\weight\|_2= 1 } |{\dual}^T \relu{\vec{c} \vec{a}_0^T \weight +\bias\vec{1}}\,|\\
        &=\argmax_{\bias,\weight:\|\weight\|_2= 1 } \Big|\sum_{i=1}^n \dualscalar_i \relu{c_i \vec{a}_0^T \weight+\bias }\,\Big|
\end{align*}
which can be equivalently stated as
\begin{align*}
    \argmax_{\bias,\weight:\|\weight\|_2= 1} \sum_{i \in \mathcal{S}} \dualscalar_i c_i\vec{a}_0^T \weight+\sum_{i \in \mathcal{S}}\dualscalar_i \bias \text{ s.t. } \begin{cases} &c_i\vec{a}_0^T \weight+\bias \geq 0, \forall i \in \mathcal{S} \\
    &c_j\vec{a}_0^T \weight+\bias  \leq 0, \forall j \in \mathcal{S}^c\end{cases}
,\end{align*}
which shows that $\weight$ must be either positively or negatively aligned with $\vec{a}_0$, i.e.,
$\weight= s \frac{\vec{a}_0}{\|\vec{a}_0\|_2}$, where $s= \pm 1$. Thus, $\bias$ must be in the range of $[\max_{i \in \mathcal{S}}(-s c_i \| \vec{a}_0\|_2), \text{ }\min_{k \in \mathcal{S}^c}(-s c_k \| \vec{a}_0\|_2)]$ Using these observations, extreme points can be formulated as follows
\begin{align*}
    \weight_{\dualscalar}= \begin{cases}\frac{\vec{a}_0}{\| \vec{a}_0\|_2} &\text{if } \sum_{i \in \mathcal{S}} \dualscalar_i c_i \geq 0\\\frac{-\vec{a}_0}{\| \vec{a}_0\|_2} & \text{otherwise}\end{cases}
    \text{ and } \bias_{\dualscalar}= \begin{cases}\min_{k \in \mathcal{S}^c}(-s_{\dualscalar} c_k \| \vec{a}_0\|_2) &\text{if } \sum_{i \in \mathcal{S}} \dualscalar_i \geq 0\\\max_{i \in \mathcal{S}}(-s_{\dualscalar} c_i \| \vec{a}_0\|_2) & \text{otherwise}\end{cases},
\end{align*}
where $s_{\dualscalar}= \sign(\sum_{i \in \mathcal{S}} \dualscalar_i c_i)$.
\end{proof}


\propeffectbias*
\begin{proof}[\textbf{Proof of Proposition \ref{prop:effect_bias}}]
Here, we add biases to the neurons in the last hidden layer of \eqref{eq:problemdef_deeprelu}. For this case, all the equations in \eqref{eq:primal_generic}-\eqref{eq:dual_generic2} hold except notational changes due to the bias term. Thus, \eqref{eq:extreme_Llayer} changes as
\begin{align} \label{eq:extreme_Llayer_bias}
    \argmax_{ \{\theta_l\}_{l=1}^{L-1} \in \Theta_{L-1},\bias_j} |\dual^{*^T}(\vec{A}_{L-2,j}\weight_{L-1,j}+\bias_j\vec{1}_n )_+ | &=\argmax_{\{\theta_l\}_{l=1}^{L-1} \in \Theta_{L-1},\bias_j} |{\dual}^{*^T} \relu{\vec{c} \vec{a}_{L-2,j}^T \weight_{L-1,j} +\bias_j\vec{1}_n}\,|\nonumber \\
        &=\argmax_{\{\theta_l\}_{l=1}^{L-2} \in \Theta_{L-1},\bias_j} \Big|\sum_{i=1}^n \dualscalar^*_i \relu{c_i \vec{a}_{L-2,j}^T \weight_{L-1,j}+\bias_j }\,\Big|
\end{align}
which can also be written as
\begin{align*}
    \argmax_{\{\theta_l\}_{l=1}^{L-1} \in \Theta_{L-1},\bias_j}  \sum_{i \in \mathcal{S}} \dualscalar_i^* c_i\vec{a}_{L-2,j}^T \weight_{L-1,j}+\sum_{i \in \mathcal{S}} \dualscalar_i^* \bias_j \text{ s.t. } \begin{cases} &c_i\vec{a}_{L-2,j}^T \weight_{L-1}+\bias_j \geq 0, \forall i \in \mathcal{S} \\
    &c_j\vec{a}_{L-2,j}^T \weight_{L-1,j}+\bias_j  \leq 0, \forall j \in \mathcal{S}^c\end{cases}
,\end{align*}
where $\mathcal{S}$ and $\mathcal{S}^c$ are the indices for which ReLU is active and inactive, respectively. This shows that $\weight_{L-1,j}$ must be $\weight_{L-1,j}= \pm 1 \frac{\vec{a}_{L-2,j}}{\|\vec{a}_{L-2,j}\|_2}$ and $\bias_j \in [\max_{i \in \mathcal{S}}(- c_i \| \vec{a}_{L-2,j}\|_2), \text{ }\min_{k \in \mathcal{S}^c}(- c_k \|\vec{a}_{L-2,j}\|_2)]$. Then, we obtain the following
\begin{align}\label{eq:deeprelu_equality_bias_weights}
    \weight_{L-1,j}^*= \begin{cases}\frac{\vec{a}_{L-2,j}}{\| \vec{a}_{L-2,j}\|_2} &\text{if } \sum_{i \in \mathcal{S}} \dualscalar^*_i c_i \geq 0\\\frac{-\vec{a}_{L-2,j}}{\| \vec{a}_{L-2,j}\|_2} & \text{otherwise}\end{cases}
    \text{ and } \bias_j^*= \begin{cases}\min_{k \in \mathcal{S}^c}(-s_{\dualscalar^*} c_k \| \vec{a}_{L-2,j}\|_2) &\text{if } \sum_{i \in \mathcal{S}} \dualscalar^*_i \geq 0\\\max_{i \in \mathcal{S}}(-s_{\dualscalar^*} c_i \| \vec{a}_{L-2,j}\|_2) & \text{otherwise}\end{cases},
\end{align}
where $s_{\dualscalar^*}= \sign(\sum_{i \in \mathcal{S}} \dualscalar^*_i c_i)$. This result reduces \eqref{eq:extreme_Llayer_bias} to the following problem
\begin{align*}
       &\argmax_{\{\theta_l\}_{1}^{L-2} \in \Theta_{L-1}} |C(\dual^*,\vec{c})|\, \| \vec{a}_{L-2,j}\|_2 ,
\end{align*}
where $C(\dual^*,\vec{c})$ is constant scalar independent of $\{\weightmat_{l,j}\}_{l=1}^{L-2}$. Hence, this problem and its solutions are the same with \eqref{eq:extreme_Llayerv2} and \eqref{eq:deeprelu_equality_weights}, respectively.

\end{proof}


\corrankonekinks*
\cormultilayerrankonekinks*
\begin{proof}[\textbf{Proof of Corollary \ref{cor:rankone_kinks} and \ref{cor:multilayer_rankone_kinks}}]
Let us particularly consider the input sample $\vec{a}_0$. Then, the activations of the network defined by \eqref{eq:deeprelu_equality_weights} and \eqref{eq:deeprelu_equality_bias_weights} are
\begin{align*}
    &\vec{a}_{1,j}^T=(\vec{a}_0^T \weightmat_1)_+=\Big(\vec{a}_0^T \frac{\vec{a}_0}{\|\vec{a}_0\|_2}\boldsymbol{\phi}_1^T\Big)_+= \|\vec{a}_0\|_2 \boldsymbol{\phi}_{1,j}^T\\
    &\vec{a}_{2,j}^T=(\vec{a}_{1,j}^T \weightmat_2)_+=\Big(\vec{a}_{1,j}^T \frac{\vec{a}_{1,j}}{\|\vec{a}_{1,j}\|_2}\boldsymbol{\phi}_{2,j}^T\Big)_+= \|\vec{a}_0\|_2 \|\boldsymbol{\phi}_{1,j}^T\|_2 \boldsymbol{\phi}_{2,j}^T\\
    &\vdots\\
    &\vec{a}_{L-2,j}^T=(\vec{a}_{L-3,j}^T \weightmat_{L-2,j})_+=\Big(\vec{a}_{L-3,j}^T \frac{\vec{a}_{L-3,j}}{\|\vec{a}_{L-3}\|_2}\boldsymbol{\phi}_{L-2}^T\Big)_+= \|\vec{a}_0\|_2 \|\boldsymbol{\phi}_{1,j}^T\|_2 \ldots \|\boldsymbol{\phi}_{L-3,j}^T\|_2 \boldsymbol{\phi}_{L-2,j}^T\\
    &a_{L-1,j}=(\vec{a}_{L-2,j}^T \weight_{L-1,j}+b)_+=(\|\vec{a}_{L-2,j}\|_2-\|\vec{a}_{L-2,j}\|_2)_+=0.
\end{align*}
Thus, if we feed $c_i \vec{a}_0$ to the network, we get $a_{L-1,j}=(c_i\|\vec{a}_{L-2,j}\|_2-c_i\|\vec{a}_{L-2,j}\|_2)_+=0$, where we use the fact that optimal biases are in the form of $b_j=-c_i \|\vec{a}_{L-2,j} \|_2$ as proved in \eqref{eq:deeprelu_equality_bias_weights}. This analysis proves that the kink of each ReLU activation occurs exactly at one of the data points.
\end{proof}


\propdeeprelustrongdualityvector*
\begin{proof}[\textbf{Proof of Proposition \ref{prop:deeprelu_strong_duality_vector}}]
For vector outputs, we have the following training problem
\begin{align*}
    &\min_{\{\theta_l\}_{l=1}^L}  \frac{1}{2}\normf{f_{\theta,L}(\data)-\vec{Y}}^2+\frac{\beta}{2}\sum_{j=1}^m\sum_{l=1}^L\| \weightmat_{l,j} \|_F^2   .
\end{align*}
After a suitable rescaling as in the previous case, the above problem has the following dual
\begin{align}
\label{eq:dual_generic2_vector}
  \begin{split}
         P^* \geq D^*=&\max_{\dualmat} -\frac{1}{2} \| \dualmat- \vec{Y}\|_F^2+\frac{1}{2}\|\vec{Y}\|_F^2
 \text{ s.t. }  \|\dualmat^T \relu{\vec{A}_{L-2,j}\weight_{L-1,j}}\|_2  \leq \beta, \;\forall \theta_l \in \Theta_{L-1}, \; \forall l \in [L-1], \forall j \in [m]
\end{split}.
\end{align}
Using \eqref{eq:dual_generic2_vector}, we can characterize the optimal layer weights as the extreme points that solve
\begin{align} \label{eq:extreme_Llayer_vector}
    &\argmax_{\{\theta_l\}_{l=1}^{L-1} \in \Theta_{L-1}}\|\dualmat^{*^T}(\vec{A}_{L-2,j}\weight_{L-1,j})_+ \|_2,
\end{align}
where $\dualmat^*$ is the optimal dual parameter. Since we assume that $\data=\vec{c}\vec{a}_{0}^T$ with $\vec{c} \in \mathbb{R}^n_+$, we have $\vec{A}_{L-2,j}=\vec{c}\vec{a}_{L-2,j}^T$, where $\vec{a}_{l,j}^T=(\vec{a}_{l-1,j}^T\weightmat_{l,j})_+$, $\vec{a}_{l,j} \in \mathbb{R}^{m_{l}}_+$ and $\forall l \in [L-1]$. Based on this observation, we have $\weight_{L-1,j}=\vec{a}_{L-2,j}/\| \vec{a}_{L-2,j}\|_2$, which reduces  \eqref{eq:extreme_Llayer_vector} to the following
\begin{align*}
       &\argmax_{\{\theta_l\}_{l=1}^{L-2} \in \Theta_{L-1}} \|\dualmat^{*^T} \vec{c} \|_2 \, \| \vec{a}_{L-2,j}\|_2  .
\end{align*}
Then, the rest of steps directly follow Theorem \ref{theo:deeprelu} yielding the following weight matrices
\begin{align*}
    \weightmat_{l,j}=\frac{\boldsymbol{\phi}_{l-1,j}}{\| \boldsymbol{\phi}_{l-1,j}\|_2}\boldsymbol{\phi}_{l,j}^T, \: \forall l \in [L-2],\;\weight_{L-1,j}=\frac{\boldsymbol{\phi}_{L-2,j}}{\| \boldsymbol{\phi}_{L-2,j}\|_2},
\end{align*}
where $\boldsymbol{\phi}_{0,j}=\vec{a}_0$, $\{\boldsymbol{\phi}_{l,j}\}_{l=1}^{L-2}$ is a set of nonnegative vectors satisfying $\|\boldsymbol{\phi}_{l,j}\|_2=t_j^*,\; \forall l \in [L-2], \forall j \in [m]$. 
\end{proof}

\theodeepreluwhitevectorclosedform*
\begin{proof}[\textbf{Proof of Theorem \ref{theo:closedform_regularized_multiclass}}]
For vector outputs, we have the following training problem
\begin{align}\label{eq:primal_generic2_whitened_vector}
       &P^*=\min_{\theta \in \Theta } \frac{1}{2}\|f_{\theta,L}(\data)-\vec{Y}\|_F^2+\frac{\beta}{2}\sum_{j=1}^m\sum_{l=1}^L \|\weightmat_{l,j}\|_F^2
\end{align}
After a suitable rescaling as in the previous case, the above problem has the following dual
\begin{align}
\label{eq:dual_generic2_whitened_vector}
  \begin{split}
         P^* \geq D^*=&\max_{\dual} -\frac{1}{2}\normf{\dualmat-\vec{Y}}+\frac{1}{2}\normf{\vec{Y}}
 \text{ s.t. }  \|\dualmat^T \relu{\vec{A}_{L-2,j}\weight_{L-1,j}}\|_2  \leq \beta, \;\forall \theta_l \in \Theta_{L-1}, \; \forall l \in [L-1], \; \forall j \in [m]
\end{split}.
\end{align}
Using \eqref{eq:dual_generic2_whitened_vector}, we can characterize the optimal layer weights as the extreme points that solve
\begin{align} \label{eq:extreme_Llayer_whitened_vector}
    &\argmax_{\{\theta_l\}_{l=1}^{L-1} \in \Theta_{L-1}}\|\dualmat^{*^T}(\vec{A}_{L-2,j}\weight_{L-1,j})_+ \|_2,
\end{align}
where $\dualmat^*$ is the optimal dual parameter. We first note that since $\data$ is whitened such that $\data \data^T=\vec{I}_n$ and labels are one-hot encoded, the dual problem has a closed-form solution as follows
\begin{align}\label{eq:multiclass_dualparam}
   \dual_k^*= \begin{cases}
    \beta t_j^{*^{2-L}} \frac{\vec{y}_k}{\|\vec{y}_k\|_2} &\text{ if } \beta \leq \|\vec{y}_k\|_2\\\
    \hfil\vec{y}_k  &\text{ otherwise } 
    \end{cases},\quad \forall k \in [K].
\end{align}
We now note that since $\vec{Y} $ has orthogonal one-hot encoded columns, the dual constraint can be decomposed into $k$ maximization problems each of which can be maximized independently to find a set of extreme points. In particular, the $j^{th}$ problem can be formulated as follows
\begin{align*}
    &\argmax_{\{\theta_l\}_{l=1}^{L-1} \in \Theta_{L-1}}|\vec{y}_k^T(\vec{A}_{L-2,j}\weight_{L-1,j})_+ |\leq \max\left\{\|\relu{\vec{y}_k}\|_2,\|\relu{\vec{-y}_k}\|_2\right\}. 
\end{align*}
Then, noting the whitened data assumption, the rest of steps directly follow Theorem \ref{theo:deeprelu} yielding the following weight matrices
\begin{align}\label{eq:multiclass_extremepoints}
    \weightmat_{l,j}=\frac{\boldsymbol{\phi}_{l-1,j}}{\| \boldsymbol{\phi}_{l-1,j}\|_2}\boldsymbol{\phi}_{l,j}^T, \: \forall l \in [L-2],\;\weight_{L-1,j}=\frac{\boldsymbol{\phi}_{L-2,j}}{\| \boldsymbol{\phi}_{L-2,j}\|_2}  ,
\end{align}
where $\boldsymbol{\phi}_{0,j}=\data^T \vec{y}_k$ and $\{\boldsymbol{\phi}_{l,j}\}_{l=1}^{L-2}$ is a set of nonnegative vectors satisfying $\|\boldsymbol{\phi}_{l,j}\|_2=t_j^*, \; \forall l$ and $\boldsymbol{\phi}_{l,i}^T\boldsymbol{\phi}_{l,j}=0 \; \forall i\neq j$. 

We now note that given the hidden layer weight in \eqref{eq:multiclass_extremepoints}, the primal problem in \eqref{eq:primal_generic2_whitened_vector} is convex and differentiable with respect to the output layer weight $\weightmat_L$. Thus, we can find the optimal output layer weights by simply taking derivative and equating it to zero. Applying these steps yields the following output layer weights
\begin{align}\label{eq:multiclass_extremepoints2}
\begin{split}
    &\weightmat_{L-1}=\begin{bmatrix}\frac{\boldsymbol{\phi}_{L-2,1}}{\| \boldsymbol{\phi}_{L-2,1}\|_2} &\ldots& \frac{\boldsymbol{\phi}_{L-2,K}}{\| \boldsymbol{\phi}_{L-2,K}\|_2} \end{bmatrix}=\sum_{r=1}^{K}\frac{\boldsymbol{\phi}_{L-2,r}}{\| \boldsymbol{\phi}_{L-2,r}\|_2}\boldsymbol{\phi}_{L-1,r}^T\\
&\weightmat_L= \sum_{r=1}^{K} \left(\| \boldsymbol{\phi}_{0,r}\|_2-\beta\right)_+ \boldsymbol{\phi}_{L-1,r} \vec{e}_r^T ,
\end{split}
\end{align}
where $\boldsymbol{\phi}_{L-1,r}=\vec{e}_r$ is the $r^{th}$ ordinary basis vector.

Let us now assume that $t_j^*=1$ for notational simplicity and then show that strong duality holds, i.e., $P^*=D^*$. We first denote the set of indices that yield an extreme point as $\mathcal{E}=\{j:\beta \leq \|\vec{y}_j\|_2, j \in [o]\}$. Then we compute the objective values for the dual problem  \eqref{eq:dual_generic2_whitened_vector} using \eqref{eq:multiclass_dualparam}
\begin{align}\label{eq:dual_obj_multiclass2} \nonumber
   D&= - \frac{1}{2}\|\dualmat^*-\vec{Y}\|_F^2+\frac{1}{2} \| \vec{Y}\|_F^2 \nonumber\\
   &=- \frac{1}{2} \sum_{j\in \mathcal{E}} (\beta-\|\vec{y}_j\|_2)^2+\frac{1}{2}\sum_{j=1}^o \|\vec{y}_j\|_2^2 \nonumber\\
   &=-\frac{1}{2} \beta^2 |\mathcal{E}|+\beta\sum_{j \in \mathcal{E}}\|\vec{y}_j\|_2 +\frac{1}{2}\sum_{j \notin \mathcal{E}} \|\vec{y}_j\|_2^2.
\end{align}
We next compute the objective value for the primal problem \eqref{eq:primal_generic2_whitened_vector} (after applying the rescaling in Lemma \ref{lemma:scaling_deep_vector}) using the weights in \eqref{eq:multiclass_extremepoints} and \eqref{eq:multiclass_extremepoints2} as follows
\begin{align*}
    P&= \frac{1}{2}\|f_{\theta,L}(\data)-\vec{Y}\|_F^2+\frac{\beta}{2}\sum_{j=1}^{m} \|\weight_{L,j}\|_2 \nonumber\\
    &=\frac{1}{2} \left\|\sum_{j\in \mathcal{E}}\left(\|\vec{y}_j\|_2-\beta\right)\frac{\vec{y}_j}{\|\vec{y}_j\|_2}\vec{e}_j^T-\vec{Y} \right\|_F^2+\beta \sum_{j\in \mathcal{E}} (\|\vec{y}_j\|_2 -\beta)\nonumber\\
     &=\frac{1}{2}\sum_{j\in \mathcal{E}} \left\|\beta\frac{\vec{y}_j}{\|\vec{y}_j\|_2}\vec{e}_j^T\ \right\|_F^2+\frac{1}{2}\sum_{j\notin \mathcal{E}} \|\vec{y}_j\vec{e}_j^T\|_F^2+\beta \sum_{j\in \mathcal{E}} \|\vec{y}_j\|_2 -\beta^2 |\mathcal{E}|\nonumber\\
      &=-\frac{1}{2}\beta^2 |\mathcal{E}|+\frac{1}{2}\sum_{j\notin \mathcal{E}} \|\vec{y}_j\|_2^2+\beta \sum_{j\in \mathcal{E}} \|\vec{y}_j\|_2 ,
\end{align*}
which has the same value with \eqref{eq:dual_obj_multiclass2}. Therefore, strong duality holds, i.e., $P^*=D^*$, and the set of weights proposed in \eqref{eq:multiclass_extremepoints} and \eqref{eq:multiclass_extremepoints2} is optimal.
\end{proof}
\theodeeprelubatchnormclosedform*

 \begin{proof}[\textbf{Proof of Theorem \ref{theo:deep_vector_closedform}}]
We first state the primal problem after applying the scaling between $\weight_L$ and $(\bnvarvecl{L-1},\bnmeanvecl{L-1})$ as in Lemma \ref{lemma:scaling_deep_vector}
\begin{align}\label{eq:deep_vector_primal}
    &P^*=\min_{\theta \in \Theta_s} \frac{1}{2} \normf{ \sum_{j=1}^{m}\relu{\bn{\vec{A}_{L-2,j} \weight_{L-1,j}}}{\weight_{L,j}}^T- \vec{Y} }^2 + \beta \sum_{j=1}^{m}\norm{\weight_{L,j}},
\end{align}
where $\Theta_s=\{\theta \in \Theta: {\bnvarl{L-1}_j}^2+{\bnmeanl{L-1}_j}^2=1,\, \forall j \in [m]\}$ and the corresponding dual is
\begin{align}\label{eq:deep_dual}
    P^*\geq D^*=&\max_{\dualmat}   -\frac{1}{2}\|\dualmat-\vec{Y}\|_F^2+\frac{1}{2}\|\vec{Y}\|_F^2  \text{ s.t. } \max_{\theta \in \Theta_s } \left \vert \dualmat^T\relu{\bn{\vec{A}_{L-2,j} \weight_{L-1,j}}}\right \vert \leq \beta .
\end{align}
We now show that the following set of solutions for the primal and dual problem achieves strong duality, i.e., $P^*=D^*$, therefore, optimal.
\begin{align*}
   \begin{split}
       &\left(\weight_{L-1,j}^* ,\weight_{L,j}^*\right)= \begin{cases}
   \left(\vec{A}_{L-2,j}^\dagger \vec{y}_j, \left(\|\vec{y}_j\|_2-\beta\right)\vec{e}_j\right) &\text{ if } \beta \leq \|\vec{y}_j\|_2\\
    \hfil - &\text{ otherwise } 
    \end{cases}  \\
      &\begin{bmatrix} {\bnvarl{L-1}_j}^* \\{ \bnmeanl{L-1}_j}^*\end{bmatrix} =\frac{1}{\|\vec{y}_j\|_2}\begin{bmatrix} \|\vec{y}_j-\frac{1}{n}\vec{1}_{n \times n} \vec{y}_j\|_2\\ \frac{1}{{\sqrt{n}} }\vec{1}_n^T\vec{y}_j \end{bmatrix}    
  \\
    & \dual_j^*= \begin{cases}
    \beta \frac{\vec{y}_j}{\|\vec{y}_j\|_2} &\text{ if } \beta \leq \|\vec{y}_j\|_2\\\
    \hfil\vec{y}_j  &\text{ otherwise } 
    \end{cases}
     \end{split}, \quad \forall j \in [K].
\end{align*}
Now let us first denote the set of indices that achieves the extreme point of the dual constraint as $\mathcal{E}=\{j:\beta \leq \|\vec{y}_j\|_2, j \in [K]\}$. Then the dual objective in \eqref{eq:deep_dual} using the optimal dual parameter above
\begin{align}\label{eq:dual_optimal}\nonumber 
   D_{L}^*&= - \frac{1}{2}\|\dualmat^*-\vec{Y}\|_F^2+\frac{1}{2} \| \vec{Y}\|_F^2 \nonumber\\
   &=- \frac{1}{2} \sum_{j\in \mathcal{E}} (\beta-\|\vec{y}_j\|_2)^2+\frac{1}{2}\sum_{j=1}^K \|\vec{y}_j\|_2^2 \nonumber\\
   &=-\frac{1}{2} \beta^2 |\mathcal{E}|+\beta\sum_{j \in \mathcal{E}}\|\vec{y}_j\|_2 +\frac{1}{2}\sum_{j \notin \mathcal{E}} \|\vec{y}_j\|_2^2.
\end{align}
We next restate the scaled primal problem
\begin{align}\label{eq:primal_optimal}
    P_{L}^*&=\frac{1}{2} \normf{ \sum_{j=1}^{K}\relu{ \frac{(\vec{I}_{n}-\frac{1}{n}\vec{1}_{n \times n} ) \vec{A}_{L-2,j}\weight_{L-1,j}^*}{\|(\vec{I}_{n}-\frac{1}{n}\vec{1}_{n \times n} ) \vec{A}_{L-2,j}\weight_{L-1,j}^*\|_2}{\bnvarl{1}_j}^* +\frac{\vec{1}}{\sqrt{n}}{\bnmeanl{1}_j}^* }\weight_{L,j}^{*^T}- \vec{Y}}^2\nonumber + \beta \sum_{j=1}^{K}\norm{\weight_{L,j}^*}  \nonumber\\
    &=\frac{1}{2} \left\|\sum_{j\in \mathcal{E}}\left(\|\vec{y}_j\|_2-\beta\right)\frac{\vec{y}_j}{\|\vec{y}_j\|_2}\vec{e}_j^T-\vec{Y} \right\|_F^2+\beta \sum_{j\in \mathcal{E}} (\|\vec{y}_j\|_2 -\beta)\nonumber\\
     &=\frac{1}{2}\sum_{j\in \mathcal{E}} \left\|\beta\frac{\vec{y}_j}{\|\vec{y}_j\|_2}\vec{e}_j^T\ \right\|_F^2+\frac{1}{2}\sum_{j\notin \mathcal{E}} \|\vec{y}_j\vec{e}_j^T\|_F^2+\beta \sum_{j\in \mathcal{E}} \|\vec{y}_j\|_2 -\beta^2 |\mathcal{E}|\nonumber\\
      &=-\frac{1}{2}\beta^2 |\mathcal{E}|+\frac{1}{2}\sum_{j\notin \mathcal{E}} \|\vec{y}_j\|_2^2+\beta \sum_{j\in \mathcal{E}} \|\vec{y}_j\|_2 ,
\end{align}
which is the same with \eqref{eq:dual_optimal}. Therefore, strong duality holds, i.e., $P^*=D^*$, and the proposed set of weights is optimal for the primal problem \eqref{eq:deep_vector_primal}. 

\end{proof}

\corneuralcollapse*

 \begin{proof}[\textbf{Proof of Corollary \ref{cor:neural_collapse} }]
 We first restate a crucial assumptions in \cite{papyan2020neuralcollapse}.
\bass \label{ass:balanced_class}
The training dataset has balanced class distribution. Therefore, if we denote the number of data samples as $n$, then we have $\frac{n}{K}$ samples for each class $j \in [K]$.
\eass
Due to Assumption \ref{ass:balanced_class} and one-hot encoded labels, we have $\norm{\vec{y}_1}=\norm{\vec{y}_2}= \ldots =\norm{\vec{y}_K}=\sqrt{\frac{n}{K}}$. Now, we assume that $\sqrt{\frac{n}{K}}>\beta$ since otherwise none of the neurons will be optimal as proven in Theorem \ref{theo:deep_vector_closedform}. We also remark that $\sqrt{\frac{n}{K}}>1 \gg \beta$ in practice so that this assumption is trivially satisfied for practical scenarios considered in \cite{papyan2020neuralcollapse}. Therefore, the weights in Theorem \ref{theo:deep_vector_closedform} imply that
\begin{align*}
  \vec{A}_{(L-1),j}&=\relu{ \frac{(\vec{I}_n-\frac{1}{n}\vec{1}_{n \times n})\vec{A}_{L-2,j}\weight_{(L-1),j}^*}{\norm{(\vec{I}_n-\frac{1}{n}\vec{1}_{n \times n})\vec{A}_{L-2,j}\weight_{(L-1),j}^*}}{\bnvarl{L-1}}^*+\frac{\vec{1}_n}{\sqrt{n}}{\bnmeanl{L-1}}^*}\\
  &= \relu{ \frac{(\vec{I}_n-\frac{1}{n}\vec{1}_{n \times n})\vec{y}_j}{\norm{\vec{y}_j}}+\frac{\vec{1}_{n \times n} \vec{y}_j}{n\norm{\vec{y}_j}}}\\
  &= \frac{\vec{y}_j}{\norm{\vec{y}_j}}\\
    &= \frac{\sqrt{K}\vec{y}_j}{\sqrt{n}},
\end{align*}
where $\vec{A}_{(L-1),j}$ denotes the $j^{th}$ column of the last hidden layer activations after BN and the last equality follows from Assumptions \ref{ass:balanced_class}. We then subtract mean from $\vec{A}_{L-1}$ as follows
\begin{align*}
    \left(\vec{I}_n-\frac{1}{n}\vec{1}_{n \times n} \right)\vec{A}_{L-1}= \left(\vec{I}_n-\frac{1}{n}\vec{1}_{n \times n} \right)\frac{\sqrt{K}}{\sqrt{n}} \vec{Y}&=\frac{\sqrt{K}}{\sqrt{n}} \begin{bmatrix} 1-\frac{1}{K} & -\frac{1}{K}& -\frac{1}{K}& \ldots &-\frac{1}{K} \\1-\frac{1}{K} & -\frac{1}{K}& -\frac{1}{K}& \ldots &-\frac{1}{K} \\ \vdots & &\cdots \\ -\frac{1}{K} & 1-\frac{1}{K}& -\frac{1}{K}& \ldots &-\frac{1}{K} \\ -\frac{1}{K} & 1-\frac{1}{K}& -\frac{1}{K}& \ldots &-\frac{1}{K} \\ \vdots & &\ldots\end{bmatrix}\\
    &= \frac{\sqrt{K}}{\sqrt{n}}\left(\vec{I}_K \otimes \vec{1}_{\frac{n}{K}}-\frac{1}{K}\vec{1}_{n \times K} \right),
\end{align*}
where we assume that samples are ordered, i.e., the first $n/K$ samples belong to class 1, next $n/K$ samples belong to class 2 and so on. Therefore, all the activations for a certain class $k$ are the same and their mean is given by 
$$  \frac{\sqrt{K}}{\sqrt{n}}\bigg[-\frac{1}{K}\; \ldots\;  \underbrace{1-\frac{1}{K}}_{k^{th}\text{ entry}} \;\ldots \; -\frac{1}{K} \bigg]= \frac{\sqrt{K}}{\sqrt{n}}\left(\vec{e}_k^T- \frac{1}{K}\vec{1}_K^T \right),$$
which is the $k^{th}$ column of a general simplex ETF with $\alpha=\sqrt{(K-1)/n}$ and $\vec{U}=\vec{I}_K$ in Definition \ref{def:simplex_etf}. Hence, our analysis in Theory \ref{theo:deep_vector_closedform} completely explains why the patterns claimed in \cite{papyan2020neuralcollapse} emerge throughout the training of the state-of-the-art architectures. We also remark that even though we use squared loss for the derivations, this analysis directly applies to the other convex loss functions including cross entropy and hinge loss as proven in Appendix \ref{sec:supp_general_loss}.

\end{proof}
\end{document}